\documentclass[manuscript, screen]{acmart}

\usepackage{layout}
\usepackage{url}
\usepackage{subfigure}
\usepackage{balance}
\usepackage{amsmath}
\usepackage{epstopdf}
\usepackage{natbib}

\DeclareMathOperator*{\argmin}{arg\,min}

\newtheorem{defi}{Definition}
\newcommand{\ModelName}{Oriented Line Segment Target Model}
\newcommand{\ModelNameAcro}{OLS}
\newcommand{\ProblemName}{Oriented Line Segment Coverage Problem}
\newcommand{\ProblemNameAcro}{OLSC}

\usepackage{amssymb,amsmath,amsthm}

\usepackage{graphicx}

\usepackage{cprotect}

\makeatletter
\DeclareFontFamily{U}{tipa}{}
\DeclareFontShape{U}{tipa}{m}{n}{<->tipa10}{}
\newcommand{\arc@char}{{\usefont{U}{tipa}{m}{n}\symbol{62}}}%

\newcommand{\arc}[1]{\mathpalette\arc@arc{#1}}

\newcommand{\arc@arc}[2]{%
  \sbox0{$\m@th#1#2$}%
  \vbox{
    \hbox{\resizebox{\wd0}{\height}{\arc@char}}
    \nointerlineskip
    \box0
  }%
}

\usepackage[]{todonotes}

\title{On Realistic Target Coverage by Autonomous Drones}
\titlenote{A preliminary version of this work was published in the following conference paper: \\Ahmed Saeed, Ahmed Abdelkader, Mouhyemen Khan, Azin Neishaboori,
Khaled A. Harras, and Amr Mohamed. 2017. Argus: Realistic Target Coverage
by Drones. In Proceedings of the 16th ACM/IEEE International Conference
on Information Processing in Sensor Networks, Pittsburgh, PA USA, April 2017
(IPSN 2017), 12 pages.}

\begin{document}

\author{Ahmed Saeed}
\affiliation{%
  \institution{Georgia Institute of Technology}
}
\email{ahmed.saeed@gatech.edu}

\author{Ahmed Abdelkader}
\affiliation{%
  \institution{University of Maryland, College Park}}
       \email{akader@cs.umd.edu}
\author{Mouhyemen Khan}
\affiliation{%
  \institution{Georgia Institute of Technology}}
       \email{mouhyemen.khan@gatech.edu}
\author{Azin Neishaboori}
\affiliation{%
  \institution{Carnegie Mellon University}}
       \email{azin.neishaboori@gmail.com}
\author{Khaled A. Harras}
\affiliation{%
  \institution{Carnegie Mellon University}}
       \email{kharras@cs.cmu.edu}
\author{Amr Mohamed}
 \affiliation{%
  \institution{Qatar University}}
       \email{amrm@qu.edu.qa}

\renewcommand{\shortauthors}{A. Saeed et al.}


\begin{abstract}
Low-cost mini-drones with advanced sensing and maneuverability enable a new class of intelligent sensing systems.
{\color{black}
To achieve the full potential of such drones, it is necessary to develop new enhanced formulations of both common and emerging sensing scenarios. Namely, several fundamental challenges in visual sensing remain unsolved including: 1) Fitting sizable targets in camera frames; 2) Effective viewpoints matching target poses; 3) Occlusion by elements in the environment, including other targets. In this paper, we introduce Argus: an autonomous system that utilizes drones to incrementally collect target information through a two-tier architecture. To tackle the stated challenges, Argus employs a novel geometric model that captures both target shapes and coverage constraints. Recognizing drones as the scarcest resource, Argus aims to minimize the number of drones required to cover a set of targets. We prove this problem is NP-hard, and even hard to approximate, before deriving a best-possible approximation algorithm along with a competitive sampling heuristic which runs up to 100x faster according to large-scale simulations. To test Argus in action, we demonstrate and analyze its performance on a prototype implementation. Finally, we present a number of extensions to accommodate more application requirements and highlight some open problems.
}
\end{abstract}





\maketitle	

\section{Introduction}

Public spaces such as airports, train stations, shopping malls and schools, are usually monitored with the aid of security cameras mounted at key locations. Such cameras greatly help overview the area of interest and guide first responders in the event of an emergency, which can have a significant impact on crime \cite{vigne2011evaluating}.
Moreover, visual sensor systems enable the automation of complex tasks like crowd counting, event detection, object tracking, target identification, and activity recognition \cite{denman2015automatic}. The automation of these tasks has the potential of providing better solutions to several operational and security issues in public spaces (e.g., queue length estimation and perimeter protection).

{\color{black} With the increasing availability of low-cost mini-drones, visual sensors are currently finding applications beyond surveillance in disaster response \cite{disaster}, structural inspection \cite{Bircher2017}, sport streaming \cite{sports} and cinematography \cite{joubert2016towards}.}
{\color{black} We begin by examining the application needs where such compact mobile sensors can be exploited by more sophisticated sensor systems. Then, before describing the system we design to fill this gap, we review recent progress in mini-drone technologies, which is the driving force behind our work, and highlight the anticipated developments which will enable more advanced systems like the one we propose. Finally, we summarize the contributions and the structure of the paper.}

\subsection{Practical Motivation}\label{sec:motivation}
There are several theoretical and practical challenges associated with the design of effective and efficient visual sensor systems as exemplified by recent work on surveillance. Such intelligent systems with advanced features like automatic identification and recognition impose a set of requirements on video footage:
\begin{itemize}
 \item Subjects should be facing the camera \cite{blanz2005face} or within a certain viewing angle \cite{bay2006surf}.
 \item Relevant portions of subjects should be fully captured, preferably by a single camera to avoid the challenging task of stitching images from multiple viewpoints \cite{Lin_2015_CVPR}.
 \item As a prerequisite, occlusions and blind spots should be avoided by positioning cameras accordingly \cite{weinland2010making}.
\end{itemize}

An extreme approach to some of these challenges is to increase the density of deployed cameras such that any object, within the area of interest, is covered from all angles \cite{wang2011full,wang2011barrier}. This approach requires a large number of cameras, thus incurring a rather high cost \cite{7218437}. Furthermore, targets are typically modeled as mere points which results in two issues. First, mutual occlusion between targets and occlusion by obstacles in the area are not accounted for, which can create blind spots. Second, assuming sizable targets can be represented by multiple points, there is no guarantee that the target will be fully captured in the frame of at least one camera if each point is treated separately and may be covered by a different camera.
Another approach is to optimize the orientations of cameras in a static deployment to minimize occlusions, however, this does not ensure the target will be facing the camera \cite{tezcan2008self}. It is clear that modeling targets by more than mere ``blips'' can improve the quality of the collected imagery and therefore deliver more effective visual sensing systems.

{\color{black} To the best of our knowledge, no earlier work in smart surveillance tackled these challenges simultaneously. Very recently, \cite{ETH, cin_director} presented novel systems for automated cinematography using drones, taking into account framing objectives, occlusions and collision-avoidance. However, they focus on scenarios involving a known number of actors where a user specifies the desired location of each actor in the camera frame. More crucially, their experiments utilize a motion capture system where actors wear helmets equipped with tracking chips. In contrast, our system leverages weaker setups to estimate target parameters and delivers close-up views of the targets of interest without any user intervention.
}

{\color{black}
\subsection{Next Generation Sensors}\label{sec:next_gen}
The recent years have witnessed rapid developments in mini-drone technologies. Professionals and amateurs alike now have access to a wide range of drone sizes and capabilities for fun and profit. A number of proposed drone classifications, along with multiple examples, can be found in~\cite{HASSANALIAN201799}. We are particularly interested in drones at the lower end of the spectrum, exemplified by robotic flying insects~\cite{Wood_TR08}, which are now available as open-hardware~\cite{open_hardware_Sensors17}; see~\cite{Helbling2017} for a recent review. It is remarkable that such small platforms can be equipped with high end sensors, which enables a multitude of previously unforeseen applications. Specifically, there has been an equally remarkable progress in camera sensors with unprecedented feature sizes~\cite{photography_small_SPM16}. Despite the stringent constraints of weight and power consumption on top of the complexity of the required tasks, there has been a steady progress in the capabilities of these robotic flying insects~\cite{MAV_stereo_TR17}.

We anticipate the utilization of these platforms for futuristic applications through a combination of advanced system architectures and novel user-friendly deployments. As seen in the parallel developments of the \emph{Internet of Things (IoT)}~\cite{IoT_context, IoT_cities, IoT_industry}, these sensing platforms have a great potential in virtually every application domain. With the introduction of commercial intelligent personal assistants and home automation technologies like Alexa~\cite{Alexa} and Google Home~\cite{GHome}, users are becoming more familiar with such ambient devices. We envision that such systems will soon be endowed with mobile agents that live symbiotically with users both indoors and outdoors~\cite{SymbIoT}.
}

\subsection{Overview}\label{sec:overview}
{\color{black}
In this paper, we introduce Argus: an autonomous system that tackles the challenges identified in~\ref{sec:motivation} by exploiting the rapid advancements in mini-drone technologies and their anticipated applications in surveillance, crowd monitoring \cite{finn2012unmanned}, infrastructure inspection \cite{Bircher2017} and cinematography \cite{joubert2016towards}. To go beyond traditional coverage, Argus accumulates target information by dynamically controlling the available sensors to estimate target parameters before assigning the mobile drones to eliminate blind spots and capture frontal views of the subjects of interest.}
{\color{black} The proposed two-tier architecture leverages recent progress in persistent coverage~\cite{Eyes_in_the_sky}. Argus employs a top tier subsystem to detect targets within the monitored area and maintain estimates of their states. Given this information, the lower tier subsystem is responsible for controlling the available drones to obtain high quality views of the targets of interest.
A crucial aspect of the lower tier is how targets are modeled and how the locations where drones are assigned to capture target footage are computed.} Argus utilizes the \emph{\ModelName} (\emph{\ModelNameAcro}), a new geometric model we develop to incorporate target pose, size, and potential occluders.
{\em With drones being the most valuable resource, we focus on the problem of drone placement to cover targets under the {\ModelNameAcro} model while minimizing the number of drones needed.}

Intuitively, \emph{\ModelNameAcro} looks at a cross-section through the object and fits a line segment and orientation to estimate its size and pose. While still being simple, the new model is more complex than plain points and requires a more advanced system to estimate it and new algorithms to utilize it. We show that minimizing the number of drones under \emph{\ModelNameAcro} is NP-hard and even hard to approximate. Then, we proceed to develop a best-possible $O(\log{n})$-approximation algorithm, where $n$ is the number of targets. The algorithm is based on a novel spatial subdivision of the search space for camera placement by the various coverage constraints, which elucidates the treatment of the new \emph{\ModelNameAcro} model for computation. We leverage these insights to develop a more efficient coverage heuristic that almost matches the performance of the approximation algorithm while running up to 100x faster in our simulations with large numbers of targets and various target and camera parameters.

{\color{black}
We implement a fully autonomous prototype of Argus with two AR.Drone 2.0 quadcopters fitted with camera sensors and a fixed PTZ-camera.} We use the prototype to demonstrate the drastic difference in coverage quality enabled by \emph{\ModelNameAcro} compared to the traditional model of targets as mere blips on the radar.
Our experiments with synthetic targets show that adopting the enhanced \emph{\ModelNameAcro} model does not introduce significant overheads with respect to the navigation and control algorithms already running in the system. {\color{black} Thus, higher quality coverage can be achieved through an intuitive target model suitable for real-time applications.}

\subsection{Further Applications for Argus}
Although surveillance is the natural use case for Argus, the same workflow can immediately be used to plan a deployment of static cameras to cover a set of static targets (e.g., artifacts in a museum). To further demonstrate the utility of the proposed system, we discuss particularly relevant applications that have received a growing interest recently.

In structural inspection, Argus is able to represent the components to be inspected as wide objects that can occlude one another as well as provide a limited number of viewpoints that need to be visited \cite{Bircher2017}. In such scenarios, the number of targets can be very large. We analyze the scalability of Argus in Section~\ref{sec:scale}.

Cinematography, both in reality or in virtual worlds, frequently considers the planning of camera trajectories to obtain the desired shots in a given scene \cite{cin_TOG15, director_ICMR11}. Argus can generate candidate viewpoints given a description of anticipated target locations, which may even be planned by a director in a cinematographic context. Argus also accommodates additional requirements on camera placements to help optimize the shots as we discuss in Section~\ref{sec:extend}.

Argus can readily be used in several other applications, e.g., defense and public safety operations \cite{DPS} as well as disaster response \cite{disaster}. The proposed algorithms can help plan the deployment of autonomous search and rescue robots and troops to scan and secure an area while minimizing the resources allocated to each task. Since Argus takes into account target sizes and how they may be positioned and oriented within a complex environment occupied by various obstacles (e.g., buildings, terrain and foliage), besides possibly more application-specific constraints, it is able to suggest the best locations to engage the targets of interest. The efficiency of the proposed algorithms make them well-suited to the dynamic nature of such scenarios where the deployment configuration may need to be updated frequently.

\vspace{\baselineskip}
\vspace{\baselineskip}
\subsection{Summary of Contributions and Paper Organization}

The contributions of this paper are four fold:

\begin{itemize}
\item We propose and develop a fully autonomous system that controls drones to provide high quality unobstructed coverage of targets from appropriate viewpoints based on a novel \textit{\ModelName} \textit{(\ModelNameAcro)}.

{\color{black}
\item We prove that computing the minimum number of drones to cover a set of \textit{\ModelNameAcro} targets is NP-hard and even hard to approximate.}

\item We design a best-possible $O(\log{n})$-approximation algorithm and an efficient heuristic for coverage. We compare the proposed algorithms through extensive simulations.

\item We implement and analyze Argus, a complete prototype of the system, to demonstrate the superior quality of coverage the system can offer, and gauge the overhead of the proposed algorithms in action. 
\end{itemize}

{\color{black}
The rest of the paper is organized as follows. In Section~\ref{sec:system}, we describe the system architecture with an emphasis on the top tier. Then, we present the new target model and formulate the coverage problem under this model, which is the focus of the lower tier, in Section~\ref{sec:coverage_problem}. We establish the hardness of covering \textit{\ModelNameAcro} targets in Section~\ref{sec:hardness}. We proceed to analyze the coverage constraints in Section~\ref{sec:cov} before developing the coverage algorithms in Section~\ref{sec:approx}.  In Section~\ref{sec:implementation}, we report on the implementation of a full prototype of Argus, which we use together with simulations to evaluate the proposed algorithms in Section~\ref{sec:evaluation}. Recognizing scenarios where the proposed system may not be adequate, we present a number of possible extensions and other open problems in Section~\ref{sec:extend}. Finally, we present an extensive survey of related work in Section~\ref{sec:related_work} and conclude the paper in Section~\ref{sec:conclusion}.
}

\section{System Architecture}
\label{sec:system}

Argus is a fully autonomous system that aims to capture high quality video footage of identified targets of interest subject to coverage constraints. Argus employs a two-tier architecture. The top tier, used for coarse grain coverage, provides the location, width, and pose of targets and obstacles. The lower tier uses the output of the top tier to provide fine grain coverage using mobile drones; a setup we believe will become more convenient as drones get smaller, e.g., \cite{mulgaonkar2014towards}.  Having a hierarchy of surveillance systems allows each tier to be responsible for different tasks \cite{kulkarni2005senseye}; see the survey in \cite{natarajan2015multi}.

{\color{black} For the top tier, Argus leverages recent work on persistent coverage \cite{Eyes_in_the_sky, robust_cov, pers_ICRA17} to monitor an area of interest and estimate basic target information; see the surveys in~\cite{machines2010013, coop_cov_review}. Traditionally, static PTZ cameras suffice for this function, especially in indoors environments.}
{\color{black} For instance, PTZ cameras were used in \cite{chang2013automatic}} to identify the type of bags carried by subjects based on the locations determined by static cameras. Alternatively, the location and pose information of targets can be provided by non-visual means like radar and device-free localization systems (e.g., \cite{adib20143d} which can detect both the location and pose of human subjects using Wi-Fi signals). {\color{black} A more general system was described in \cite{Eyes_in_the_sky}, which accommodates heterogeneous sensors of varying degrees of freedom to achieve the required coverage. More recent proposals enable continuous adaptive coverage of changing environments~\cite{robust_cov} possibly containing obstacles \cite{pers_ICRA17}. Furthermore, these systems can be extended to take into account both energy~\cite{docking} and communication constraints~\cite{relay_networks, video_streaming}.} Using the information collected by the top tier sensors, the lower tier utilizes mobile drones to eliminate blind spots and capture frontal views of the subjects of interest.


\begin{figure}[t]
\centering
 \includegraphics[width=0.90\linewidth ]{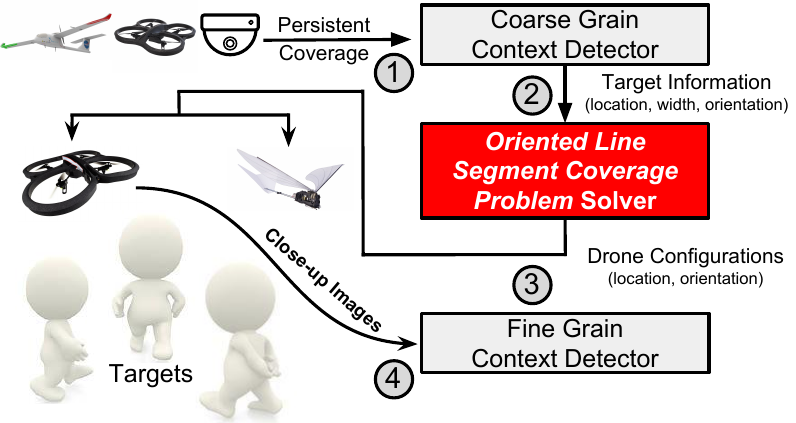}
  \caption{Operational flow of Argus.}
  \label{fig:arch}
\end{figure}

Figure~\ref{fig:arch} depicts the operational flow of the Argus system. {\color{black} The system consists of three main components: 1) \textit{Coarse Grain Context Detector}, 2) \textit{\ProblemNameAcro} \textit{Solver}, and 3) \textit{Fine Grain Context Detector}. 
The \textit{Coarse Grain Context Detector} is responsible for obtaining basic target information, exploiting data collected by the top tier, which the \textit{\ProblemNameAcro} \textit{Solver} uses to determine the positioning strategy of drones in the lower tier. The inputs that the \textit{\ProblemNameAcro} \textit{Solver} requires are the location, width, and pose of each target. The output of the \textit{\ProblemNameAcro} \textit{Solver} are used to move the lower tier drones to capture high quality unobstructed images of whole targets.} These images can then be further processed, through the \textit{Fine Grain Context Detector}, by different context extraction algorithms \cite{hu2004survey}. \textit{We realize that implementing each component is challenging in its own right with many open research problems. {\color{black} In particular, we propose a generic design which can be adapted to the specific application and context in question. Figure~\ref{fig:arch} hints at the different options for implementing each layer using the appropriate combination of sensors that can range from high altitude surveillance drones to flapping wing micro-drones like the one shown in the figure from~\cite{flapping_wing_image}.} Hence, we focus on the \textit{\ProblemNameAcro} \textit{Solver} and present vanilla approaches to the other components of Argus {\color{black} using typical hardware that can be found at most research labs}.}




\section{The Coverage Problem}
\label{sec:coverage_problem}
{\color{black}
In this section, we formulate the problem of covering a set of targets under the new \textit{\ModelNameAcro} model as required by the lower tier of the Argus system. We start by defining a convenient abstract model for both sensors and targets along with the required conditions for coverage. Then, we formally define the coverage problem that aims to minimize the number of drones needed for coverage. Finally, we highlight and justify the assumptions we make on the top tier in order to provide the information needed to provide the input to the coverage problem.}

\subsection{Definitions}\label{sec:model}
{\color{black}
We define the geometric models we use to capture the coverage problem at hand; see Figure~\ref{fig:camera_p}.
}

\begin{figure}
\centering
\subfigure[AOV ($\theta$) and VD ($\alpha$). \label{fig:aov}]{\includegraphics[width=.17\linewidth ]{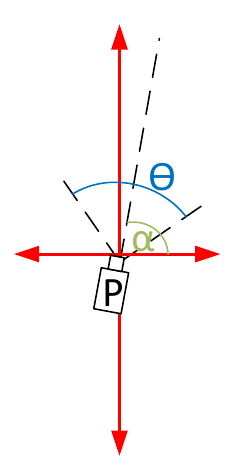}}
\hspace{0.2in}
\subfigure[Parameters controlling the shape and size of a camera's FOV (right). Violations of coverage constraints for targets T1, T2, T4 (left). Target T5 is covered but not fully covered. \label{fig:camera_target}]{\includegraphics[width=.4\linewidth]{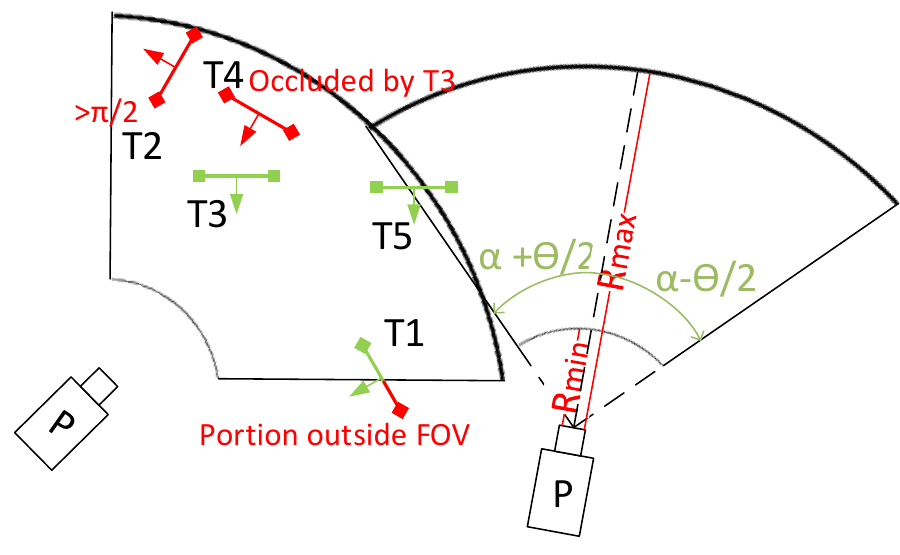}}
\caption{Camera parameters.}
\label{fig:camera_p}
\end{figure}

\subsubsection{Sensor Model} We think of sensors as autonomous mini-drones, equipped with cameras. The configuration of a sensor is a tuple $S_i = \langle P_i, \alpha_i, \theta, R_{min}, R_{max} \rangle$, where: $P_i$ is the location of the sensor in 2D and $\alpha_i$ is the \textbf{Viewing Direction (VD)} measured counter-clockwise from the positive x-axis (Figure~\ref{fig:aov}), $\theta$ is the \textbf{Angle of View (AOV)}, $R_{min}$ and $R_{max}$ are the minimum and maximum allowable distances between the camera and any target for acceptable viewing quality. Similar models has been used for anisotropic or directional sensors, e.g., \cite{amac2011coverage}.

\begin{defi} \emph{Field of View (FOV) (Figure~\ref{fig:camera_target}):} The unoccluded area that can be viewed by a sensor with an acceptable quality. Formally, it is the spherical frustum having the camera at $P$ as its apex with an axis at angle $\alpha$, an opening angle of $\theta$, and limited by $R_{min}$ and $R_{max}$ with any occlusions subtracted.
\end{defi}

\subsubsection{{\ModelName} ({\ModelNameAcro})} We model targets in 2D as line segments whose lengths are the width of the targets, and pose is a vector perpendicular to the line segment. Larger targets can be represented by one or more line segments representing their different aspects and their corresponding poses. Formally, the configuration of a target is the tuple $T_j = \langle P^s_j, P^e_j, \overrightarrow{D_j} \rangle$, where $P^s_j$ and $P^e_j$ are the start and end points of the line segment and $\overrightarrow{D_j}$ is the pose vector. Furthermore, we define $M_j$ as the midpoint of the target and let $W_j$ denote its width. We assume $W_j \ll R_{max} \:\: \forall \: j$.

\subsubsection{Obstacle Model} We reuse the line segment primitive to represent obstacles by the segments along their boundaries. Obstacle $O_k$ is a chain of segments $\{\langle P^s_1, P^e_1 \rangle, \langle P^s_2, P^e_2 \rangle, \dots\}$, which block visibility but, unlike targets, have no pose.

\subsubsection{Coverage Model} A sensor $S_i$ is said to fully cover a target $T_j$ if the following conditions apply: 1) $T_j$ falls in the FOV of $S_i$ which means that it is neither too far nor too close and that a line segment from $S_i$ to any point on $T_j$ does not intersect any other target or obstacle. 2) The angle between $\overrightarrow{D_j}$ and $\overrightarrow{M_j P_i}$ is $\leq \pi/2$, meaning that $S_i$ can capture frontal views of $T_j$. See Figure~\ref{fig:camera_target} for a summary.

\begin{defi}
\label{def:fullC}
\emph{Full Coverage:} A target $T_j$ is fully covered if $\overline{P^s_jP^e_j}$ is fully contained in the FOV of some camera $S_{i^\ast}$, with $T_j$ facing $S_{i^\ast}$.
\end{defi}

\subsection{Minimizing the Number of Drones}
We formally define the coverage problem for \textit{\ModelNameAcro} targets and briefly discuss its hardness and the approaches we take to compute a solution.

\begin{defi} \ProblemName{} {(\ProblemNameAcro)}: Let $\mathcal{T}$ be a set of $n$ oriented line segments, that may only intersect at their end points, and $\mathcal{O}$ be a set of $u$ obstacles. Find the minimum number of mobile directional visual sensors, with uniform  $\langle \theta, R_{min} ,R_{max} \rangle$, required to fully cover all segments in $\mathcal{T}$.
\end{defi}
It is necessary to establish lower-bounds on the efficiency of algorithms for such problems to better understand how to tackle them in practice. To this end, we show that \emph{\ProblemNameAcro} is NP-Hard and even hard to approximate by studying a variant of the Art Gallery Problem with an AOV $\theta < 360^\circ$ (\S~\ref{sec:hardness}).

Solving \emph{\ProblemNameAcro} requires the generation of a set of candidate camera placement configurations (i.e., location and orientation pairs) and selecting a set of configurations that cover all targets while minimizing the number of cameras. This approach relies on subdividing the search space (i.e., the plane) by the various coverage requirements of the targets in $\mathcal{T}$. These subdivisions produce a finite set of potential camera location and orientation pairs ($\mathcal{R}$) which is convenient for computation. We consider $\mathcal{R}$ to be \textit{comprehensive} if it contains at least one representative for each region of space where cameras could be placed to cover any given subset of targets. With that, \emph{\ProblemNameAcro} is reduced to picking a subset of $\mathcal{R}$ to cover all targets in $\mathcal{T}$, which is equivalent to solving the \verb|SET-COVER| problem over $(\mathcal{T}, \mathcal{R})$. Hence, applying the well-known greedy selection scheme guarantees an $O(\log{n})$-approximation of the minimum number of cameras needed to cover $\mathcal{T}$ \cite{chvatal1979greedy}, which, by our lower-bound results, is the best-possible for \emph{\ProblemNameAcro}.

\subsection{Modeling Assumptions}

{\color{black} The main assumption we make in this work is that target locations and poses can be estimated by a coarse grain coverage system. As described in Section~\ref{sec:system}, we rely on a suitable coarse grain context detector to fulfill this requirement. The \textit{Coarse Grain Context Detector} leverages established results in target detection and tracking using fixed cameras, e.g., visual sensors for pedestrian tracking \cite{chen2012we}, or other contextual sensors, e.g., device-free RF-based techniques \cite{adib20143d}. In addition, recent advances in persistent coverage~\cite{robust_cov, pers_ICRA17} greatly extend the range of application contexts where such information can be robustly collected within complex dynamic environments. The \textit{\ModelNameAcro} model} is essentially proposed to obtain close-up views of relevant targets using mobile cameras and provide fine grain surveillance as needed. This approach is a natural next step in multi-tier camera sensor networks where coarse grain information may be acquired via higher tier cameras providing low granularity coverage sufficient for detection and localization, but insufficient for identification, recognition, or activity monitoring \cite{kulkarni2005senseye}. {\color{black} We point out that for a variety of scenarios, and in particular for monitoring public spaces, there are numerous contextual hints that can be exploited to simplify the information that has to be estimated. For example, for human subjects, a median value of target width suffices for the operation of the system, without having to estimate more accurate values per target.}

\section{On the Hardness of \ProblemNameAcro{}}
\label{sec:hardness}

As discussed in the previous section, the \ProblemNameAcro{} Solver is the focus of this paper. We analyze the hardness of \ProblemNameAcro{} by relating it to polygon illumination problems. We relate the two problems by looking at the generic Omni-\ProblemNameAcro{}. Then, we leverage our earlier results on the inapproximability of polygon illumination \cite{abdelkaderinapproximability} and a reduction by Eidenbenz et al. \cite{eidenbenz2001inapproximability} to prove the inapproximability of \ProblemNameAcro{}.

\subsection{Omni-\ProblemNameAcro{} is hard}\label{hardness_easy}

We consider a special case with camera parameters fixed to $R_{min} = 0$, $R_{max} = \infty$, and $\text{AOV} = 360^\circ$. We call this problem {Omni-\ProblemNameAcro} as all cameras are now omnidirectional and their viewing direction (VD) is no longer relevant. {Omni-\ProblemNameAcro} allows for a direct reduction from the point guard art gallery problem as it is solely defined by the set of targets $\mathcal{T}$. This enables us to develop the intuition behind the proof of hardness for the general {\ProblemNameAcro}. 

\begin{figure}[!t]
\centering
\subfigure[Point Guard AGP instance $\Upsilon = \{ \upsilon_1, \upsilon_2, \dots, \upsilon_9 \}$. \label{fig:theorem_explanation_1}]{\includegraphics[width=.3\linewidth ]{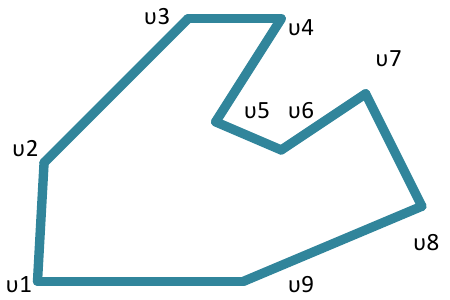}}
\hspace{0.2in}
\subfigure[The corresponding \ProblemNameAcro{} instance $\mathcal{T} = \{ T_1, T_2, \dots, T_9 \}$.]{\includegraphics[width=.3\linewidth]{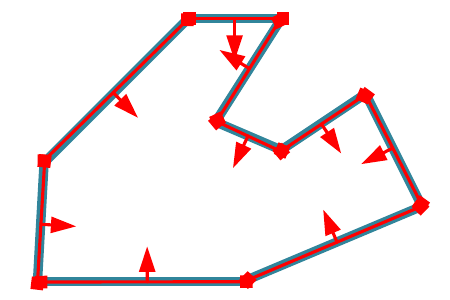}}
\caption{Reducing the Point Guard Art Gallery Problem (PGAGP) to the {\ProblemName} {(\ProblemNameAcro)}.}
\label{fig:theorem_explanation}
\end{figure}

We define a polygon $\Upsilon$ as an ordered sequence of $n$ vertices $\upsilon_1,\upsilon_2, ..., \upsilon_n$ where $n \geq  3$, as shown in Figure~\ref{fig:theorem_explanation_1}. $\Upsilon$ forms a closed planar region bounded by the edges $\overline{\rm \upsilon_1\upsilon_2}$, $\overline{\rm \upsilon_2\upsilon_3}$, ..., $\overline{\rm \upsilon_{n-1}\upsilon_n}, \overline{\rm \upsilon_{n}\upsilon_1}$. A simple polygon $\Upsilon$, without holes, divides the plane into three faces relative to $\Upsilon$: interior ($\Upsilon_i$), exterior ($\Upsilon_e$), and boundary ($\Upsilon_b$). A point is said to lie in $\Upsilon$ iff it belongs to $\Upsilon_i \cup \Upsilon_b$. Two points $x$ and $y$ in $\Upsilon$ are said to be mutually visible if the line segment $\overline{xy}$ lies completely in $\Upsilon$. Finally, we use $cover(x)$ to denote the subset of points in $\Upsilon_b$ visible from a point $x \in \Upsilon$.

Our proof makes use of the \emph{Point Guard Art Gallery Problem (PGAGP)}. A set of points $G \subset \Upsilon$, referred to as a guard set, is said to cover the boundary of $\Upsilon$ iff $\Upsilon_b \subset \cup_{x\in G} cover(x)$. PGAGP is to find a guard set $G^\ast$ to cover the boundary of given polygon $\Upsilon$, such that $|G^\ast|$ is minimum. PGAGP for boundary coverage was shown to be NP-hard \cite{eidenbenz2001inapproximability} even when limited to convex \cite{culberson1988covering} or star-shaped guard views \cite{lee1986computational}. 

\begin{theorem}
\label{th:omni_DDTC}
Omni-{\ProblemNameAcro} is NP-Hard.
\end{theorem}

\begin{proof}

We encode a given PGAGP instance $\Upsilon$ as an Omni-\emph{\ProblemNameAcro} instance $\mathcal{T}$, e.g. as shown in Figure~\ref{fig:theorem_explanation}. We map each edge $\overline{\rm \upsilon_{i-1}\upsilon_i}$ to a target $T_i$ whose $P^s_i$ is $\upsilon_{i-1}$, $P^e_i$ is $\upsilon_i$, and $\overrightarrow{D_i}$ points to the interior of $\Upsilon$. By Definition~\ref{def:fullC}, the placement of cameras is limited to the interior of $\Upsilon$. Hence, the minimum number of omnidirectional sensors required to cover all targets is exactly the minimum number of point guards required to cover the polygon. It follows that Omni-\emph{\ProblemNameAcro} is at least as hard as PGAGP.
\end{proof}

The same reduction presented above can also be used to reduce PGAGP for polygons with holes to Omni-{\ProblemNameAcro}. The edges bounding each polygonal hole are mapped to \ModelNameAcro{} targets facing the interior of the polygon.

\subsection{Polygon Illumination and \ProblemNameAcro}



Our study of polygon illumination is motivated by the natural reduction to our problem as described in the previous subsection. Two additional restrictions to the art gallery problem are needed to capture the {\ModelNameAcro} model. First, since we use cameras having a limited FOV, we turn our attention to the Point $\alpha$-Floodlight Illumination Problem for art galleries with Holes (PFIPH), where the art gallery is {\color{black} to be covered with $\alpha$-floodlights, i.e., floodlights of angle $\alpha$}. Second, we further require that each edge is \textit{fully covered} by at least one $\alpha$-floodlight to get a restricted variant that we call F-PFIPH. Illumination of polygons without holes by $\alpha$-floodlights where floodlights can only be placed at vertices of the polygon is NP-hard \cite{bagga1996complexity}. However, the hardness of PFIPH is an open problem \cite{urrutia2000art}. In summary, we will be dealing with the two problems below.

\begin{defi}
\label{def:pfip}
\emph{Point $\alpha$-Floodlight Illumination Problem for polygons with Holes (PFIPH):} Given a polygon $P$ with holes and angle $\alpha$, find the minimum number of $\alpha$-floodlights to illuminate the whole polygon such that $\alpha$-floodlights can be placed anywhere inside the polygon.
\end{defi}

\begin{defi}
\label{def:fpfip}
\emph{Full-coverage Point $\alpha$-Floodlight Illumination Problem for polygons with Holes (F-PFIPH):} Given a polygon $P$ with holes and angle $\alpha$, find the minimum number of $\alpha$-floodlights to illuminate the whole polygon such that $\alpha$-floodlights can be placed anywhere inside the polygon and each edge is fully illuminated by at least one $\alpha$-floodlight.
\end{defi}

It is clear that \emph{PFIPH} and \emph{F-PFIPH} are relevant to many surveillance problems, especially with the increasing interest in coverage with directional sensors \cite{amac2011coverage}. Hence, it is necessary to establish lower-bounds on the efficiency of algorithms for such problems to enable more principled approaches for tackling them in practice.

\subsection{{\ProblemNameAcro} is Hard, Even to Approximate}

We develop the Point Floodlight Gadget (PFG) that we use in combination with the gadgets of \cite{eidenbenz2001inapproximability} and \cite{abdelkaderinapproximability} for our reduction. We prove the hardness and inapproximability of \emph{PFIPH} using a reduction that immediately yields the same results for both \emph{F-PFIPH} and {\ProblemNameAcro}. With that, we settle the open problem regarding the hardness of PFIPH \cite{urrutia2000art} and establish the hardness of {\ProblemNameAcro}.

\begin{figure}[t]
\centering
\begin{minipage}[b]{0.35\linewidth}
\centering
\includegraphics[width=1\linewidth ]{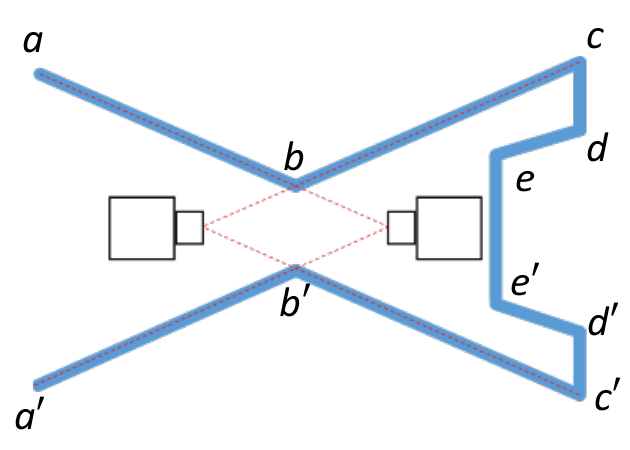}
  \caption{Point Floodlight Gadget.}
  \label{fig:pcg}
\end{minipage}
\hfill
\begin{minipage}[b]{0.62\linewidth}
\centering
\includegraphics[width=1\linewidth ]{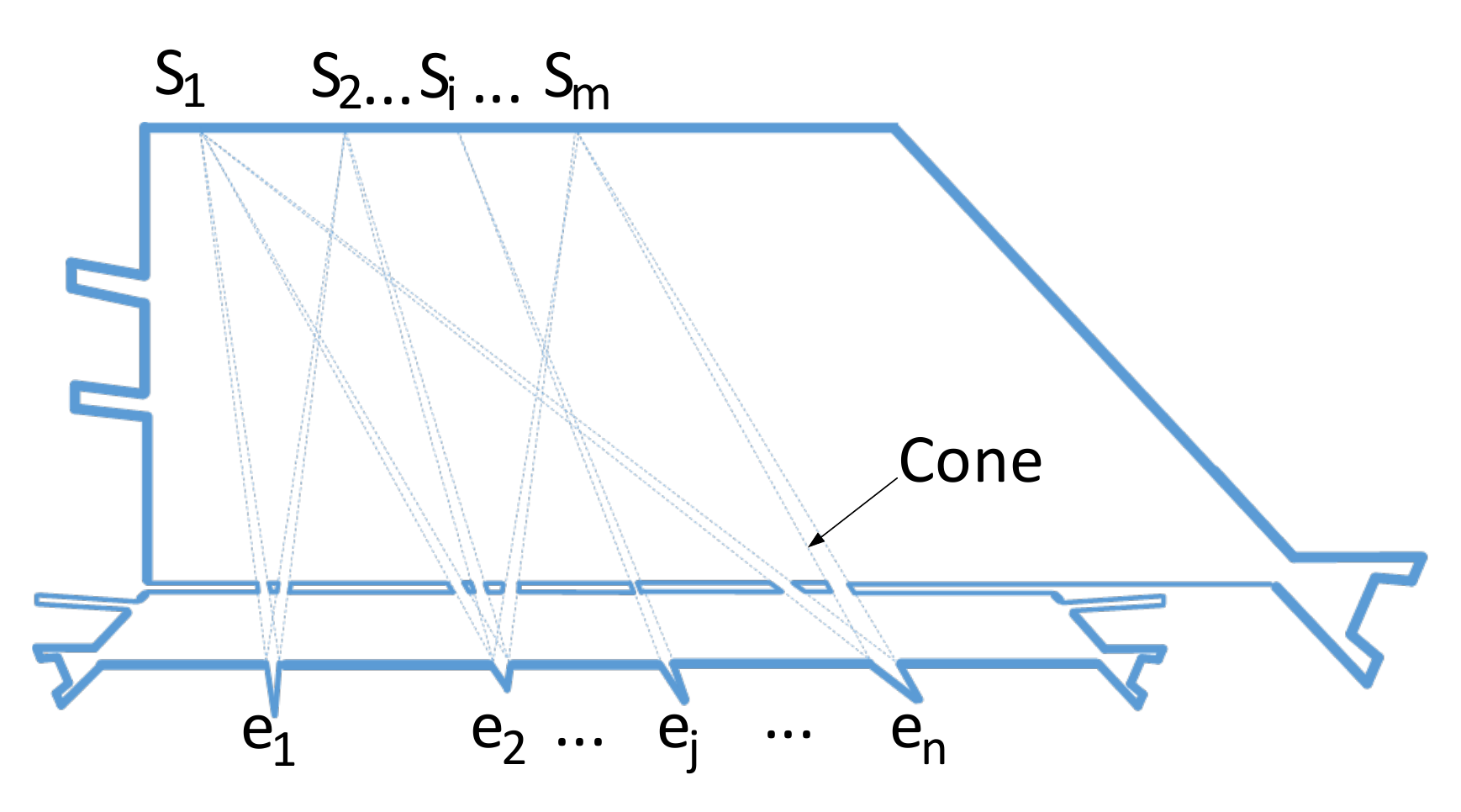}
  \caption{Reducing SET-COVER to Point $\alpha$-Floodlight Illumination for polygons with Holes (PFIPH).}
 \label{fig:setcover_1}
\end{minipage}
\end{figure}
The \emph{\textbf{Point Floodlight Gadget}} \textbf{(PFG)} (Figure~\ref{fig:pcg}) is used to force the placement of floodlights inside the polygon, rather than on its edges or vertices. This simple gadget facilitates the encoding of PFIPH constraints which leverages earlier hardness proofs for the art gallery problem (e.g., \cite{bagga1996complexity} and \cite{culberson1988covering}) by forcing floodlights to be placed inside the polygon while allowing some control on where floodlights are pointed.


As shown in Figure~\ref{fig:pcg}, a PFG requires exactly two point floodlights to be fully covered. The intuition behind this design is to have a shape that can only be covered by a single camera from a unique configuration. The PFG is designed by intersecting two isosceles triangles with apex angle $\alpha_{PFG}=\min(\alpha, 60^\circ)$, an arbitrary upper bound chosen for convenience. The right triangle has two dents to force the placement of the left camera, while the left triangle acts as an interface to allow plugging the PFG into larger constructs. Such constructs also need to have dents that force the placement of the right camera. Each of the two triangles forming a PFG can be covered by a single $\alpha$-floodlight (i.e., a floodlight with angle $\alpha \geq \alpha_{PFG}$) placed at its apex.


We use the PFG to adapt the construction in \cite{eidenbenz2001inapproximability} in order to obtain a reduction from \verb|SET-COVER| to  \emph{PFIPH}. Given a set system $(E, S)$, the \verb|SET-COVER| problem asks for the minimum number of subsets from $S$ to cover all elements in $E$, where $S \subseteq 2^E$. As shown in Figure \ref{fig:setcover_1}, the reduction naturally maps the set system to a polygon with holes such that each element in $E$ is represented by a dent on the bottom side and each subset in $S$ is a potential floodlight configuration at the top edge. A horizontal barrier with tapered corridors is added to enforce this correspondence by limiting which dents are visible from each location.


We define a \textit{cone} as the maximal convex area inside the polygon that includes one dent and one corridor. A floodlight placed above the barrier may illuminate a dent only if it belongs to one of the cones containing it. The polygon is designed with the barrier at a certain height such that at most two cones from different dents intersect above the barrier \cite{eidenbenz2001inapproximability}.
 
 We extend the construction in \cite{eidenbenz2001inapproximability} as follows: 1) Three PFGs are added to illuminate everything other than the dents. This incurs $4$ more \textit{guards} than the design in \cite{eidenbenz2001inapproximability}. 2) We set the height of the top side of the polygon to ensure that none of the guards requires an angle of view greater than $\alpha$ to cover the subset of dents assigned to it. , This height, $y_0$, is an arbitrary constant independent of the width of the polygon \cite{eidenbenz2001inapproximability}. It is clear the required changes do not violate any of the properties of the original construction and the results carry through.

\begin{theorem}
 PFIPH is NP-hard.
 \label{th:pfiph_hardness}
\end{theorem}

For hardness of approximation, we reduce from a restricted variant of \verb|SET-COVER| using the same procedure outlined above. \verb|RESTRICTED-SET-COVER| simply restricts \verb|SET-COVER| by requiring $|S| \leq |E|$ and is easily shown at least as hard to approximate as \verb|DOMINATING-SET| (Lemma 9, \cite{eidenbenz2001inapproximability}). Since our version of the reduction only changes the cost of a solution by a constant additive factor of $4$, we also get an equivalent of the promise problem in (Lemma 10, \cite{eidenbenz2001inapproximability}) with slightly different constants, i.e., $c+6$ instead of $c+2$. This effectively yields an approximation preserving reduction from \verb|RESTRICTED-SET-COVER| and we get a corresponding result.

\begin{theorem}\label{th:pfip_holes}
PFIPH cannot be approximated by a polynomial time algorithm with an approximation ratio of $\frac{1-\epsilon}{12}\ln{n}$ for any $\epsilon > 0$, where $n$ is the number of vertices of the input polygon, unless $NP \subseteq TIME(n^{O(\log{\log{n}})})$.
\end{theorem}


We observe that all edges in our reductions are fully covered by at least one guard. This means that the above results extend to \emph{F-PFIPH} as well. Finally, using the simple approximation-preserving reduction outlined in Figure~\ref{fig:theorem_explanation}, the same applies to {\ProblemNameAcro}. Hence, the following is a corollary of theorems \ref{th:omni_DDTC}, \ref{th:pfiph_hardness}, and \ref{th:pfip_holes}.

\begin{corollary}\label{th:OLSI}
{\ProblemNameAcro} is NP-Hard.
\end{corollary}

\begin{corollary}\label{th:OLSI_inapprox}
{\ProblemNameAcro} cannot be approximated by a polynomial time algorithm with an approximation ratio of $\frac{1-\epsilon}{12}\ln{n}$ for any $\epsilon > 0$, where $n$ is the number of vertices of the input polygon, unless $NP \subseteq TIME(n^{O(\log{\log{n}})})$.
\end{corollary}

\section{Covering \ModelNameAcro{} Targets}
\label{sec:cov}
{\color{black}
Our drone placement algorithms rely on a decomposition of space by the various coverage constraints per target.
Recalling the definitions of Section~\ref{sec:coverage_problem},} we have four constraints for a camera to cover a target: range (i.e., being within $R_{min}$ and $R_{max}$ from the target), angle of view (i.e., being within the camera's FOV of width $\theta$), target pose (i.e., capturing the target from its significant perspective) and occlusion avoidance (i.e., having no target or obstacle occluding the target of interest). First, we focus on satisfying all these constraints for a single target, which allows us to develop the essential tools needed to compute drone placements. Then, we show the extension to a pair of targets using a convenient approach to covering multiple target simultaneously. Arbitrary subsets of targets can then be covered by satisfying their coverage constraints in a pairwise fashion.

\begin{figure}[!t]
\centering
\subfigure[Boundary of points no farther than $R_{max}$ from any point on the target.\label{fig:cpf_def1}]{\includegraphics[width=.325\linewidth ]{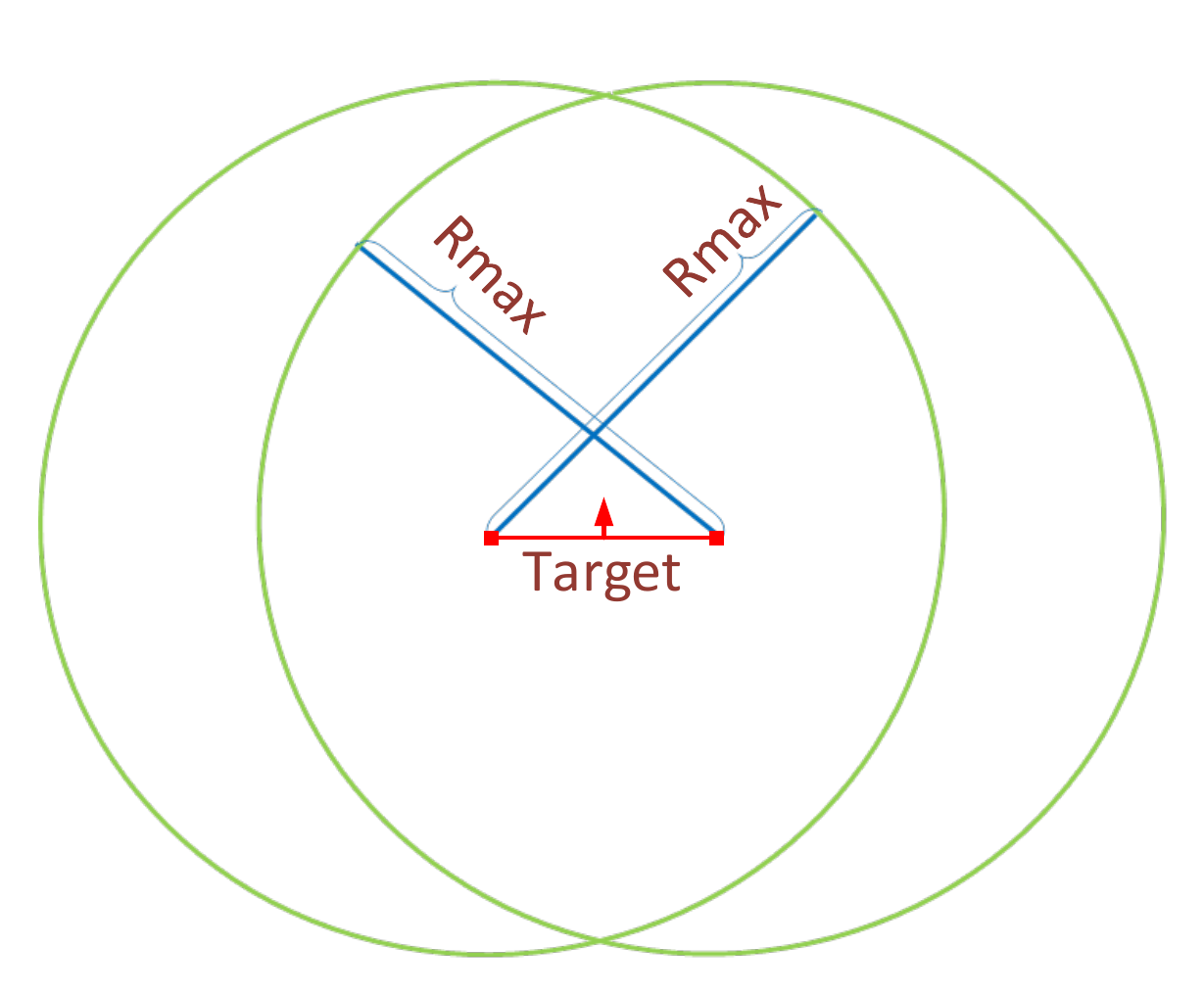}}
\hfill
\subfigure[BCPF for $R_{max}$, AOV and target pose constraints. \label{fig:bcpf01}]{\includegraphics[width=.325\linewidth ]{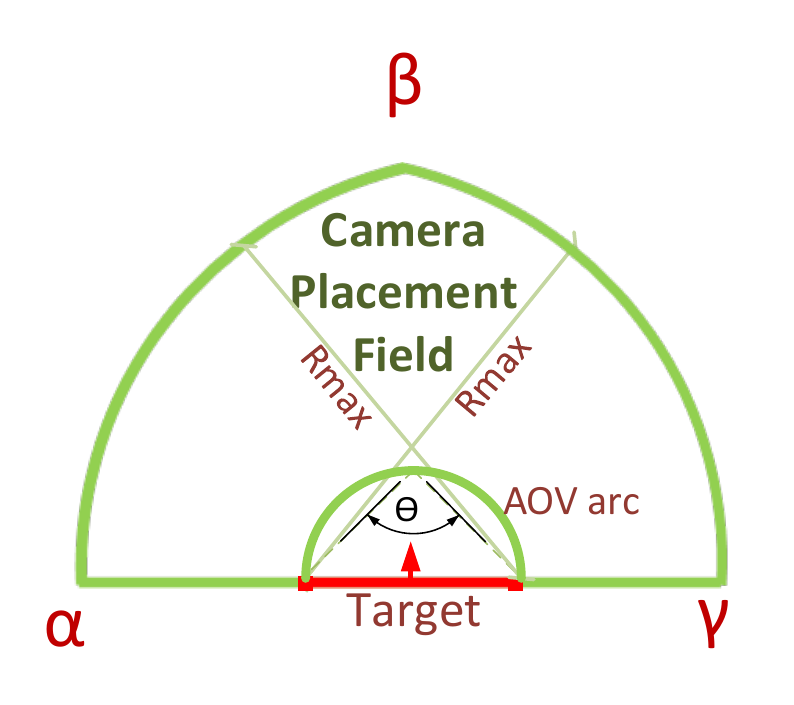}}
\hfill
\subfigure[BCPF parametrized by viewing angle $\phi$ for quality control. \label{fig:bcpf2}]
{\includegraphics[width=.325\linewidth ]{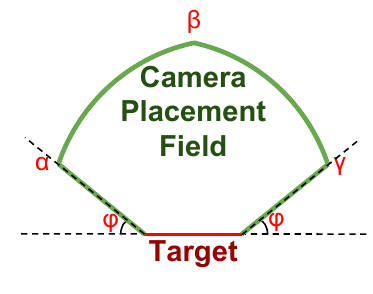}}
\caption{Basic Camera Placement Field (BCPF).}
\label{fig:cpf_def}
\end{figure}

\begin{figure}[!t]
\captionsetup[subfigure]{labelformat=empty}
\centering
\subfigure{\includegraphics[width=.24\linewidth ]{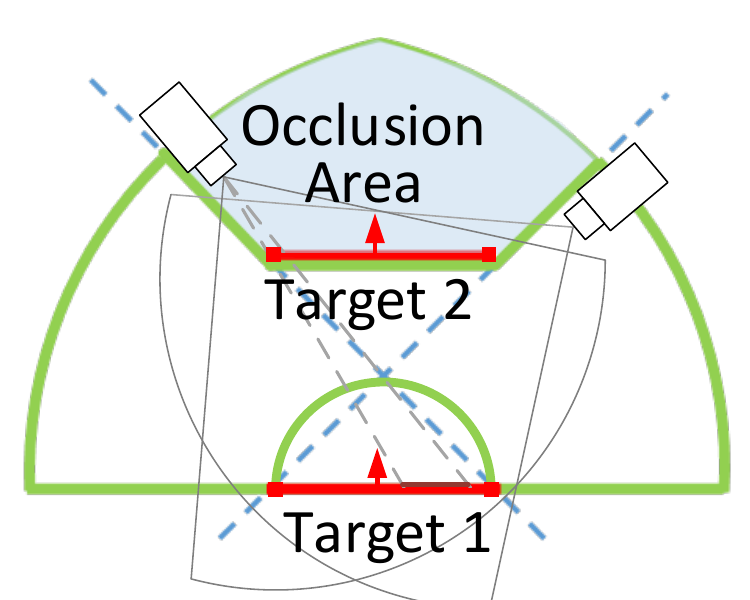}}
\hspace{0.2in}
\subfigure{\includegraphics[width=.24\linewidth ]{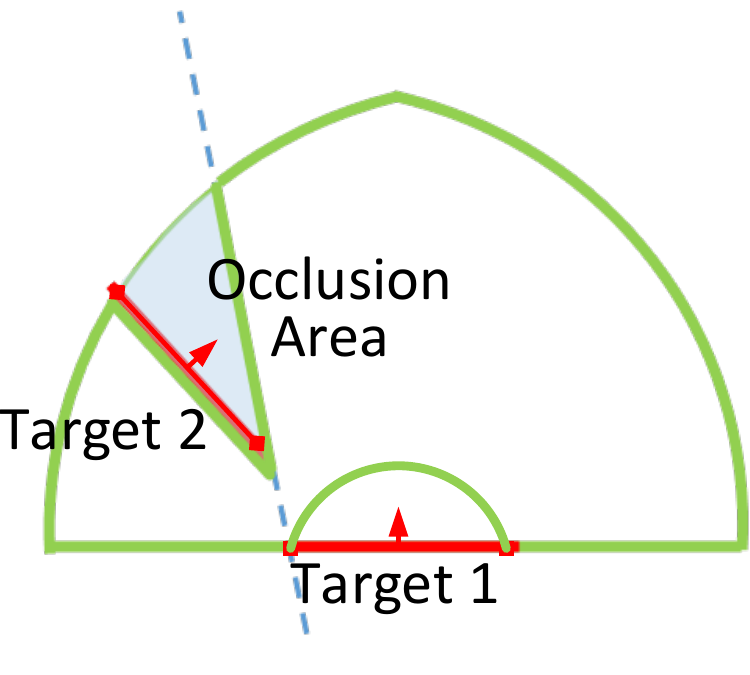}}
\caption{Camera Placement Field (CPF) of Target 1 after excluding areas blocked by some occluder (labeled as Target 2).}
\label{fig:occlusion}
\end{figure}

\subsection{Covering a single target by a single camera}
We aim to determine the region around a target where a camera can be placed and oriented to fully cover this target. We call this region the \emph{Camera Placement Field (CPF)}. It is more convenient to define the CPF by introducing one constraint at a time.

Starting with range constraints, Figure~\ref{fig:cpf_def1} shows how the space around a target is restricted by $R_{max}$ to the intersection of two circles each centered at one end point of the target segment, since target width is $\ll R_{max}$ ($R_{min} = 0$ was used to simplify the figure). Next, for the AOV constraint, we exclude locations too close to the target such that the angle required for full coverage would be larger than $\theta$. The area to exclude is bounded by an \textit{AOV arc} with the target segment as a chord at an inscribed angle of $\theta$. Then, we exclude everything behind the target to account for target pose.

Applying the first three constraints only results in an area we refer to as the \textit{Basic Camera Placement Field (BCPF)}. The BCPF is bounded by three arcs and two line segments as illustrated in Figure~\ref{fig:bcpf01}, which assumes that a camera can cover targets at $90^\circ$ rotations. While some tasks like face detection can still yield high accuracy at $90^\circ$ rotations \cite{chen2008hybrid}, the accuracy of object matching and point matching between two images drop significantly for rotations larger than $45^\circ$ \cite{bay2006surf}. To incorporate notions of quality in the coverage model, the BCPF can be restricted to only include locations within a certain rotation with respect to the target. This is achieved by a controllable parameter $\phi$ that constrains the range of acceptable rotations as illustrated in Figure~\ref{fig:bcpf2}.

Applying the last constraint, if other targets or obstacles intersect the BCPF of the target at hand, it is necessary to exclude the \textit{occlusion area} of all points within the BCPF where any camera cannot provide full coverage of this target. This is obtained by the lines connecting opposite ends of the occluding segment and the target segment as illustrated in Figure~\ref{fig:occlusion}.
The vertices along the boundary of the CPF will be referred to as the \emph{critical points} of the CPF as they play a crucial role in our algorithms.

\subsection{Covering a pair of targets by a single camera}
For a single camera to cover two targets $T_a$ and $T_b$, it must fall in the CPF of each, meaning that camera placement is limited to the Intersection Area (IA) of their CPFs. This guarantees a placement that satisfies range, pose, and occlusion constraints for both targets as shown in Figure~\ref{fig:IA_with_aov}. The AOV constraint, on the other hand, requires for a candidate camera location $x$ and any choice of points $q_a \in T_a$ and $q_b \in T_b$ that $\angle q_a x q_b \leq \theta$. This constraint on camera locations can be conveniently encoded by a pair of \textit{AOV circles} that we define next.

\begin{figure}[!t]
\centering
\subfigure[Construction of AOV circle pairs (solid) using helper circles (dashed). \label{fig:aov_construction}]{\includegraphics[width=.3\linewidth ]{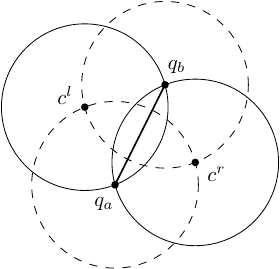}}
\hspace{0.2in}
\subfigure[AOV circles ($\theta < \frac{\pi}{2}$). \label{fig:aov_constraint}]{\includegraphics[width=.25\linewidth ]{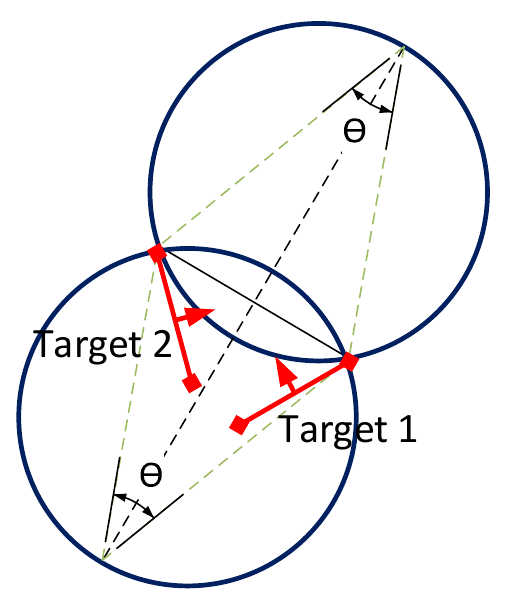}}
\hspace{0.2in}
\subfigure[Intersection Area (yellow) and its restriction by only one of the 4 AOV circle pairs ($\theta > \frac{\pi}{2}$).\label{fig:IA_with_aov}]{\includegraphics[width=.35\linewidth ]{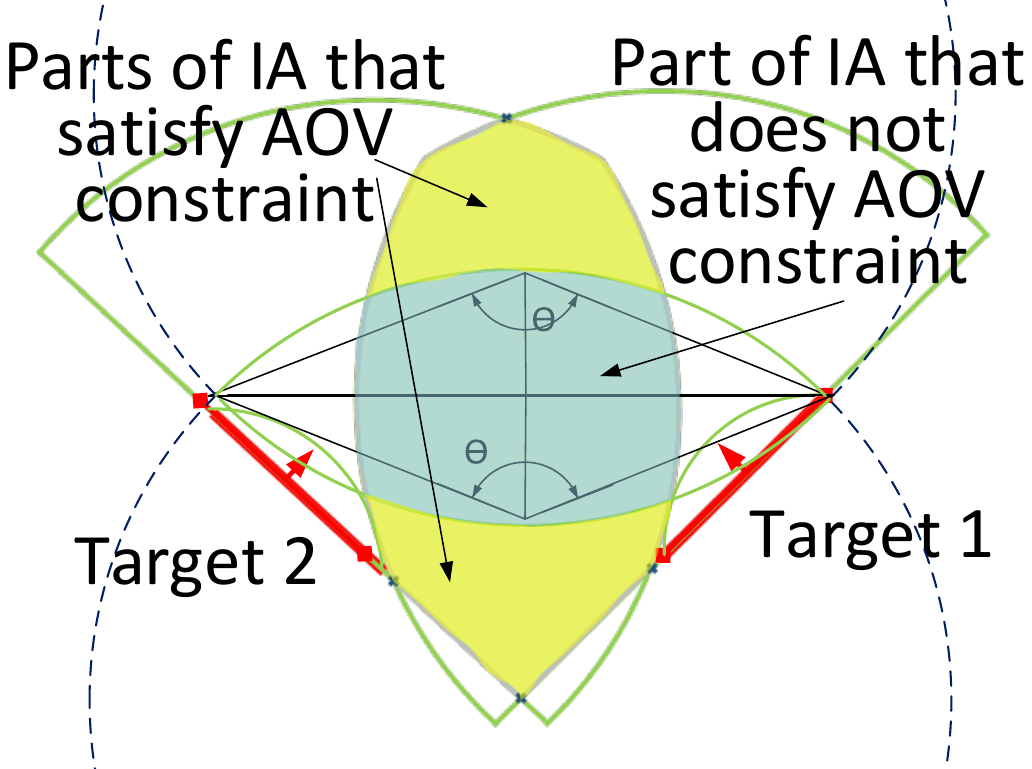}}
\caption{AOV circle pairs and intersection area (IA).}
\label{fig:AOV_and_IAs}
\end{figure}

\begin{definition}[AOV circle pair]
For any pair of points $(q_a, q_b)$, the \textit{AOV circle pair} is the two congruent circles sharing $\overline{q_a q_b}$ as a chord at an inscribed angle equal to the AOV $\theta$.
\end{definition}

For $\theta \leq \frac{\pi}{2}$, we exclude the union of the AOV circle pair while for $\theta > \frac{\pi}{2}$ we exclude their intersection. Note that the camera never lies inside both AOV circles as the intersection is always excluded. Hence, we can use individual AOV circles to enforce one constraint at a time. It is easy to verify that the circles having $\overline{q_a q_b}$ as a chord at an inscribed angle $\theta$ have radius $r_{q_a, q_b}(\theta) = \frac{|q_a q_b|}{2 \sin{\theta}}$. One way to construct the centers of these circles is to compute the two points of intersection for the two helper circles with radius $r_{q_a, q_b}(\theta)$ centered at $q_a$ and $q_b$ as shown in Figure~\ref{fig:aov_construction}.

For each pair of targets $(T_a, T_b)$ we need a set of \textit{AOV circles} to exclude all locations that cannot fully cover both targets simultaneously. We use the four diagonals connecting one end point from each target to generate four AOV circle pairs which can be shown to contain all AOV circles for all pairs of points $(q_a, q_b)$ on the two targets. We defer the formal proof to Appendix~\ref{sec:appendix}.

Intuitively, consider any pair of points $(q_a, q_b)$ on the two targets and a camera location $x$ where a camera is to be placed to cover both $q_a$ and $q_b$ simultaneously. Fixing $x$, observe that each of $q_a$ and $q_b$ can be moved to one of the end points on their respective target segments to make $\angle q_a x q_b$ larger. It follows that if a camera at $x$ can cover all diagonals then it can also cover any such pair of points and consequently the whole two targets. Hence, it suffices to exclude all camera locations that cannot cover any diagonal. Figure~\ref{fig:AOV_and_IAs} shows two examples of AOV circle pairs. {\color{black} In concurrent work, an alternative approach to covering pairs of targets was developed in~\cite{Lino_toric}; using Euler angles to parameterize the set of camera placements yielding a specific on-screen composition of the given pair of targets, the best viewpoint is found by an interval-based search in the parametric space defined by the Euler angles. In contrast, our approach works directly in the Cartesian space of camera placements yielding a more intuitive spatial decomposition which is easier to incorporate with the other constraints we consider for the purposes of coverage as required in surveillance and similar applications.}

\section{Drone Placement: the {\ProblemNameAcro} Solver}
\label{sec:approx}

Deploying drones requires configuring each with a location to move to and a direction to point its camera sensor. There are infinitely many possible configurations spanning every location where a drone can be positioned and every direction it can be covering. In order to get a handle on the problem of drone placement, the key step is to reduce the space of configurations into a small finite subset. The goal of the \textit{\ProblemNameAcro} \textit{Solver} is to compute a set of configurations that covers a given set of \emph{\ModelNameAcro} targets using the minimum number of drones. The \textit{\ProblemNameAcro} \textit{Solver} can be broken down into three modules: 1) A \textit{Spatial Discretizer} responsible for finding a small subset of points to work with, 2) An \textit{Angular Discretizer} that determines the relevant directions to consider at each of the points selected by the \textit{Spatial Discretizer}, and 3) A \textit{Configuration Selector} to pick a subset of the configurations generated by the \textit{Angular Discretizer}.



\subsection{Spatial Discretizer}

The goal of the \textit{Spatial Discretizer} is to generate the candidate locations for camera placement. Each candidate location can be used to view a subset of targets under the coverage model. A key characteristic of the \textit{Spatial Discretizer} is the nature of the set of candidate locations it generates. We define two types of candidate sets: 1) comprehensive and 2) heuristic, denoted by $\mathcal{P}$ and $\hat{\mathcal{P}}$, respectively. \textit{Comprehensive representation of the search space means that the set of candidate locations is guaranteed to include all optimal configurations, up to an equivalence. Two configurations are equivalent with respect to a subset of targets if both configurations can cover these targets under the same constraints.} Heuristic sets are not guaranteed to be comprehensive but are an effective alternative which is also practical as they include fewer locations allowing faster computation of drone configurations at the expense of a potential increase in the number of drones.

Formally, given a comprehensive set of candidate locations $\mathcal{P}$ we are able to obtain an $O(\log{n})$-approximation algorithm. However, generating the $O(N^4)$ candidates required for a comprehensive set can be an overkill and incurs much higher overhead. This, in turn, slows down both the \textit{Angular Discretizer} and \textit{Configuration Selector} as they would have to go through too many candidates. To remedy this, we develop a heuristic spatial discretizer that generates $O(N)$ candidates $\hat{\mathcal{P}}$, enabling the \emph{\ProblemNameAcro} \textit{Solver} to handle larger numbers of targets.


\subsubsection{Comprehensive Spatial Discretization}
\label{sec:exact}


Our objective is to identify candidate locations that comprehensively represent the search space through spatial subdivisions based on target, obstacle, and camera constraints. 

Per Section~\ref{sec:cov}, for a single camera to cover three or more targets, the camera must fall in the IA of all of their CPFs and outside some of their AOV circles. It is clear that any computation on the power set of $\mathcal{T}$, examining all subsets to generate all possible IAs, would take an exponential number of steps. We avoid this paradigm of enumerating IAs explicitly, and only compute discrete representatives for them.

The representatives we compute are the intersection points of the geometric coverage constraints. Note that the vertices along the boundary of any potential region for camera placement to cover a given subset of targets are either critical points of a CPF, intersection points between CPFs, or intersection points between CPFs and AOV circles; we use $\mathcal{P}$ to denote the set of all such vertices. We prove that $\mathcal{P}$ is a comprehensive representation.

\begin{theorem} \label{th:rep}
Given an \emph{\ProblemNameAcro} instance $\langle \mathcal{T}, \theta, R_{min}, R_{max} \rangle$, the set $\mathcal{P}$ contains at least one representative for each feasible coverage configuration for all subsets of $\mathcal{T}$.
\end{theorem}

\begin{proof}
Let $S \subseteq \mathcal{T}$ be a subset of $k$ targets that can be covered simultaneously by a single camera $c$ placed at point $x$. The case where $k=1$ is trivial, since any critical point on the CPF of a single target can be used as a representative for covering that target. Since $\mathcal{P}$ contains all critical points of all CPFs, we are done. For $k \geq 2$, let $A_k$ be the region around $x$ to which $c$ can be moved and rotated accordingly while still being able to cover all $k$ targets in $S$. Clearly, $A_k$ must lie in the intersection of CPFs of all targets in $S$. Otherwise, by the definition of a CPF, at least one of the pose, range ($R_{max}$ and $R_{min}$) or occlusion constraints would be violated for at least one target in $S$, a contradiction. Moreover, $A_k$ must lie outside at least one of the AOV circles generated by all pairs of targets in $S$. Otherwise, by the definition of an AOV circle pair, $c$ would not be able to simultaneously cover at least two of the targets in $S$ by an AOV $\theta$, again a contradiction. We may therefore think of $A_k$ as a region enclosed in a set of CPFs with some parts taken out by a set of AOV circles. This implies that $A_k$ is bounded by at least one CPF and possibly some AOV circles. This allows us to describe $A_k$ by the curves outlining its boundary and their intersection points. Regarding $A_k$ as the equivalence class of points where a camera can be placed to cover $S$, any of these intersection points can serve as a representative. As there is at least one CPF boundary for $A_k$, these intersection points must contain either an intersection point of two CPFs or an intersection point of a CPF and an AOV circle. By construction, $\mathcal{P}$ contains all such intersection points.
\end{proof}

We consider the complexity of generating $\mathcal{P}$. Letting $N = n + u$, each CPF can be represented by up to $O(N)$ pieces as all other $n$ targets and $u$ obstacles can split the BCPF into several parts. Thus, the operation of intersecting two CPFs is $O(N^2)$ and performing this operation pairwise for all targets is $O(n^2N^2)$ (See Figure~\ref{fig:occlusion}). The operation of intersecting a CPF with an AOV circle is $O(N)$, and is repeated $O(n^2)$ times for all AOV circles resulting in an $O(n^2N)$ operation per target. Hence, repeating this operation $O(n)$ times takes $O(n^3N)$. This amounts to a total of $O(n^2(n^2 + nu + u^2))$. We relax this expression to $O(N^4)$ and loosely bound $|\mathcal{P}| = O(N^4)$.

\subsubsection{Heuristic Spatial Discretization}
\label{sec:heuristic}
The $O(N^4)$ candidates generated by the approximation algorithm are too demanding for real-time applications. On top of that, we can still produce good solutions using far fewer candidates at the cost of missing tightly packed configurations corresponding to small intersection areas. The reason is that each additional target further restricts the region of space where cameras can be placed to cover the set of targets simultaneously. In practice, such configurations are neither robust to errors in target localization and drone navigation nor stable enough to capture the anticipated views before targets move apart. 
This motivates a more efficient and robust approach to the generation of candidates. We propose the \textit{Basic Camera Placement Field (BCPF) sampling}. 

An intuitive approach to yield $O(n)$ candidate locations is to sample a constant number of points per target taking occlusion into account. However, an easy first order relaxation is that any camera placement covering a given target must fall in its BCPF of that target (Figure~\ref{fig:bcpf01}). The advantage of using the BCPF instead of the actual CPF, is that BCPF can be computed in $O(1)$ per target compared to $O(N)$ for the CPF. Once the BCPF is known, uniformly sampling its interior should capture most of the useful configurations. Note that the intersection of multiple BCPFs gets sampled proportionally which favorably reduces the probability of missing good candidate points. However, as suggested by our simulations with uniformly random target where adversarial arrangements are unlikely, it suffices to sample points along the boundary of the BCPF. Letting $\rho$ be the sum of the central angles of the two BCPF arcs and the apex angle of the triangle in between, we can fix suitable BCPF sampling steps $\epsilon_a$ and $\epsilon_r$ for the angular and radial axes, respectively. With that, we generate $O(\frac{\rho \cdot R_{max}}{\epsilon_a \cdot  \epsilon_r} \cdot n)$ candidate locations that we call $\hat{\mathcal{P}}$. Our experiments show the promise of this almost agnostic approach to candidate generation as it is able to match the quality of the approximation algorithm while being much faster.

\subsection{Angular Discretizer}
\label{sec:sweeping}

Once a camera is placed at a given location $x$ from either $\mathcal{P}$ or $\hat{\mathcal{P}}$, we need to determine the relevant viewing directions (VDs) to consider. We achieve this in two stages: First, we perform an angular sweep to identify one representative VD for each subset of targets that can be covered simultaneously from the location in question. Then, we optimize representative VDs for better footage quality.

\textbf{Angular sweep:} This step identifies a set of representative VDs $sweep(x) = \{\hat{\alpha}_1, \hat{\alpha}_2, \dots\}$ for each maximal subset of targets that can be covered simultaneously by a camera placed at $x$. Each such maximal subset can be covered by a range of viewing directions $[\alpha^l_i, \alpha^h_i]$. The application may specify a criteria for selecting the best direction from this range. As a default setting, we use $\hat{\alpha}_i = (\alpha^l_i + \alpha^h_i) / 2$. Let $cov(x, \alpha)$ denote the maximal subset of targets covered by a camera at $x$ when its VD is set to $\alpha$. Observe that if we perform a radial sort around $x$ of the end points of all target segments visible from $x$, no two targets overlap. Given the radial sort of all end points, we can easily determine which targets are visible by discarding segments interrupted by a closer point and enumerate $sweep$ in $O(N)$. The radial sort can easily be found in $O(N \log{N})$. Alternatively, a \textit{visibility diagram} for the set of segments can be constructed in $O(N^2)$ \cite{Vegter1990}. Using the diagram, $sweep$ queries take $O(N)$.

\textbf{Viewing direction optimization:} Ideally, surveillance footage should provide clear frontal views by an assignment of cameras to targets with each camera-target pair nearly facing one another. This easily breaks down when the camera's viewing direction is not directly towards the target. Given a candidate location for camera placement $x$, each maximal subset of targets $cov(x, \hat{\alpha}_i)$ may be covered by any viewing direction $\alpha \in [\alpha^l_i, \alpha^h_i]$. Within this range, one extreme might favor certain targets placing them right at the center of the FOV, while other targets barely fit at the side. Depending on the spread of these targets and the direction each of them is facing, a camera positioned at $x$ can adjust its VD to obtain the best views possible. A natural objective is to minimize the \textit{deviation}, defined as the angle between the camera's VD and the line-of-sight from $x$ to the target's midpoint. Let $d(x, \alpha, T_j)$ denote the deviation for target $T_j$ when viewed by a camera at $x$ with VD $\alpha$. With that, we seek to minimize the total deviation over all targets $f_1(x, \hat{\alpha}, \alpha) =  \sum_{T_j \in cov(x, \hat{\alpha})}d(x, \alpha, T_j)$. The optimal VD $\alpha^\ast$ can then be chosen as $\argmin_{\alpha \in [\alpha^l, \alpha^h]} f_1(x, \hat{\alpha}, \alpha)$. Alternatively, we may choose to minimize the worst deviation for any one target $f_\infty(x, \hat{\alpha}, \alpha) =  \max_{T_j \in cov(x, \hat{\alpha})}d(x, \alpha, T_j)$.

{\color{black}
\textbf{Remark:} By design, the \textit{Angular Discretizer} is restricted to the candidate locations returned by the \textit{Spatial Discretizer}. However, recall that such candidate locations are merely suggested as witnesses that certain subsets of targets can be covered by a single camera. It is possible to choose better locations to cover a given subset of targets than the representative location provided by the \textit{Spatial Discretizer}. This is further discussed in Section \ref{sec:open}.}

\subsection{Configuration Selector}
With the output of the \textit{Angular Discretizer} as the set of configurations $\mathcal{R} = \{ cov(x, \hat{\alpha}) \:|\: x \in \mathcal{P}, \hat{\alpha} \in sweep(x) \}$, our goal is to find a \textit{minimum set cover} which is a subset $\mathcal{R}_{opt} \subseteq \mathcal{R}$ whose union is $\mathcal{T}$ with $|\mathcal{R}_{opt}|$ minimized. Using the standard \verb+greedy+ approximation scheme, we compute a cover $\mathcal{R}_{greedy}$ with a guaranteed bound $\frac{|\mathcal{R}_{greedy}|}{|\mathcal{R}_{opt}|} = O(\log{|\mathcal{T}|})$ \cite{chvatal1979greedy}. In each round, the algorithm greedily picks the set that covers the largest number of uncovered targets, updates the sets and repeats until all targets are covered. Using the notion of \textit{deviation} we used for optimizing the viewing direction per candidate location, we can also rank different candidate locations according to the quality of coverage they can offer. At iteration $i$, among all candidates $\{(x, \hat{\alpha})\}$ that can cover the maximum number of targets, we favor the one achieving the minimum $f_1(x, \hat{\alpha}, \alpha^\ast)$. The greedy algorithm will then return a coverage scheme providing better views while still approximating the minimum number of cameras needed.

To obtain an $O(\log{n})$-approximation, the comprehensive set of candidates $\mathcal{P}$ is used. As sweeping over $\mathcal{P}$ to generate $\mathcal{R}$ takes $O(N^5)$ steps, we loosely bound the time complexity of the proposed approximation algorithm by $O(nN^5)$. Similarly, defining a set of configurations $\hat{\mathcal{R}}$ using $\hat{\mathcal{P}}$ from the heuristic spatial discretizer results in an $O(\frac{\rho \cdot R_{max}}{\epsilon_a \cdot  \epsilon_r} \cdot n^2N)$ algorithm. 






\section{Implementing Argus}
\label{sec:implementation}

\begin{figure}[!t]
\centering
	\includegraphics[width=0.6\linewidth]{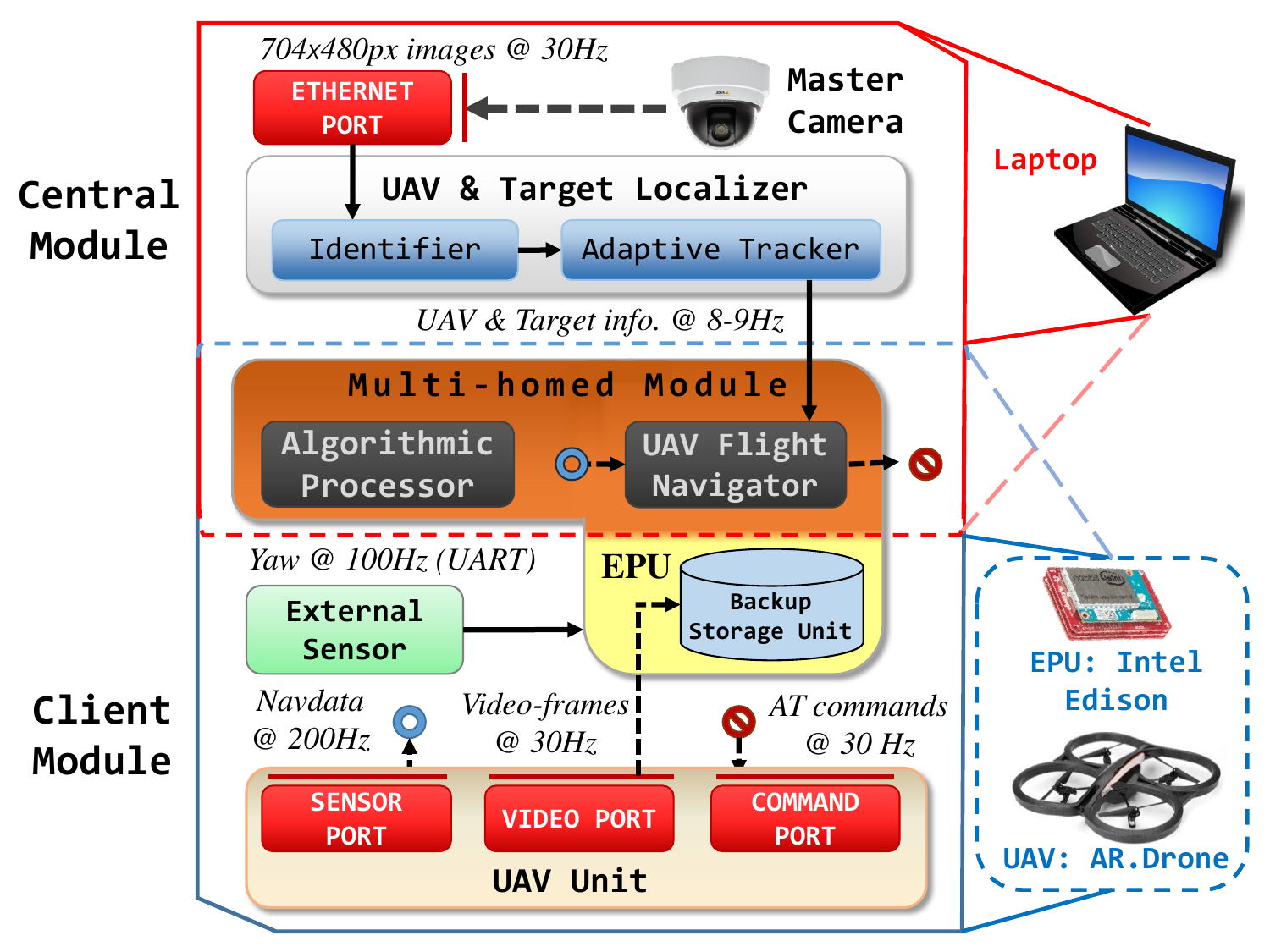}
	\caption{Architecture of the Argus prototype.}
	\vspace{-0.15in}
	\label{fig:system_arch}
\end{figure}

Our goal is to develop a fully autonomous instance of Argus to measure the overhead of the \textit{\ProblemNameAcro} \textit{Solver} under realistic conditions. We build upon our earlier work on developing an autonomous testbed for multi-drone experiments \cite{una,khan2016simulating}. Figure \ref{fig:system_arch} depicts the architecture of the Argus prototype that we fully implement as three modules: \textit{Central}, \textit{Client}, and \textit{Multi-homed}.

\textit{The Central Module} is responsible for localizing quadcopters and targets in 2D and running the \textit{\ProblemNameAcro} \textit{Solver}. The \textit{\ProblemNameAcro} \textit{Solver} runs only the \textit{BCPF Sampling} algorithm as it is more efficient while being competitive to the approximation algorithm.  In our setup, the \textit{Central Module} is run on a Lenovo ThinkPad Y50. The \textit{Central Module} uses a master camera to obtain the input for its \textit{UAV and Target Localizer} component. We use an Axis 213 PTZ network camera located directly above the testbed area. The master camera is connected to the \textit{Central Module} through an Ethernet cable and provides images at a frequency of 30~$Hz$. The \textit{UAV and Target Localizer} filters the noise and locates all quadcopters and targets in the image. Each image is then passed to the \textit{Adaptive Tracker} which makes use of the last known location of each drone or target to localize it in the scene. This approach reduces the processing time of the localization step by performing local searches in the image for drones and targets.

\textit{A Client Module} is the mobile camera component of the system. We choose a quadcopter platform for its low-cost, small size, and maneuverability even in small spaces. In particular, we use the Parrot AR.Drone 2.0 \cite{krajnik2011ar} which is equipped with an ARM processor running an embedded Linux 2.6.32 BusyBox. The Parrot AR. Drone 2.0 is also equipped with two cameras: a front 720p camera with a 93$^\circ$ lens and a vertical QVGA camera with a 64$^\circ$ lens. We mainly use the front camera in our experiments. We allow the client to add as many sensors as needed which can help obtain more surveillance information (e.g., depth sensors) or better navigate the drone (e.g., accelerometers). To this end, we use an External Processing Unit (EPU) which collects recorded video from the camera and sensory readings from the external sensors. Communication between the drone and the EPU is performed over Wi-Fi.


For the EPU, we use Intel Edison which is an ultra-small computing device powered by an Atom system-on-chip dual-core CPU at 500~$MHz$ and 1~$GB$ RAM. Intel Edison has integrated Wi-Fi, Bluetooth, 50 multiplexed GPIO interfaces, and runs Yocto Linux. The EPU is powered by a Battery Block. Additional sensors are hardwired into the EPU using an Arduino block. We use an Inertial Measurement Unit (IMU) as the external sensor in this setup. The IMU improves the autonomous navigation of drones by providing finer grain yaw angles to help with drone orientation. Another benefit of mounting EPUs on quadcopters is the extra processing power and added flexibility they offer. We can install our own drivers, operating systems, integrated sensors, and overcome the typical closed-nature restriction of off-the-shelf quadcopters. We attach the EPU on top and close to the center of gravity of the drone to avoid disturbing the balance and stability of the vehicle. The EPU setup is shown in Figure \ref{fig:edi-setup}.



\textit{The Multi-Homed Module} is a special set of sub-modules that can belong to either the \textit{Central} or \textit{Client} module. The flexibility of housing its sub-modules allows easy migration between a centralized and a distributed platform. We use two such sub-modules: \textit{UAV Flight Navigator} and \textit{Algorithmic Processor}. The \textit{UAV Flight Navigator} receives a set of parameters from the \textit{UAV Localizer} (i.e., 2D coordinates) and the \textit{IMU} (i.e., yaw angles) and controls the drone through its navigation parameters to properly fly to the desired coordinates. The \textit{Algorithmic Processor} handles any sensory information processing (i.e., \textit{Fine Grain Context Detector} functionalities).		

\begin{figure}[!t]
\centering
	\includegraphics[width=0.3\linewidth]{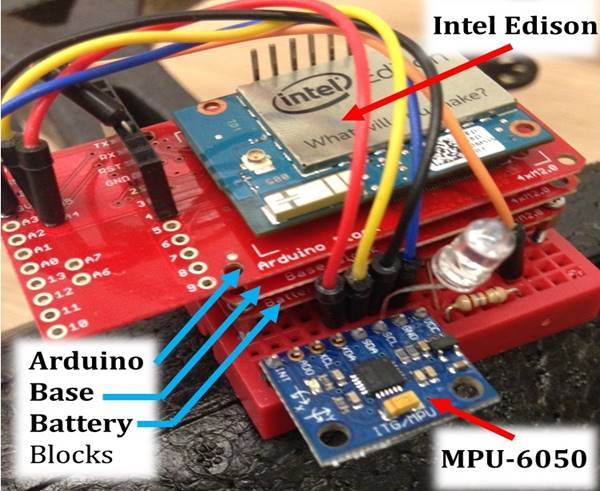}		
	\caption{EPU setup.}
	\label{fig:edi-setup}
	
\end{figure}

\begin{figure}[!t]
\centering
\includegraphics[width=0.6\linewidth ]{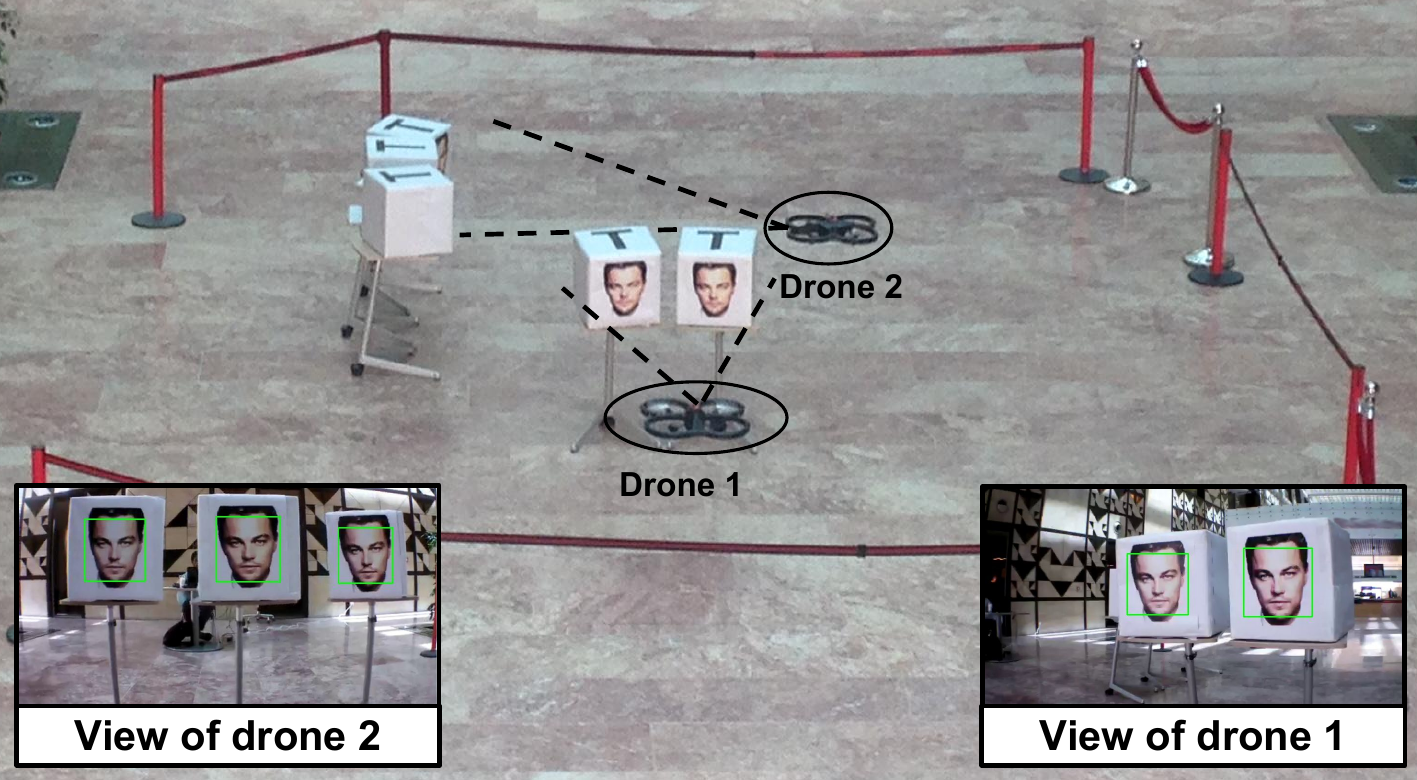}
\caption{Experimental setup: target layout, drone configurations and captured images.}
\label{fig:testbed}
\end{figure}

\textbf{Experimental setup:} The testbed covers an area of 30~$m^2$ where we place one to five synthetic targets positioned in configurations that require a maximum of two drones (Figure~\ref{fig:testbed}); for scenarios with one to three targets, we only need one drone for coverage, and for scenarios with four or five targets, we need two drones. Our target apparatus is a white box mounted on top of a podium with a printed face attached to one of its vertical sides to represent the significant perspective. A letter ``T'' on the top side of the box helps simplify location and pose estimation.

Noting that drone control and sensory information analysis are processing-intensive operations, we aim to achieve real-time processing with minimal latency. To this end, we handle the autonomous control of drones on the \textit{Central Module} and distribute the processing of video feeds from each drone under the \textit{Client Module} running on the drone's EPU to detect faces on the sides of targets. We set $R_{max} = 2\:m$ and $\theta = 75^\circ$, which is slightly smaller than the camera's actual AOV, to avoid cases where covered targets barely fit in the captured frame.

\begin{table}
\renewcommand{\arraystretch}{1.1}
\centering
{
\begin{tabular}{c|c|c|c}
 & Horizontal Motion & Rotation & Hovering \\
\hline
Power (Watt) & $65.625$ & $68.750$ & $61.125$\\
\hline
Energy per meter (Joule/m) & $65.630$ & $68.750$ & N/A\\
\end{tabular}
}
\caption{Energy consumption of typical drone maneuvers measured over time (i.e., power) and distance traveled.}
\label{table:energy}
\vspace{-0.25in}
\end{table}

\textbf{Real-time adaptation to target mobility:} Argus needs to repeatedly invoke the \textit{\ProblemNameAcro} \textit{Solver} to respond to updates in the locations of either targets or obstacles. As shown in Section~\ref{sec:evaluation}, the algorithm can take up to a few seconds based on the number of targets. Until a drone is assigned a new configuration, decisions have to be made locally by each drone to respond to target mobility in real-time. Argus allows drones to hover in place or move horizontally for short distances to maintain target coverage using standard tracking algorithms \cite{kim2013unmanned}. Local decisions, based on the energy footprint of each maneuver, are computed on the EPU to minimize the cost of the proposed strategy. Table~\ref{table:energy} summarizes the power consumption of typical drone maneuvers. To avoid large rotations or displacements, drones cooperate to keep targets in view \cite{hausman2015cooperative}. This autonomous behavior also serves as a fallback strategy if the communication link between the \textit{Central Module} and the drone is broken.

\section{Evaluation}
\label{sec:evaluation}


{\color{black}
We aim to assess the performance of a working instance of Argus in real-time and verify the efficiency of the proposed algorithms. To this end, we demonstrate the advantages of the \ModelNameAcro{} model, compared to the traditional model of targets as mere points, through the prototype we implement per Section~\ref{sec:implementation}. In particular, we analyze the overhead of the \textit{\ProblemNameAcro} \textit{Solver} within the system and establish the feasibility of adopting this enhanced model in a real surveillance system. The prototype we employ for this evaluation leverages typical hardware that can be found at most research labs and is comparable in scale to the experiments reported on closely related systems~\cite{ETH, cin_director}. In addition, we present a set of large scale simulations that compare the performance of the proposed algorithms against a baseline and establish the sampling heuristic as the method of choice, which we employ on the prototype.
}


\subsection{Argus Evaluation}
\label{sec:system_evaluation_by_numbers}

We demonstrate the pitfalls of traditional target coverage algorithms, where target size and pose are not taken into account \cite{ours_mobiwac}, by comparing them to Argus in a realistic setting. Then, we break down the delays in the presented system and compare against the delay introduced by the \textit{\ProblemNameAcro} \textit{Solver}.

\textbf{{\ModelNameAcro} vs. blips on the radar:} To demonstrate the advantages of the proposed model, we take for example the surveillance footage in Figures~\ref{fig:drone1}, \ref{fig:drone2}, and \ref{fig:top_camera}. Recall that these images are captured by the master camera and the front cameras on each drone. We choose this particular target configuration to put the quality of coverage of a typical target coverage algorithm in contrast with \textit{\ModelNameAcro}. Figure~\ref{fig:cs_direction_f} shows two targets covered from the opposite direction of their significant perspective because typical coverage algorithms do not take target pose into account. Moreover, typical target coverage algorithms do not take target size and potential occlusions between targets into account, which is demonstrated in Figure~\ref{fig:cs_occlusion_f} where one target occludes two other targets. When \textit{\ModelNameAcro} is employed, these issues are resolved and cameras are positioned to properly cover the targets as shown in Figures~\ref{fig:cs_direction_t} and \ref{fig:cs_occlusion_t}. Note that the generated configurations are based on target width and camera constraints (e.g., $R_{max}$ of 2~$m$) which represents a constraint on the quality of images used for face detection. 


\begin{figure*}[!t]
\centering
\begin{minipage}[b]{0.31\linewidth}
\centering
\subfigure[Targets may be covered from behind.\label{fig:cs_direction_f}]{\includegraphics[width=1\linewidth ]{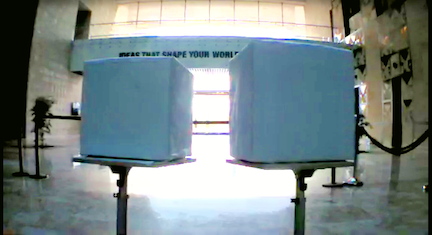}}
\subfigure[Target pose is taken into account.\label{fig:cs_direction_t}]{\includegraphics[width=1\linewidth ]{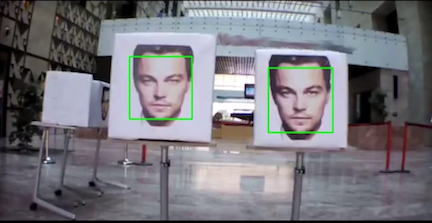}}
\caption{Comparing the view from Drone 1 under a typical target coverage algorithm (top) and {\ProblemNameAcro} (bottom).}
\label{fig:drone1}
\end{minipage}
\quad
\begin{minipage}[b]{0.31\linewidth}
\centering
\subfigure[Targets may occlude one another.\label{fig:cs_occlusion_f}]{\includegraphics[width=1\linewidth ]{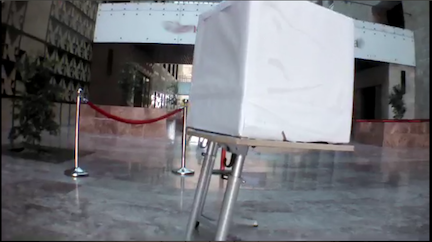}}
\subfigure[Potential occlusions are taken into account.\label{fig:cs_occlusion_t}]{\includegraphics[width=1\linewidth ]{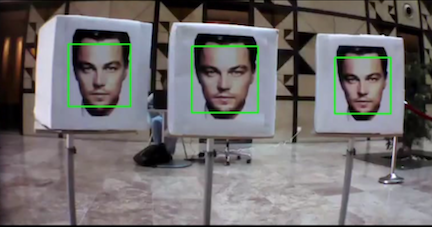}}
\caption{Comparing the view from Drone 2 under a typical target coverage algorithm (top) and {\ProblemNameAcro} (bottom).}
\label{fig:drone2}
\end{minipage}
\quad
\begin{minipage}[b]{0.31\linewidth}
\centering
\subfigure[Drones cover targets from wrong angles. \label{fig:fail_top}]{\includegraphics[width=0.8\linewidth ]{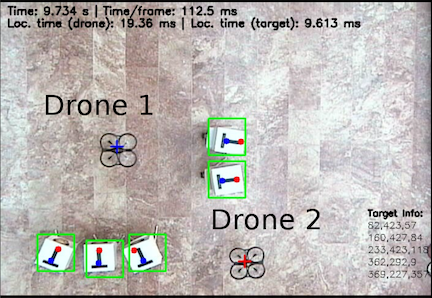}}
\subfigure[Drones properly cover all targets.\label{fig:correct_top}]{\includegraphics[width=0.8\linewidth ]{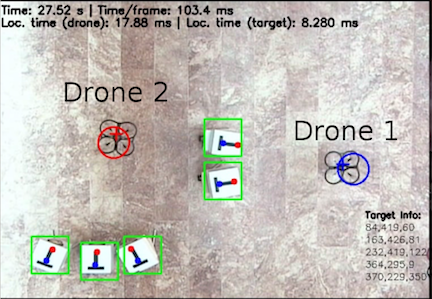}}
\caption{Top views from the master camera showing drone configurations corresponding to Figures \ref{fig:drone1} and \ref{fig:drone2}.}
\label{fig:top_camera}
\end{minipage}
\end{figure*}

\begin{figure}[!t]
\centering
\begin{minipage}[b]{0.45\linewidth}
\centering
\includegraphics[width=1\linewidth ]{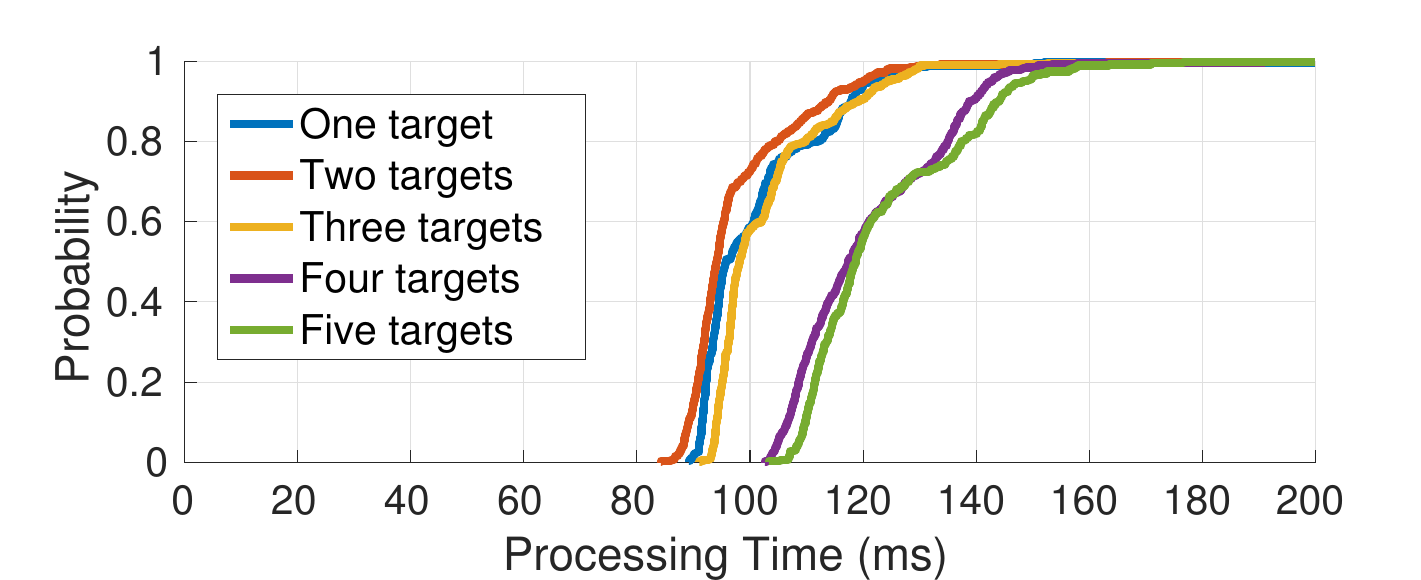}
\caption{CDF of processing time per frame including image fetching, target and drone localization, and drone communication.}
\label{fig:cdf_all_proc}
\end{minipage}
\hfill
\begin{minipage}[b]{0.24\linewidth}
\centering
\includegraphics[width=1\linewidth ]{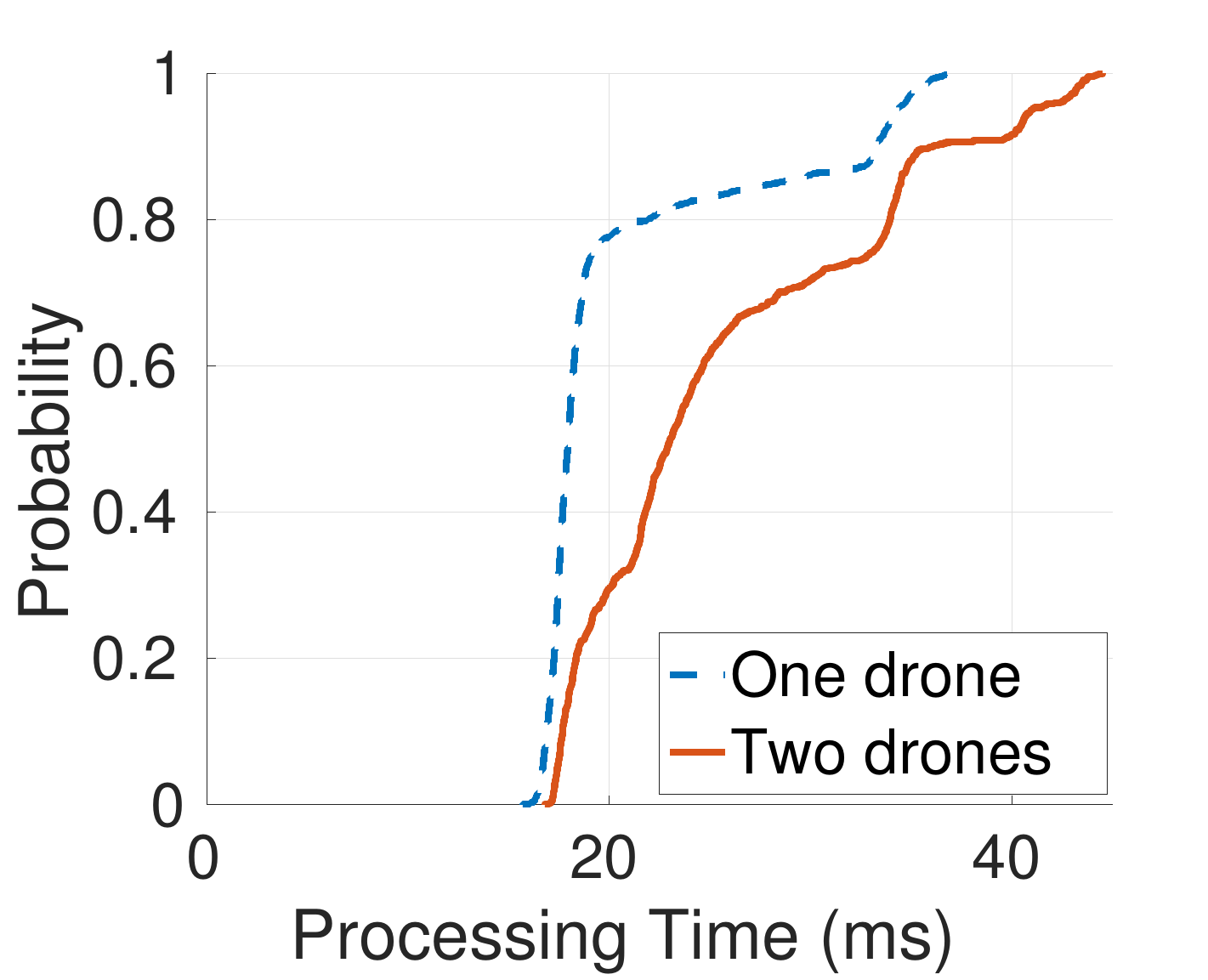}
\caption{CDF of processing time of the \textit{UAV Localizer}.}
\label{fig:cdf_drone}
\end{minipage}
\hfill
\begin{minipage}[b]{0.24\linewidth}
\centering
\includegraphics[width=1\linewidth ]{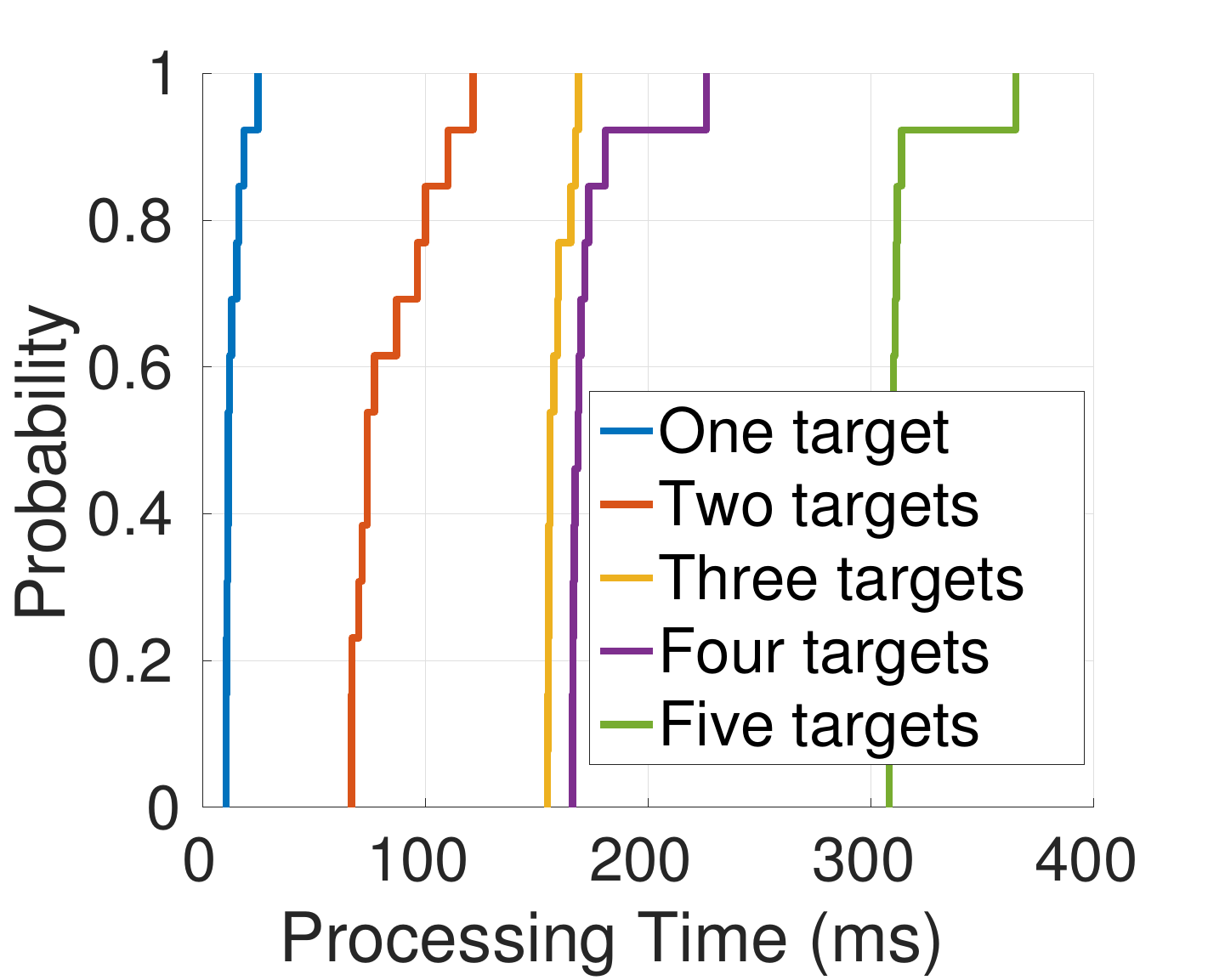}
\caption{CDF of processing time of the \textit{\ProblemNameAcro} \textit{Solver}.}
\label{fig:cdf_alg}
\end{minipage}
\end{figure}

\textbf{Implementation delay breakdown:} Figure~\ref{fig:cdf_all_proc} shows the CDF of the processing time per frame, which captures the overall processing performed by the \textit{Central Module} apart from the \textit{\ProblemNameAcro} \textit{Solver}. This processing includes image fetching, decoding and preprocessing, target localization, drone localization and communication. Our target apparatus can be detected efficiently within a few milliseconds. This reduces the processing time per frame as complex targets would take longer to detect (e.g., 120~$ms$ per frame for human body pose estimation \cite{flohr2015probabilistic}).

The difference in processing time per frame for one to three targets and four and five targets is dominated by the overhead of handling the extra drone. This added overhead can be seen in the CDF of localization time in Figure~\ref{fig:cdf_drone}. Recall that when the drone makes large displacements, locality over consecutive frames is lost. This occasionally forces the algorithm to search the entire frame, resulting in the skewed shape of the CDF observed in Figure~\ref{fig:cdf_drone}. In our experiments, the \textit{UAV Localizer} has to be invoked at a minimum frequency of 8~$Hz$ for smooth control of the drone.

\textbf{{\ProblemNameAcro} as a component of a surveillance system:} We compare the processing time per frame, which corresponds to the overhead of the \textit{Coarse Grain Context Detector}, to the overhead of the \textit{\ProblemNameAcro} \textit{Solver} (Figure~\ref{fig:arch}). Figure~\ref{fig:cdf_alg} shows the CDF of the processing time of the \textit{\ProblemNameAcro} \textit{Solver} for the number of targets in our tests. The solver is implemented in \verb|MATLAB| and we expect it can be significantly optimized. Still, with five targets, the solver can be invoked once for every three processed frames. As mentioned in the previous section, several techniques can be exploited to maintain target coverage while the solver is running. This task is made easier by the ability to invoke the solver at a relatively high frequency (i.e., one third the frequency of updates in the input parameters).


\subsection{Argus at Scale}\label{sec:scale}

We evaluate, through \verb|MATLAB| simulations, the performance of the proposed coverage algorithms under large scale conditions that we cannot test on the prototype. We compare the performance of the approximation algorithm to the BCPF sampling heuristic with two levels of granularity for angular sampling using an $\epsilon_a$ of 0.01 and 0.1 $rad$ and an $\epsilon_r$ of $R_{max}$.

As a baseline for comparison, we present a \textbf{grid sampling heuristic}. We use a simple discretization of the search space: a uniform grid of $\epsilon \times \epsilon$ cells. As $\epsilon \to 0$, grid points would hit all possible intersection areas of target CPFs. If $w \times h$ are the dimensions of the bounding box of $\mathcal{T}$, the number of grid points will be $O(\frac{w \cdot h}{\epsilon^2})$, but is otherwise independent of $|\mathcal{T}|$. Treating these points as candidate locations, we generate representative coverage configurations at each point by an angular sweep before running the greedy selection scheme, which amounts to a runtime of $O(\frac{w \cdot h}{\epsilon^2} \cdot nN)$. We use this naive approach to verify the effectiveness of our proposed method in finding appropriate candidate points to minimize the number of cameras needed. To do so, we use relatively small instances of \emph{\ProblemNameAcro} such that $\epsilon$ need not be too small and the runtime and memory requirements of the grid heuristic are feasible.

\begin{figure}[!t]
\centering
\subfigure{\includegraphics[width=0.35\linewidth ]{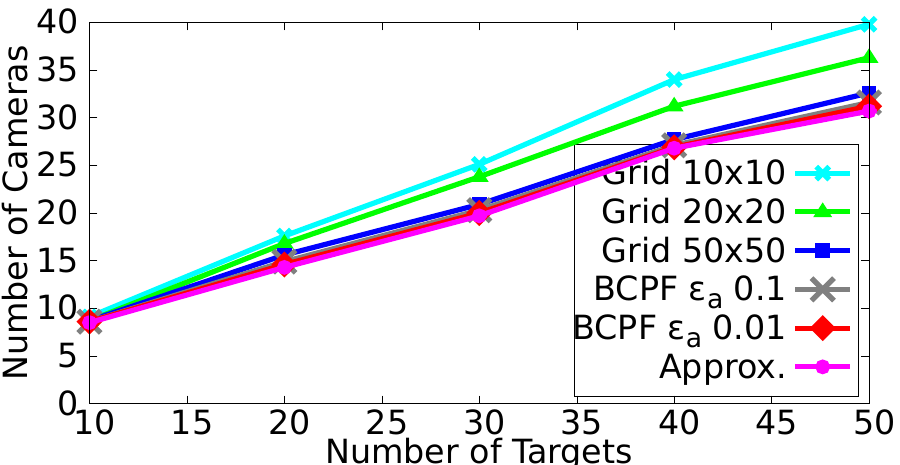}}
\subfigure{\includegraphics[width=0.35\linewidth ]{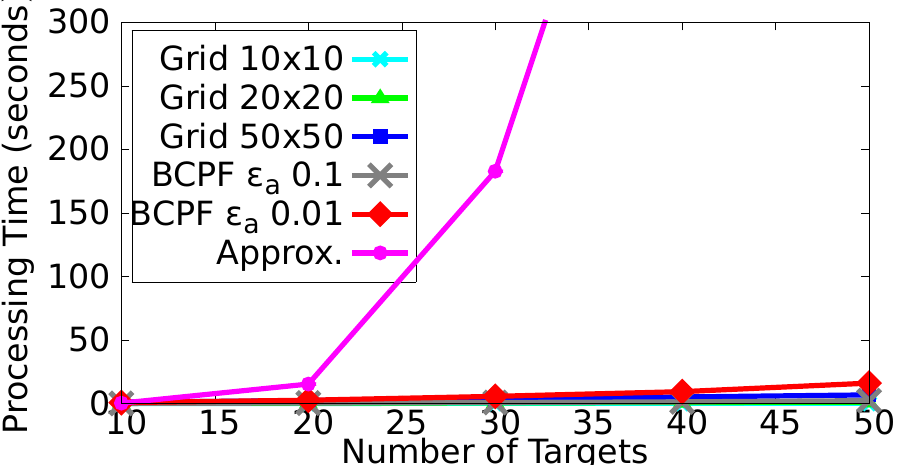}}
\caption{Comparing the performance of all algorithms for increasing numbers of targets.}
\label{fig:target_perf}
\end{figure}

We evaluate the algorithms in the extreme case where all present objects are targets (i.e., no obstacles). Since both targets and obstacles act as occluders while only targets need to be covered, this setup requires maximal computations for the chosen number of objects. The goal of this evaluation is to show the effect of changing the number of targets, range, and AOV on the number of drones and processing time required to perform the coverage task under the proposed model. Targets are placed at random locations with random poses over the area of interest such that they do not overlap. The default values of simulation parameters are shown in Table~\ref{table:pars} for both small and large scenarios. We use small scenarios to evaluate the approximation algorithm, which suffices to show the advantages of sampling, and use larger scenarios to compare the different heuristics. We use three resolutions of grid sampling: Grid 10x10 (sparse), Grid 20x20 (medium) and Grid 50x50 (dense) for $\epsilon$ set to 10~$m$, 5~$m$ and 2~$m$, respectively.

\begin{table}[!t]
\renewcommand{\arraystretch}{1.1}
\centering
{
\begin{tabular}{c|c|c}
Parameter & Range & Nominal value\\
\hline
Dimensions & $100\:m\times 100\:m$ & $100\:m\times 100\:m$\\
\hline
Target Width & 1~$m$ & 1~$m$\\
\hline
AOV & $40^{\circ}$ - $140^{\circ}$ & $100^{\circ}$\\
\hline
Target count & 10 - 140 & 30 (small), 80 (large)\\
\hline
$R_{\max}$ & 10~$m$ - 50~$m$ & 20~$m$ (small), 30~$m$ (large)\\
\end{tabular}
}
\caption{Simulation parameters.}
\label{table:pars}
\vspace{-0.25in}
\end{table}


\begin{figure*}[!t]
\centering
\begin{minipage}[b]{0.32\linewidth}
\centering
\subfigure{\includegraphics[width=1\linewidth ]{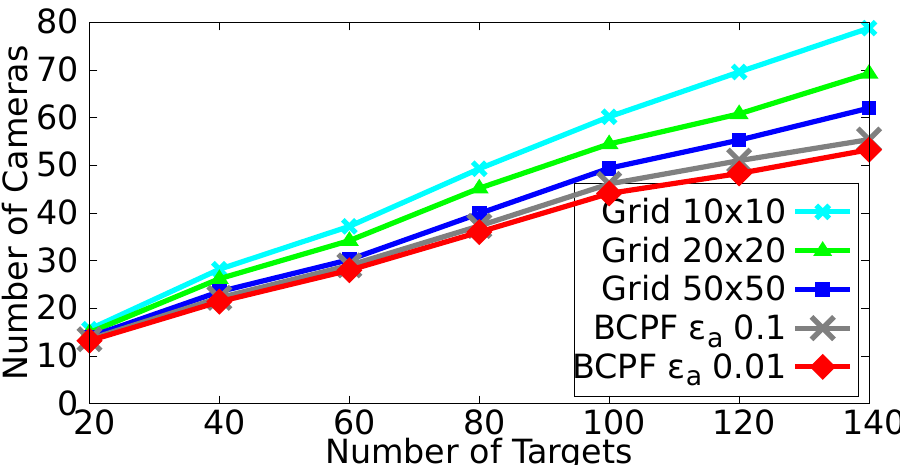}}
\subfigure{\includegraphics[width=1\linewidth ]{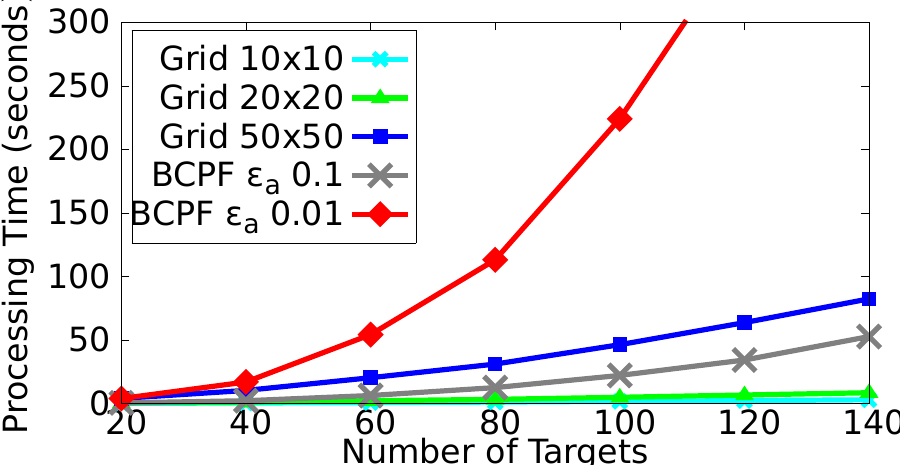}}
\caption{Comparing all heuristics for increasing numbers of targets.}
\label{fig:target_perf_more}
\end{minipage}
\hfill
\begin{minipage}[b]{0.32\linewidth}
\centering
\subfigure{\includegraphics[width=1\linewidth ]{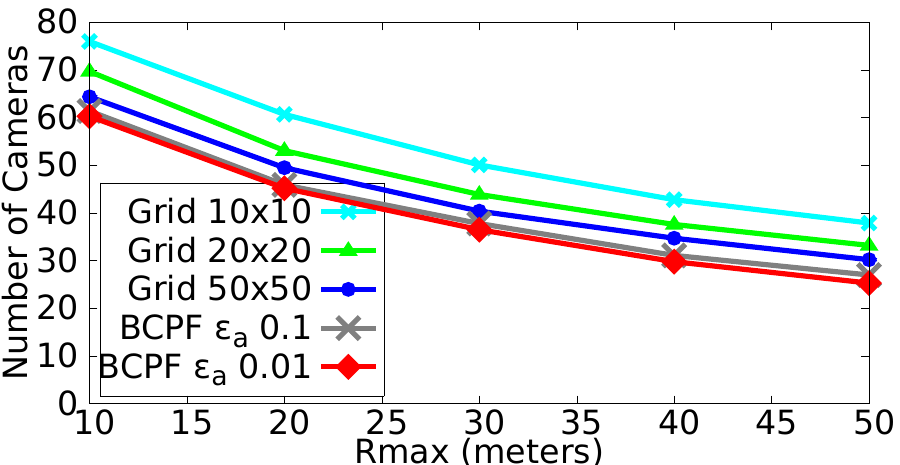}}
\subfigure{\includegraphics[width=1\linewidth ]{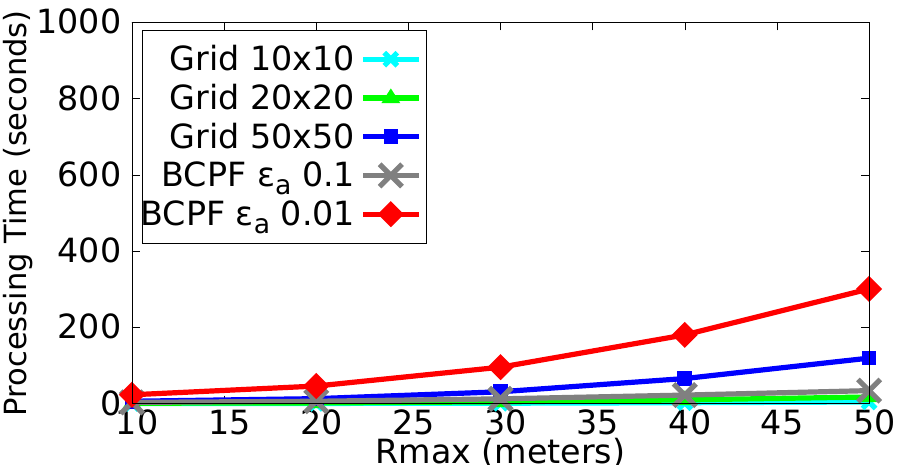}}
\caption{Effect of varying $R_{max}$ for a fixed number of targets.}
\label{fig:rmax_perf_more}
\end{minipage}
\hfill
\begin{minipage}[b]{0.32\linewidth}
\centering
\subfigure{\includegraphics[width=1\textwidth ]{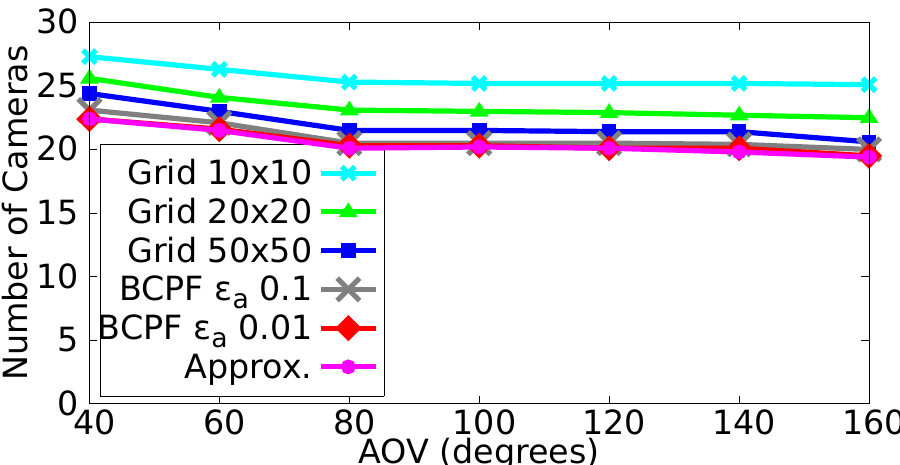}}
\subfigure{\includegraphics[width=1\textwidth ]{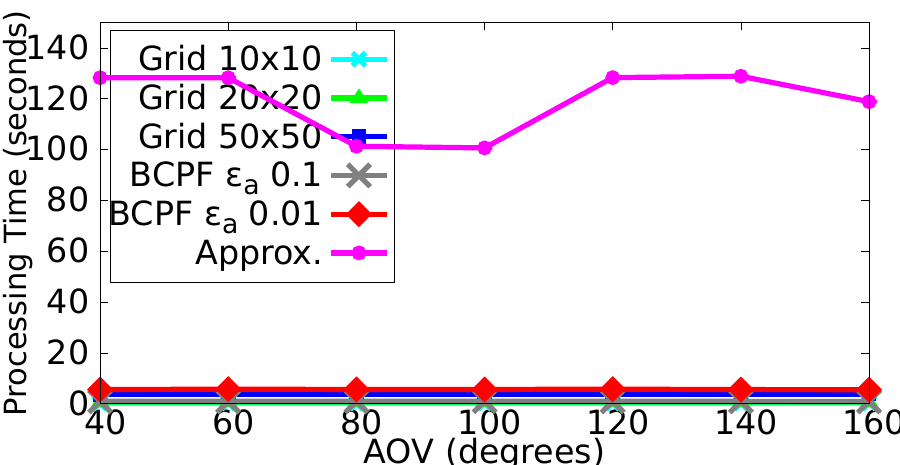}}
\caption{Effect of varying angle of view (AOV) for a fixed number of targets.}
\label{fig:fov_perf}
\end{minipage}
\end{figure*}

Figure~\ref{fig:target_perf} shows the effect of increasing the number of targets. The approximation algorithm produces the best performance in terms of the number of cameras required while taking orders of magnitude more time than sampling approaches due to its higher complexity. The approximation algorithm exceeds a minute per calculation for less than 25 targets while BCPF sampling computes a coverage for 140 targets in around a minute with $\epsilon_a=0.1 \: rad$.

Figure~\ref{fig:target_perf_more} contrasts the performance of sampling approaches in large scale scenarios.
Grid sampling provides a comparable number of cameras for small numbers of targets where it is unlikely to have compact configurations of CPF intersections that grid sampling might miss. However, it is clear that BCPF sampling is superior in terms of the number of cameras. Moreover, for a moderate $\epsilon_a$ of 0.1~$rad$, BCPF sampling outperforms grid sampling requiring 12\% less cameras and running up to 2x faster on 140 targets.

Figure \ref{fig:rmax_perf_more} shows the effect of increasing $R_{max}$ which increases the area covered by each CPF. Having larger CPFs increases the number of regions to be considered in the approximation algorithm and the number of CPFs that a sample point can belong to in the sampling approaches. On the other hand, changing the value of the AOV, shown in Figure~\ref{fig:fov_perf}, does not impact the processing time by much as it does not increase the area of the CPF considerably. However, it increases the number of targets included at each step of sweeping which reduces the number of cameras needed for coverage.

Based on our simulation results, BCPF sampling is the method of choice for a wide range of scenarios as it combines time and resource efficiency especially for large numbers of targets. 

\section{Extensions and Open Problems}
\label{sec:extend}

\textit{\ProblemNameAcro} presents a new powerful model that can be harnessed to capture many scenarios with little to no modifications. In this section, we discuss a few of those scenarios with two goals in mind: 1) facilitate the porting of this model to be used in other domains, and 2) illustrate future research directions where this model can be extended or applied to improve target models in smart surveillance and visual sensing systems.

\subsection{Coverage in 3D}

{\color{black}
The coverage model adopted in this paper, per Section~\ref{sec:model}, is aimed to enhance traditional coverage models used primarily in the surveillance literature. We developed the \ModelNameAcro{} model as a convenient alternative to modeling targets by mere points, which can be incorporated with little overhead. On the other hand, the \ModelNameAcro{} model does not fully capture the 3D nature of targets and their coverage constraints. We propose a simple adaptation that accounts for the changes in visibility and allows the placement of drones at different altitudes.
}

In scenarios where cameras are mounted on flying robots, there are many more configurations available for covering any given set of targets. In particular, for ground targets it is possible to mitigate occlusion effects by flying at a higher altitude. To make our {\ModelNameAcro} model even more realistic, information about the 3D shape of the target should be taken into account to calculate the best camera locations. The main parameter we consider here is target height. Although we could extend our coverage constraints to 3D and attempt a similar approach to what we have done in 2D, we propose a sampling strategy that builds on the heuristics we developed and tested in 2D.

We propose a heuristic solution to the problem in this new setting, by defining a new 2D problem at each discretization of flying altitudes $h$. The 2D problems we define are almost identical to the situation we had before, except in the following: 1) We get occlusion constraints by intersecting the 3D shadow prism with the horizontal plane at height $h$ as shown in Figure \ref{fig:3d}; 2) We adjust $R^h_{max}$ and $R^h_{min}$ at each height, the new ranges are calculated as a function of $h$ such that the calculated ranges do not violate the original range restrictions; 3) Angular sweeping will not be performed in 2D as mentioned in Section~\ref{sec:sweeping}, but rather in 3D to cover both possible camera pans as well as camera tilts.

\begin{figure}[!t]
\centering
\subfigure[Occlusion volume for a 3D model (green).\label{fig:3d_3d}]{\includegraphics[width=0.35\linewidth ]{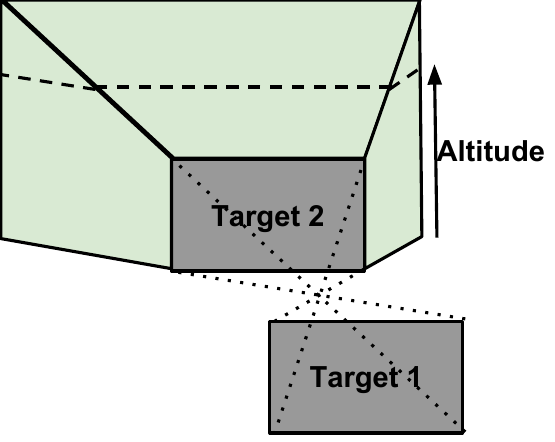}}
\hspace{0.2in}
\subfigure[Occlusion area in 2D at altitude $h$ (green).\label{fig:2d_3d}]{\includegraphics[width=0.35\linewidth ]{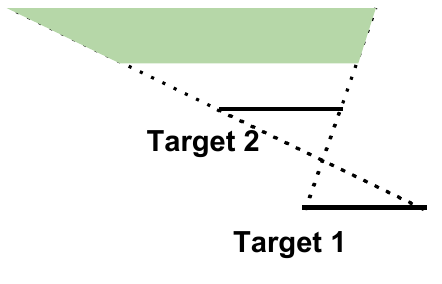}}
  \caption{3D shadow prism and occlusion region at a given height mapped to a 2D setting.}
  \label{fig:3d}
\end{figure}

The candidates sampled at each height $h$ are all added to the set of candidates and we map to \verb|SET-COVER| as before. For every given grouping of targets, we may get redundant candidates at different heights that all cover the same group. Before solving the \verb|SET-COVER| instance, we may filter these redundant candidates by only keeping ones with preferred altitudes or any other criteria. After we select the set of candidates to use for coverage, we may follow with a pan-tilt optimization to get the best quality of coverage. A 3D version of this algorithm will require target heights, in addition to target locations, poses, and widths.

For applications dealing with relatively large objects (e.g. buildings and statues) this may require breaking big targets into multiple smaller targets. Several endeavors in surveillance and computer visions are directed towards extracting 3D information from 2D pictures or videos \cite{ramakrishna2012reconstructing,aubry2014seeing} or depth sensors \cite{henry2012rgb} to capture the relevant features of objects. Such algorithms and sensors can provide a 3D version of \textit{\ProblemNameAcro} with the necessary input for its operation.

\subsection{\ProblemNameAcro{} using a fixed number of cameras} Availability of a limited number of mobile cameras is to be expected in most scenarios. We propose a 2-phase algorithm that first solves {\ProblemNameAcro} and suggests a minimal number of camera configurations. The second phase is to solve an instance of the Multiple Traveling Salesmen Problem (\verb|mTSP|) where mobile cameras try to visit all the identified locations and take the required shots through the most efficient (e.g., shortest) trajectories \cite{bektas2006multiple}. While this does not guarantee optimal coverage, it provides a nice extension for the proposed model to handle a common scenario. Solving the actual problem requires an extended formulation that captures both \verb|SET-COVER| and \verb|TSP|, which we consider as a future research direction in improving the current model. A closely related problem was studied in \cite{wang2007metric}.

\subsection{Target Quality of View Requirements}

Several metrics of quality of coverage can be considered when evaluating a coverage algorithm. On top of that, different quality constraints might be required for each type of targets. Quality requirements include: angle of view, target-camera distance, tracking performance, and accounting for device constraints (i.e., pan, tilt, and zoom limitations) \cite{krahnstoever2008collaborative}. Those requirements vary from one application to the other, and even from one target to the other. 

\textit{\ProblemNameAcro} captures most of those quality requirements by having a flexible representation of the camera placement field, which encodes range, viewing-angle and occlusion constraints. In addition, the target can basically have two effective widths: a total width used to evaluate occlusion constraints for other targets and a \textit{feature width} used to ensure the target itself is adequately covered. Tracking performance, however, remains a challenge which we further discuss in the next section.

\subsection{Open Problems}
\label{sec:open}

In this section, we discuss some open problems that are not treated in our study of Argus.

{\color{black}\textbf{Continuously updating camera configurations:} Recall that Argus relies on the top tier to perform target tracking as part of its persistent coverage functionality; this is the responsibility of the \emph{Target Localizer} module. For drone navigation, we also assume that the top tier provides suitable positioning information to help plan efficient and safe paths to position each drone at its assigned location; which is implemented in the \emph{UAV Localizer} module. In addition, to cope with positioning errors, each drone can factor in the readings from its on-board sensors to better position itself with respect to the targets it is assigned to cover; this is handled by the \emph{Adaptive Tracker} module.}

{\color{black} When it comes to dynamically updating camera configurations, target handover and mobility management are essential issues \cite{foresti2005active}}. Argus relies on mobile cameras which means that as targets move, cameras need to handover targets among themselves and possibly move to maintain all targets in view, or the target configuration will necessitate adding new cameras due to new occlusion or pose conditions. Beyond updating camera locations and assignments, the system is also required to optimize the energy used to move the cameras and the quality of coverage of targets during handover.

While handling such scenarios can simply be achieved by running the algorithm every $\delta t$ seconds, and updating the locations of cameras to adopt the new solution, this would result in considerable overhead and delays as no relationship is assumed between the camera configuration $U_t$ and the configuration $U_{t+1}$, for a time instant $t$. Several algorithms were studied to address such scenarios. {\color{black} We highlight some of those approaches} and leave it to future work to select the approach most suitable for Argus and the specific application at hand. 

One approach relies on computer vision algorithms to detect targets as they move between camera views without requiring cameras to move \cite{chen2008adaptive}. Cameras can then move slightly in order to better cover the new targets that entered its field of view to satisfy {\ProblemNameAcro} constraints. Another approach relies on cameras bidding on which of them will be required to cover a target that just moved, based on a coverage utility function \cite{esterle2014socio}. An {\ProblemNameAcro} implementation of this approach can introduce as a utility the amount of energy consumed by the camera-mounted robot to move in order to better cover the target.


\textbf{Positioning tolerance and camera calibration:}
As the provable approximation algorithm used candidate locations defined by the intersection points of coverage constraints, such locations may not be the best to use in practice. In particular, we have shown that such locations will lie on the boundary of a region where a single camera may cover the same set of targets. Attempting to place cameras at the boundary means that any inaccuracy in target localization and camera control or the expected movement of targets can all lead to missing some of the targets. Hence, the solution returned by {\ProblemNameAcro} only suggests a feasible partitioning of targets, where each partition can be covered by a single camera. This suggests that once {\ProblemNameAcro} optimization identifies a set of locations for camera placement, each camera should follow by optimizing its own configuration for the purpose of covering the set of targets it has been assigned. This becomes more relevant when planning the trajectory each camera should execute as it transitions from one configuration to another, possibly in response to target mobility.

\section{Related Work}
\label{sec:related_work}


\textbf{Area and target coverage:} The goal of area coverage algorithms is to detect any activity of interest within a certain area in a sensor network deployment, or to guarantee quality communication over a wireless network among clients in that area. Several approaches to area coverage have been studied including static randomly deployed sensors \cite{carmi2007covering} and strategically placed mobile sensors \cite{1200627} using either isotropic \cite{hexsel2011coverage} or anisotropic sensors \cite{6471967}. The related problem of barrier coverage was studied in \cite{kumar2005barrier}, where the objective is to detect any targets crossing the barrier into an area of interest. To cover a set of targets within an area, target coverage algorithms were studied in randomly deployed sensors \cite{ai2006coverage,johnson2011pan, AVSS2012}, or strategically placed directional sensors or antennas \cite{berman2007packing,ours_mass,ours_mobiwac}. Target coverage using static randomly deployed Pan-Tilt-Zoom (PTZ) cameras, that possibly zoom in to obtain better views, was shown to be NP-hard and a 2-approximation algorithm was presented \cite{johnson2011pan}. For antenna placement to serve a set of static targets with bounds on the bandwidth demand per antenna, a 3-approximation algorithm was presented in \cite{berman2007packing}. In order to satisfy connectivity requirements between antennas, \cite{han2008deploying} gave a 9-approximation algorithm.
\textit{We propose a more realistic model for target coverage by visual sensors that greatly enhances the point model typically used by earlier algorithms \cite{amac2011coverage}. Our approach requires fewer sensors compared to area coverage techniques as it only attempts to cover the present targets rather than the whole area of interest. We establish lower bounds on minimizing the number of sensors required by the new model and develop a matching approximation algorithm in Section~3.}

\textbf{Full-view coverage:} Full-view coverage is a variant of area coverage with the extra objective of ensuring that any target is covered from all angles \cite{wang2011full}. \cite{wu2012achieving} studies the necessary conditions for full-view coverage in static camera deployments and \cite{hu2014critical} studies full-view coverage using heterogeneous mobile cameras. Full view barrier coverage was then introduced \cite{wang2011barrier} and further extended to accommodate stochastic deployments in \cite{7218437}. Taking self-occlusions into account, ensuring all sides of a convex target are always visible was studied in \cite{Tokekar2014}. \textit{Our proposed approach is different in two aspects: 1) It overcomes occlusion scenarios and takes target size into account in addition to target pose. 2) It is concerned with target coverage rather than area coverage which is the main concern of full-view coverage.}

{\color{black}
\textbf{Persistent Coverage:} In a seminal paper \cite{Eyes_in_the_sky} a decentralized control strategy was introduced for the deployment of heterogeneous cameras for coverage tasks, which significantly improved upon earlier methods~\cite{1284411}. In follow-up works, the control strategy was further enhanced to robustly learn and adapt to changes in the environment \cite{robust_cov,pers_dist_TR16}, while emphasizing decentralized control over a communication networks. More recently, the control strategy was endowed with optimal path planning while avoiding obstacles in the environment~\cite{multi_tracking_IROS14,pers_ICRA17,Best2017} and target unpredictability~\cite{dynamic_cov_IROS16}. Several challenging aspects of persistent coverage have also been studied: energy-awareness~\cite{docking}, connectivity~\cite{relay_networks}, adaptive streaming \cite{video_streaming} and dynamic priorities~\cite{8001408}. For a more thorough survey, we refer the reader to~\cite{machines2010013, coop_cov_review}. \textit{The proposed system leverages established results in persistent coverage through a two tier architecture. We make critical use of the information provided by the top tier providing persistent coverage over the environment to plan the placement of mobile cameras in the lower tier to obtain high quality views of the targets of interest.}}

\textbf{Video capture using drones:} There has been a growing interest in using drones and drone swarms for surveillance and video capture \cite{burkle2009collaborating} {\color{black}, e.g., for sport streaming~\cite{sports}}. In such applications, several challenges including target mobility and low quality footage (e.g., due to distance) were studied in \cite{dornetchallenges}. For mobile target tracking, using either a single drone \cite{naseer2013followme} or multiple drones \cite{mueller2016persistent} can be used for persistent tracking. Such applications focus on target coverage without restricting the angles from which targets are viewed. Autonomous cinematography is another application for drones, beyond coverage and tracking, the aesthetic quality plays a key role in viewpoint planning \cite{joubert2016towards}.  {\color{black} Building upon earlier work in virtual cinematography~\cite{director_ICMR11}, this exciting line of work has recently been witnessing very interesting developments~\cite{cin_TOG15,ETH,cin_director}.} In earlier work, we developed several target coverage algorithms for targets represented as points and deployed them on our testbed \cite{ours_mass,ours_mobiwac,una,khan2016simulating}. \textit{In this paper, our work leverages recent advances in drone technologies to develop an autonomous system that utilizes our enhanced target model and demonstrate the feasibility of running the proposed coverage algorithms on a real system. We envision extensions of the proposed model to accommodate specific aesthetic or gesture capture requirements to allow more control over the quality of coverage as required for persistent tracking and cinematography.}

\textbf{Art Gallery Problem:} A classical problem in discrete and computational geometry asks for the minimum number of guards required to see every point in an art gallery represented by a polygon with or without holes \cite{urrutia2000art}. Several variants were introduced constraining guard placement (e.g. point, vertex, or edge) and coverage (i.e convex, star shaped, or spiral shaped) \cite{culberson1988covering,lee1986computational}. In particular, the art gallery illumination problem considered guards having a limited angle of view \cite{bagga1996complexity,bose1997floodlight}. Visibility algorithms have found many applications in wireless communications, sensor networks, surveillance, and robotics. However, several variants were shown to be NP-hard \cite{o1983some} {\color{black} , and more recently even $\exists \mathbb{R}$-complete~\cite{AGP_ETR}}. In addition, inapproximability results for art gallery coverage with and without holes were shown in \cite{eidenbenz2001inapproximability} and also for the illumination of art galleries without holes \cite{abdelkaderinapproximability}. On the approximation side, the works in \cite{deshpande2007pseudopolynomial,gonzalez2001randomized} presented algorithms for the coverage of art galleries with and without holes, respectively. \textit{We settle the hardness and approximability of art gallery illumination for polygons with holes and use this to prove the hardness of {\ProblemNameAcro}. We also present a best-possible approximation algorithm for {\ProblemNameAcro} based on a spatial subdivision derived from the coverage constraints. The novelty of our algorithm lies in the incorporation of a limited angle of view camera model with our newly proposed target model.} Earlier approximation algorithms relied on triangulations \cite{deshpande2007pseudopolynomial} or sampling \cite{gonzalez2001randomized} while assuming omnidirectional cameras.



\section{Conclusion}
\label{sec:conclusion}

We presented Argus, an autonomous system that utilizes drones to provide better coverage of targets taking into account their size, pose, and potential occlusions. We started by introducing \emph{\ModelNameAcro}, a novel geometric model that captures wide oriented targets and the conditions necessary for their coverage. Then, we formulated the \emph{\ProblemName} \emph{(\ProblemNameAcro)} that aims at minimizing the number of cameras required to cover a set of targets represented by this new model.
We devised a best-possible $O(\log{n})$-approximation algorithm and a sampling heuristic that runs up to 100x faster while performing favorably compared to the provably-bounded approximation algorithm. Finally, we developed a fully autonomous prototype that uses quadcopters to monitor synthetic targets in order to measure the overhead of the proposed algorithms in realistic scenarios and show the improved quality of coverage provided by the new model.

\appendix
\section{A Technical Lemma for Handling AOV Constraints}
\label{sec:appendix}

Recall that for any pair of points $(q_a, q_b)$, an \textit{AOV circle pair} was defined in \ref{sec:exact} as the two congruent circles sharing $\overline{q_a q_b}$ as a chord at an inscribed angle equal to the AOV $\theta$. Now, for two targets $T_a$ and $T_b$ represented by the line segments $(P^s_a, P^e_a)$ and $(P^s_b, P^e_b)$, respectively, we define the set of diagonals as $$T_a \otimes T_b = \{ (P_a^s, P_b^s), (P_a^s, P_b^e), (P_a^e, P_b^s), (P_a^e, P_b^e) \}.$$ The next lemma shows that the four AOV circle pairs corresponding to the diagonals $T_a \otimes T_b$ suffice to capture the AOV constraints for a pair of targets $T_a$ and $T_b$, as needed for the comprehensive discretization of the search space for camera placement per \ref{sec:exact}.

\begin{lemma}\label{aov_lemma}
For a fixed AOV angle $\theta$ and any pair of points $(q_a, q_b) \in T_a \times T_b$, the corresponding pair of AOV circles is completely contained in the union of AOV circles corresponding to the diagonals $T_a \otimes T_b$.
\end{lemma}

\begin{proof}
For any pair of points $(q_1, q_2)$ let $\{ \mathcal{C}^r_{q_1, q_2}, \mathcal{C}^l_{q_1, q_2} \}$ be the corresponding pair of AOV circles for an AOV angle $\theta$ and let $\{ c^r_{q_1, q_2}, c^l_{q_1, q_2} \}$ be their centers with $c^r_{q_1, q_2}$ to the right of $\overrightarrow{q_1 q_2}$ and $c^l_{q_1, q_2}$ to its left. Without loss of generality, we will show that $$\mathcal{C}^r_{q_a, q_b}
\subseteq
\mathcal{C}^r_{q_a, P_b^s} \cup \mathcal{C}^r_{q_a, P_b^e} \subseteq (\mathcal{C}^r_{P_a^s, P_b^s} \cup \mathcal{C}^r_{P_a^e, P_b^s}) \cup (\mathcal{C}^r_{P_a^s, P_b^e} \cup \mathcal{C}^r_{P_a^e, P_b^e}).$$

The key claim is that for any pair of points $(q_a, q_b)$, the circle $\mathcal{C}^r_{q_a, q_b}$ is contained in the union of the two circles obtained by replacing one of the points with the two end points on its segment, i.e., $\mathcal{C}^r_{q_a, P_b^s} \cup \mathcal{C}^r_{q_a, P_b^e}$. Observe that the point we do not replace, i.e., $q_a$, will be shared by all three circles.

Recall that the circle $\mathcal{C}^r_{q_a, q_b}$ has radius $\frac{|q_a q_b|}{2 \sin{\theta}}$, which is linear in $|q_a q_b|$. It follows that fixing $q_a$ and moving $q_b$ to any point $q_{b'}$ on the closed line segment $\overline{P_b^s P_b^e}$, the radius of the intermediate circle $C^r_{q_a, q_{b'}}$ varies linearly as $\frac{|q_a q_{b'}|}{2 \sin{\theta}}$. Consequently, the centers of all intermediate circles $c^r_{q_a, q_{b'}}$ lie on the line segment between $c^r_{q_a, P_b^s}$ and $c^r_{q_a, P_b^e}$; denote this line segment by $l_{s,e}$. It follows that all intermediate circles pass through not only $q_a$, but both points of intersection between $\mathcal{C}^r_{q_a, P_b^s}$ and $\mathcal{C}^r_{q_a, P_b^e}$; we denote the other point by $q_{\hat{a}}$ as it is the mirror image of $q_a$ about $l_{s,e}$. See Figure~\ref{fig:aov_proof} for an example.

We argue that $\mathcal{C}^r_{q_a, q_{b'}} \subseteq \mathcal{C}^r_{q_a, P_b^s} \cup \mathcal{C}^r_{q_a, P_b^e}$ for all points $q_{b'} \in \overline{P_b^s P_b^e}$, such as $q_b$. Using the common chord $\overline{q_a q_{\hat{a}}}$, we cut the circle $\mathcal{C}^r_{q_a, q_{b'}}$ into two arcs $\arc{q_a q_{\hat{a}}}^s$ and $\arc{q_a q_{\hat{a}}}^e$ where the first intersects $l_{s,e}$ closer to $c^r_{q_a, P_b^s}$ and the latter intersects $l_{s,e}$ closer to $c^r_{q_a, P_b^e}$. Since any two circles intersect in at most two points and $\{q_a, q_{\hat{a}}\}$ are the two points of intersection between $\mathcal{C}^r_{q_a, q_{b'}}$ and $\mathcal{C}^r_{q_a, P_b^s}$, which may also coincide if $q_{b'} = P^s_b$, it follows that $\arc{q_a q_{\hat{a}}}^s$ cannot exit $\mathcal{C}^r_{q_a, P^s_b}$. Hence, $\arc{q_a q_{\hat{a}}}^s \subseteq \mathcal{C}^r_{q_a, P_b^s}$ and similarly $\arc{q_a q_{\hat{a}}}^e \subseteq \mathcal{C}^r_{q_a, P_b^e}$, as shown in Figure~\ref{fig:aov_proof_arc}.

This shows that indeed $\mathcal{C}^r_{q_a, q_{b'}} \subseteq \mathcal{C}^r_{q_a, P_b^s} \cup \mathcal{C}^r_{q_a, P_b^e}$, as required. By symmetry, we also obtain $\mathcal{C}^r_{q_a, q_{b'}} \subseteq \mathcal{C}^r_{P_a^s, q_{b'}} \cup \mathcal{C}^r_{P_a^e, q_{b'}}$. Now, setting $q_{b'} = P^s_b$ yields
$\mathcal{C}^r_{q_a, P_b^s} \subseteq
\mathcal{C}^r_{P_a^s, P_b^s} \cup \mathcal{C}^r_{P_a^e, P_b^s}$ and $q_{b'} = P^e_b$ yields
$\mathcal{C}^r_{q_a, P_b^e} \subseteq
\mathcal{C}^r_{P_a^s, P_b^e} \cup \mathcal{C}^r_{P_a^e, P_b^e}$, which completes the proof.
\end{proof}

\begin{figure}[t]
\centering
\subfigure[Illustrating the two properties of AOV circles that we use in the proof. \label{fig:aov_proof}]{\includegraphics[width=.25\linewidth ]{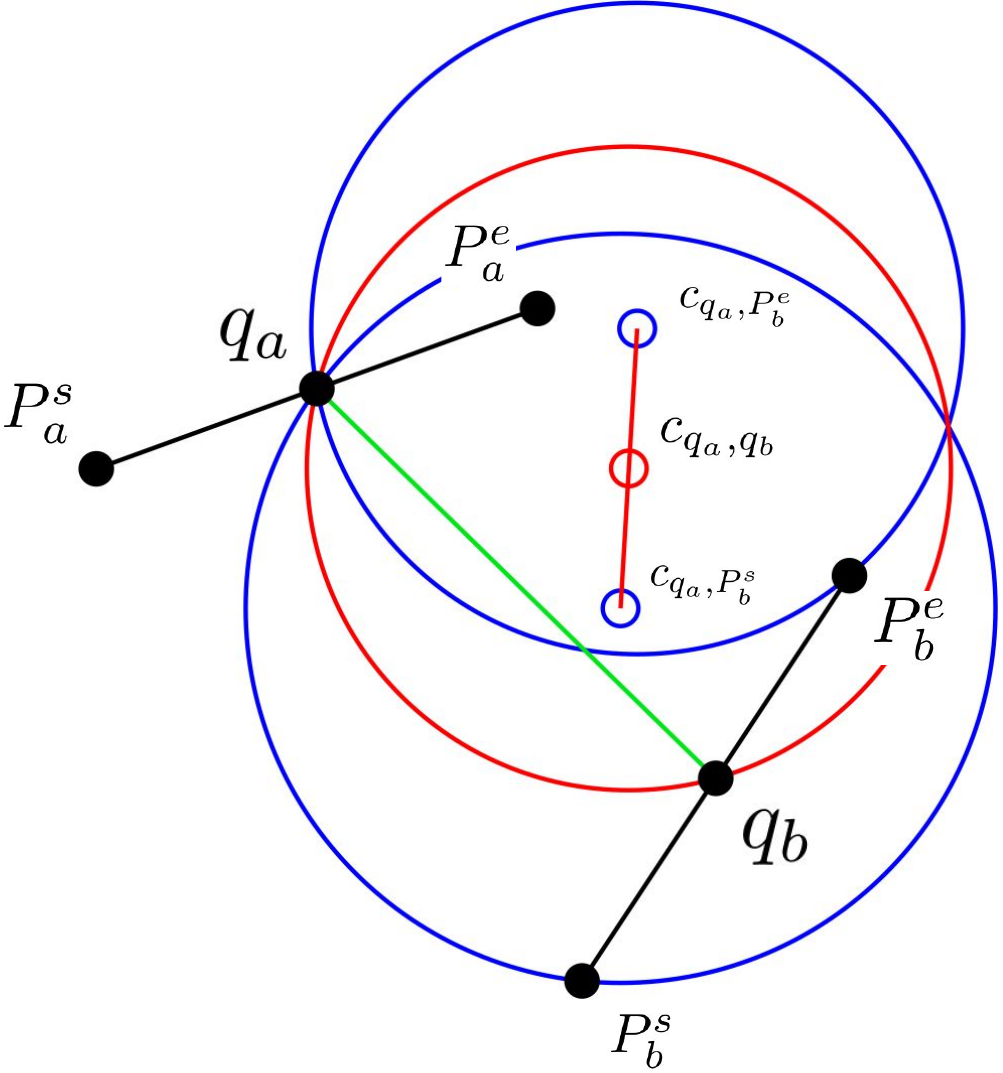}}
\hspace{0.2in}
\subfigure[Arcs of a common chord are fully contained in the other circle.\label{fig:aov_proof_arc}]{\includegraphics[width=.35\linewidth ]{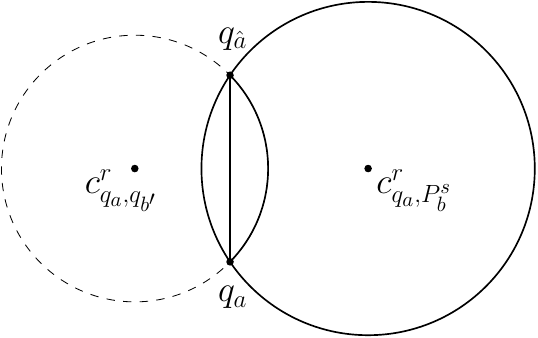}}
\caption{The key elements in the proof of Lemma~\ref{aov_lemma}.}
\label{fig:AOV_proof}
\end{figure}

\balance

\bibliographystyle{ACM-Reference-Format}
\bibliography{main}


\begin{thebibliography}{00}


\ifx \showCODEN    \undefined \def \showCODEN     #1{\unskip}     \fi
\ifx \showDOI      \undefined \def \showDOI       #1{{\tt DOI:}\penalty0{#1}\ }
  \fi
\ifx \showISBNx    \undefined \def \showISBNx     #1{\unskip}     \fi
\ifx \showISBNxiii \undefined \def \showISBNxiii  #1{\unskip}     \fi
\ifx \showISSN     \undefined \def \showISSN      #1{\unskip}     \fi
\ifx \showLCCN     \undefined \def \showLCCN      #1{\unskip}     \fi
\ifx \shownote     \undefined \def \shownote      #1{#1}          \fi
\ifx \showarticletitle \undefined \def \showarticletitle #1{#1}   \fi
\ifx \showURL      \undefined \def \showURL       #1{#1}          \fi
\providecommand\bibfield[2]{#2}
\providecommand\bibinfo[2]{#2}
\providecommand\natexlab[1]{#1}
\providecommand\showeprint[2][]{arXiv:#2}

\bibitem[\protect\citeauthoryear{Abdelkader, Mokhtar, and El-Alfy}{Abdelkader
  et~al\mbox{.}}{2012}]%
        {AVSS2012}
\bibfield{author}{\bibinfo{person}{Ahmed Abdelkader}, \bibinfo{person}{Moamen
  Mokhtar}, {and} \bibinfo{person}{Hazem El-Alfy}.}
  \bibinfo{year}{2012}\natexlab{}.
\newblock \showarticletitle{Angular Heuristics for Coverage Maximization in
  Multi-camera Surveillance}. In \bibinfo{booktitle}{{\em AVSS'12}}.
\newblock


\bibitem[\protect\citeauthoryear{Abdelkader, Saeed, Harras, and
  Mohamed}{Abdelkader et~al\mbox{.}}{2015}]%
        {abdelkaderinapproximability}
\bibfield{author}{\bibinfo{person}{Ahmed Abdelkader}, \bibinfo{person}{Ahmed
  Saeed}, \bibinfo{person}{Khaled Harras}, {and} \bibinfo{person}{Amr
  Mohamed}.} \bibinfo{year}{2015}\natexlab{}.
\newblock \showarticletitle{The Inapproximability of Illuminating Polygons by
  $\alpha$-Floodlights}. In \bibinfo{booktitle}{{\em CCCG'15}}.
\newblock


\bibitem[\protect\citeauthoryear{Abrahamsen, Adamaszek, and Miltzow}{Abrahamsen
  et~al\mbox{.}}{2017}]%
        {AGP_ETR}
\bibfield{author}{\bibinfo{person}{Mikkel Abrahamsen}, \bibinfo{person}{Anna
  Adamaszek}, {and} \bibinfo{person}{Tillmann Miltzow}.}
  \bibinfo{year}{2017}\natexlab{}.
\newblock \showarticletitle{The Art Gallery Problem is
  {\(\exists\)}{\(\mathbb{R}\)}-complete}.
\newblock \bibinfo{journal}{{\em CoRR\/}}  \bibinfo{volume}{abs/1704.06969}
  (\bibinfo{year}{2017}).
\newblock
\showeprint[arxiv]{1704.06969}
\showURL{%
\url{http://arxiv.org/abs/1704.06969}}


\bibitem[\protect\citeauthoryear{Adib, Kabelac, Katabi, and Miller}{Adib
  et~al\mbox{.}}{2014}]%
        {adib20143d}
\bibfield{author}{\bibinfo{person}{Fadel Adib}, \bibinfo{person}{Zach Kabelac},
  \bibinfo{person}{Dina Katabi}, {and} \bibinfo{person}{Robert~C. Miller}.}
  \bibinfo{year}{2014}\natexlab{}.
\newblock \showarticletitle{3D Tracking via Body Radio Reflections}. In
  \bibinfo{booktitle}{{\em NSDI'14}}.
\newblock


\bibitem[\protect\citeauthoryear{Ai and Abouzeid}{Ai and Abouzeid}{2006}]%
        {ai2006coverage}
\bibfield{author}{\bibinfo{person}{Jing Ai} {and} \bibinfo{person}{Alhussein~A
  Abouzeid}.} \bibinfo{year}{2006}\natexlab{}.
\newblock \showarticletitle{Coverage by directional sensors in randomly
  deployed wireless sensor networks}.
\newblock \bibinfo{journal}{{\em Journal of Combinatorial Optimization\/}}
  \bibinfo{volume}{11}, \bibinfo{number}{1} (\bibinfo{year}{2006}),
  \bibinfo{pages}{21--41}.
\newblock


\bibitem[\protect\citeauthoryear{Amac~Guvensan and Gokhan~Yavuz}{Amac~Guvensan
  and Gokhan~Yavuz}{2011}]%
        {amac2011coverage}
\bibfield{author}{\bibinfo{person}{M Amac~Guvensan} {and} \bibinfo{person}{A
  Gokhan~Yavuz}.} \bibinfo{year}{2011}\natexlab{}.
\newblock \showarticletitle{On coverage issues in directional sensor networks:
  A survey}.
\newblock \bibinfo{journal}{{\em Ad Hoc Networks\/}} \bibinfo{volume}{9},
  \bibinfo{number}{7} (\bibinfo{year}{2011}), \bibinfo{pages}{1238--1255}.
\newblock


\bibitem[\protect\citeauthoryear{Amazon.com}{Amazon.com}{2014}]%
        {Alexa}
\bibfield{author}{\bibinfo{person}{Inc. Amazon.com}.}
  \bibinfo{year}{2014}\natexlab{}.
\newblock \bibinfo{title}{Amazon Alexa}.
\newblock   (\bibinfo{year}{2014}).
\newblock
\showURL{%
\url{http://alexa.amazon.com/}}


\bibitem[\protect\citeauthoryear{Aubry, Maturana, Efros, Russell, and
  Sivic}{Aubry et~al\mbox{.}}{2014}]%
        {aubry2014seeing}
\bibfield{author}{\bibinfo{person}{Mathieu Aubry}, \bibinfo{person}{Daniel
  Maturana}, \bibinfo{person}{Alexei~A Efros}, \bibinfo{person}{Bryan~C
  Russell}, {and} \bibinfo{person}{Josef Sivic}.}
  \bibinfo{year}{2014}\natexlab{}.
\newblock \showarticletitle{Seeing 3D chairs: exemplar part-based 2D-3D
  alignment using a large dataset of CAD models}. In \bibinfo{booktitle}{{\em
  CVPR'14}}.
\newblock


\bibitem[\protect\citeauthoryear{Bagga, Gewali, and Glasser}{Bagga
  et~al\mbox{.}}{1996}]%
        {bagga1996complexity}
\bibfield{author}{\bibinfo{person}{Jay Bagga}, \bibinfo{person}{Laxmi Gewali},
  {and} \bibinfo{person}{David Glasser}.} \bibinfo{year}{1996}\natexlab{}.
\newblock \showarticletitle{The Complexity of Illuminating Polygons by
  alpha-flood-lights}. In \bibinfo{booktitle}{{\em CCCG'96}}.
\newblock


\bibitem[\protect\citeauthoryear{Bay, Tuytelaars, and Van~Gool}{Bay
  et~al\mbox{.}}{2006}]%
        {bay2006surf}
\bibfield{author}{\bibinfo{person}{Herbert Bay}, \bibinfo{person}{Tinne
  Tuytelaars}, {and} \bibinfo{person}{Luc Van~Gool}.}
  \bibinfo{year}{2006}\natexlab{}.
\newblock \showarticletitle{Surf: Speeded up robust features}.
\newblock In \bibinfo{booktitle}{{\em Computer vision--ECCV 2006}}.
\newblock


\bibitem[\protect\citeauthoryear{Bektas}{Bektas}{2006}]%
        {bektas2006multiple}
\bibfield{author}{\bibinfo{person}{Tolga Bektas}.}
  \bibinfo{year}{2006}\natexlab{}.
\newblock \showarticletitle{The multiple traveling salesman problem: an
  overview of formulations and solution procedures}.
\newblock \bibinfo{journal}{{\em Omega\/}} (\bibinfo{year}{2006}).
\newblock


\bibitem[\protect\citeauthoryear{Berman, Jeong, Kasiviswanathan, and
  Urgaonkar}{Berman et~al\mbox{.}}{2007}]%
        {berman2007packing}
\bibfield{author}{\bibinfo{person}{Piotr Berman}, \bibinfo{person}{Jieun
  Jeong}, \bibinfo{person}{Shiva~Prasad Kasiviswanathan}, {and}
  \bibinfo{person}{Bhuvan Urgaonkar}.} \bibinfo{year}{2007}\natexlab{}.
\newblock \showarticletitle{Packing to angles and sectors}. In
  \bibinfo{booktitle}{{\em SPAA'07}}.
\newblock


\bibitem[\protect\citeauthoryear{Best, Faigl, and Fitch}{Best
  et~al\mbox{.}}{2017}]%
        {Best2017}
\bibfield{author}{\bibinfo{person}{Graeme Best}, \bibinfo{person}{Jan Faigl},
  {and} \bibinfo{person}{Robert Fitch}.} \bibinfo{year}{2017}\natexlab{}.
\newblock \showarticletitle{Online planning for multi-robot active perception
  with self-organising maps}.
\newblock \bibinfo{journal}{{\em Autonomous Robots\/}} (\bibinfo{date}{13 Dec}
  \bibinfo{year}{2017}).
\newblock
\showISSN{1573-7527}
\showDOI{%
\url{http://dx.doi.org/10.1007/s10514-017-9691-4}}


\bibitem[\protect\citeauthoryear{Bircher, Alexis, Schwesinger, Omari, Burri,
  and Siegwart}{Bircher et~al\mbox{.}}{2017}]%
        {Bircher2017}
\bibfield{author}{\bibinfo{person}{Andreas Bircher}, \bibinfo{person}{Kostas
  Alexis}, \bibinfo{person}{Ulrich Schwesinger}, \bibinfo{person}{Sammy Omari},
  \bibinfo{person}{Michael Burri}, {and} \bibinfo{person}{Roland Siegwart}.}
  \bibinfo{year}{2017}\natexlab{}.
\newblock \showarticletitle{An incremental sampling-based approach to
  inspection planning: the rapidly exploring random tree of trees}.
\newblock \bibinfo{journal}{{\em Robotica\/}} \bibinfo{volume}{35},
  \bibinfo{number}{6} (\bibinfo{year}{2017}), \bibinfo{pages}{1327–1340}.
\newblock
\showDOI{%
\url{http://dx.doi.org/10.1017/S0263574716000084}}


\bibitem[\protect\citeauthoryear{Blanz, Grother, Phillips, and Vetter}{Blanz
  et~al\mbox{.}}{2005}]%
        {blanz2005face}
\bibfield{author}{\bibinfo{person}{Volker Blanz}, \bibinfo{person}{Patrick
  Grother}, \bibinfo{person}{P~Jonathon Phillips}, {and}
  \bibinfo{person}{Thomas Vetter}.} \bibinfo{year}{2005}\natexlab{}.
\newblock \showarticletitle{Face recognition based on frontal views generated
  from non-frontal images}. In \bibinfo{booktitle}{{\em CVPR'05}}.
\newblock


\bibitem[\protect\citeauthoryear{Bose, Guibas, Lubiw, Overmars, Souvaine, and
  Urrutia}{Bose et~al\mbox{.}}{1997}]%
        {bose1997floodlight}
\bibfield{author}{\bibinfo{person}{Prosenjit Bose}, \bibinfo{person}{Leonidas
  Guibas}, \bibinfo{person}{Anna Lubiw}, \bibinfo{person}{Mark Overmars},
  \bibinfo{person}{Diane Souvaine}, {and} \bibinfo{person}{Jorge Urrutia}.}
  \bibinfo{year}{1997}\natexlab{}.
\newblock \showarticletitle{The floodlight problem}.
\newblock \bibinfo{journal}{{\em International Journal of Computational
  Geometry \& Applications\/}} \bibinfo{volume}{7}, \bibinfo{number}{01n02}
  (\bibinfo{year}{1997}).
\newblock


\bibitem[\protect\citeauthoryear{B{\"u}rkle}{B{\"u}rkle}{2009}]%
        {burkle2009collaborating}
\bibfield{author}{\bibinfo{person}{Axel B{\"u}rkle}.}
  \bibinfo{year}{2009}\natexlab{}.
\newblock \showarticletitle{Collaborating miniature drones for surveillance and
  reconnaissance}. In \bibinfo{booktitle}{{\em Proc. of SPIE
  Unmanned/Unattended Sensors and Sensor Networks VI}},
  Vol.~\bibinfo{volume}{7480}.
\newblock


\bibitem[\protect\citeauthoryear{Carmi, Katz, and Lev-Tov}{Carmi
  et~al\mbox{.}}{2007}]%
        {carmi2007covering}
\bibfield{author}{\bibinfo{person}{Paz Carmi}, \bibinfo{person}{Matthew~J
  Katz}, {and} \bibinfo{person}{Nissan Lev-Tov}.}
  \bibinfo{year}{2007}\natexlab{}.
\newblock \showarticletitle{Covering points by unit disks of fixed location}.
\newblock In \bibinfo{booktitle}{{\em Algorithms and Computation}}.
  \bibinfo{pages}{644--655}.
\newblock


\bibitem[\protect\citeauthoryear{Chang, Chua, Leman, Wang, and Zhang}{Chang
  et~al\mbox{.}}{2013}]%
        {chang2013automatic}
\bibfield{author}{\bibinfo{person}{Ronald Chang}, \bibinfo{person}{Teck~Wee
  Chua}, \bibinfo{person}{Karianto Leman}, \bibinfo{person}{Hee~Lin Wang},
  {and} \bibinfo{person}{Jie Zhang}.} \bibinfo{year}{2013}\natexlab{}.
\newblock \showarticletitle{Automatic cooperative camera system for real-time
  bag detection in visual surveillance}. In \bibinfo{booktitle}{{\em
  ICDSC'13}}.
\newblock


\bibitem[\protect\citeauthoryear{Chen and Odobez}{Chen and Odobez}{2012}]%
        {chen2012we}
\bibfield{author}{\bibinfo{person}{Cheng Chen} {and} \bibinfo{person}{Jean-Marc
  Odobez}.} \bibinfo{year}{2012}\natexlab{}.
\newblock \showarticletitle{We are not contortionists: coupled adaptive
  learning for head and body orientation estimation in surveillance video}. In
  \bibinfo{booktitle}{{\em CVPR'12}}.
\newblock


\bibitem[\protect\citeauthoryear{Chen, Huang, and Fu}{Chen
  et~al\mbox{.}}{2008a}]%
        {chen2008hybrid}
\bibfield{author}{\bibinfo{person}{Hsiuao-Ying Chen},
  \bibinfo{person}{Chung-Lin Huang}, {and} \bibinfo{person}{Chih-Ming Fu}.}
  \bibinfo{year}{2008}\natexlab{a}.
\newblock \showarticletitle{Hybrid-boost learning for multi-pose face detection
  and facial expression recognition}.
\newblock \bibinfo{journal}{{\em Pattern Recognition\/}} \bibinfo{volume}{41},
  \bibinfo{number}{3} (\bibinfo{year}{2008}), \bibinfo{pages}{1173--1185}.
\newblock


\bibitem[\protect\citeauthoryear{Chen, Lai, Hung, and Chen}{Chen
  et~al\mbox{.}}{2008b}]%
        {chen2008adaptive}
\bibfield{author}{\bibinfo{person}{Kuan-Wen Chen}, \bibinfo{person}{Chih-Chuan
  Lai}, \bibinfo{person}{Yi-Ping Hung}, {and} \bibinfo{person}{Chu-Song Chen}.}
  \bibinfo{year}{2008}\natexlab{b}.
\newblock \showarticletitle{An adaptive learning method for target tracking
  across multiple cameras}. In \bibinfo{booktitle}{{\em CVPR'08}}.
\newblock


\bibitem[\protect\citeauthoryear{Chvatal}{Chvatal}{1979}]%
        {chvatal1979greedy}
\bibfield{author}{\bibinfo{person}{Vasek Chvatal}.}
  \bibinfo{year}{1979}\natexlab{}.
\newblock \showarticletitle{A greedy heuristic for the set-covering problem}.
\newblock \bibinfo{journal}{{\em Mathematics of operations research\/}}
  \bibinfo{volume}{4}, \bibinfo{number}{3} (\bibinfo{year}{1979}),
  \bibinfo{pages}{233--235}.
\newblock


\bibitem[\protect\citeauthoryear{Cortes, Martinez, Karatas, and Bullo}{Cortes
  et~al\mbox{.}}{2004}]%
        {1284411}
\bibfield{author}{\bibinfo{person}{J. Cortes}, \bibinfo{person}{S. Martinez},
  \bibinfo{person}{T. Karatas}, {and} \bibinfo{person}{F. Bullo}.}
  \bibinfo{year}{2004}\natexlab{}.
\newblock \showarticletitle{Coverage control for mobile sensing networks}.
\newblock \bibinfo{journal}{{\em IEEE Transactions on Robotics and
  Automation\/}} \bibinfo{volume}{20}, \bibinfo{number}{2}
  (\bibinfo{date}{April} \bibinfo{year}{2004}), \bibinfo{pages}{243--255}.
\newblock
\showISSN{1042-296X}
\showDOI{%
\url{http://dx.doi.org/10.1109/TRA.2004.824698}}


\bibitem[\protect\citeauthoryear{Culberson and Reckhow}{Culberson and
  Reckhow}{1988}]%
        {culberson1988covering}
\bibfield{author}{\bibinfo{person}{Joseph~C Culberson} {and}
  \bibinfo{person}{Robert~A Reckhow}.} \bibinfo{year}{1988}\natexlab{}.
\newblock \showarticletitle{Covering polygons is hard}. In
  \bibinfo{booktitle}{{\em FOCS'88}}.
\newblock


\bibitem[\protect\citeauthoryear{da~Silva, Bernardo, de~Oliveira, and
  Rosa}{da~Silva et~al\mbox{.}}{2017}]%
        {8001408}
\bibfield{author}{\bibinfo{person}{L.~C.~Batista da Silva},
  \bibinfo{person}{R.~M. Bernardo}, \bibinfo{person}{H.~A. de Oliveira}, {and}
  \bibinfo{person}{P.~F.~F. Rosa}.} \bibinfo{year}{2017}\natexlab{}.
\newblock \showarticletitle{Unmanned aircraft system coordination for
  persistent surveillance with different priorities}. In
  \bibinfo{booktitle}{{\em 2017 IEEE 26th International Symposium on Industrial
  Electronics (ISIE)}}. \bibinfo{pages}{1153--1158}.
\newblock
\showDOI{%
\url{http://dx.doi.org/10.1109/ISIE.2017.8001408}}


\bibitem[\protect\citeauthoryear{Denman, Kleinschmidt, Ryan, Barnes, Sridharan,
  and Fookes}{Denman et~al\mbox{.}}{2015}]%
        {denman2015automatic}
\bibfield{author}{\bibinfo{person}{Simon Denman}, \bibinfo{person}{Tristan
  Kleinschmidt}, \bibinfo{person}{David Ryan}, \bibinfo{person}{Paul Barnes},
  \bibinfo{person}{Sridha Sridharan}, {and} \bibinfo{person}{Clinton Fookes}.}
  \bibinfo{year}{2015}\natexlab{}.
\newblock \showarticletitle{Automatic surveillance in transportation hubs: No
  longer just about catching the bad guy}.
\newblock \bibinfo{journal}{{\em Expert Systems with Applications\/}}
  \bibinfo{volume}{42}, \bibinfo{number}{24} (\bibinfo{year}{2015}).
\newblock


\bibitem[\protect\citeauthoryear{Derenick, Michael, and Kumar}{Derenick
  et~al\mbox{.}}{2011}]%
        {docking}
\bibfield{author}{\bibinfo{person}{J. Derenick}, \bibinfo{person}{N. Michael},
  {and} \bibinfo{person}{V. Kumar}.} \bibinfo{year}{2011}\natexlab{}.
\newblock \showarticletitle{Energy-aware coverage control with docking for
  robot teams}. In \bibinfo{booktitle}{{\em 2011 IEEE/RSJ International
  Conference on Intelligent Robots and Systems}}. \bibinfo{pages}{3667--3672}.
\newblock
\showISSN{2153-0858}
\showDOI{%
\url{http://dx.doi.org/10.1109/IROS.2011.6094977}}


\bibitem[\protect\citeauthoryear{Deshpande, Kim, Demaine, and Sarma}{Deshpande
  et~al\mbox{.}}{2007}]%
        {deshpande2007pseudopolynomial}
\bibfield{author}{\bibinfo{person}{Ajay Deshpande}, \bibinfo{person}{Taejung
  Kim}, \bibinfo{person}{Erik Demaine}, {and} \bibinfo{person}{Sanjay Sarma}.}
  \bibinfo{year}{2007}\natexlab{}.
\newblock \showarticletitle{A Pseudopolynomial Time
  {$O(\log{n})$}-Approximation Algorithm for Art Gallery Problems}.
\newblock \bibinfo{journal}{{\em Algorithms and Data Structures\/}}
  (\bibinfo{year}{2007}), \bibinfo{pages}{163--174}.
\newblock


\bibitem[\protect\citeauthoryear{Dhillon and Chakrabarty}{Dhillon and
  Chakrabarty}{2003}]%
        {1200627}
\bibfield{author}{\bibinfo{person}{Santpal~Singh Dhillon} {and}
  \bibinfo{person}{Krishnendu Chakrabarty}.} \bibinfo{year}{2003}\natexlab{}.
\newblock \showarticletitle{Sensor placement for effective coverage and
  surveillance in distributed sensor networks}. In \bibinfo{booktitle}{{\em
  WCNC'03}}.
\newblock


\bibitem[\protect\citeauthoryear{Eidenbenz, Stamm, and Widmayer}{Eidenbenz
  et~al\mbox{.}}{2001}]%
        {eidenbenz2001inapproximability}
\bibfield{author}{\bibinfo{person}{Stephan Eidenbenz},
  \bibinfo{person}{Christoph Stamm}, {and} \bibinfo{person}{Peter Widmayer}.}
  \bibinfo{year}{2001}\natexlab{}.
\newblock \showarticletitle{Inapproximability results for guarding polygons and
  terrains}.
\newblock \bibinfo{journal}{{\em Algorithmica\/}} \bibinfo{volume}{31},
  \bibinfo{number}{1} (\bibinfo{year}{2001}), \bibinfo{pages}{79--113}.
\newblock


\bibitem[\protect\citeauthoryear{Erdelj, Natalizio, Chowdhury, and
  Akyildiz}{Erdelj et~al\mbox{.}}{2017}]%
        {disaster}
\bibfield{author}{\bibinfo{person}{M. Erdelj}, \bibinfo{person}{E. Natalizio},
  \bibinfo{person}{K.~R. Chowdhury}, {and} \bibinfo{person}{I.~F. Akyildiz}.}
  \bibinfo{year}{2017}\natexlab{}.
\newblock \showarticletitle{Help from the Sky: Leveraging UAVs for Disaster
  Management}.
\newblock \bibinfo{journal}{{\em IEEE Pervasive Computing\/}}
  \bibinfo{volume}{16}, \bibinfo{number}{1} (\bibinfo{date}{Jan}
  \bibinfo{year}{2017}), \bibinfo{pages}{24--32}.
\newblock


\bibitem[\protect\citeauthoryear{Esterle, Lewis, Yao, and Rinner}{Esterle
  et~al\mbox{.}}{2014}]%
        {esterle2014socio}
\bibfield{author}{\bibinfo{person}{Lukas Esterle}, \bibinfo{person}{Peter~R
  Lewis}, \bibinfo{person}{Xin Yao}, {and} \bibinfo{person}{Bernhard Rinner}.}
  \bibinfo{year}{2014}\natexlab{}.
\newblock \showarticletitle{Socio-economic vision graph generation and handover
  in distributed smart camera networks}.
\newblock \bibinfo{journal}{{\em ACM Transactions on Sensor Networks (TOSN)\/}}
  \bibinfo{volume}{10}, \bibinfo{number}{2} (\bibinfo{year}{2014}),
  \bibinfo{pages}{20}.
\newblock


\bibitem[\protect\citeauthoryear{Finn and Wright}{Finn and Wright}{2012}]%
        {finn2012unmanned}
\bibfield{author}{\bibinfo{person}{Rachel~L Finn} {and} \bibinfo{person}{David
  Wright}.} \bibinfo{year}{2012}\natexlab{}.
\newblock \showarticletitle{Unmanned aircraft systems: Surveillance, ethics and
  privacy in civil applications}.
\newblock \bibinfo{journal}{{\em Computer Law \& Security Review\/}}
  \bibinfo{volume}{28}, \bibinfo{number}{2} (\bibinfo{year}{2012}),
  \bibinfo{pages}{184--194}.
\newblock


\bibitem[\protect\citeauthoryear{Flohr, Dumitru-Guzu, Kooij, and Gavrila}{Flohr
  et~al\mbox{.}}{2015}]%
        {flohr2015probabilistic}
\bibfield{author}{\bibinfo{person}{Fabian Flohr}, \bibinfo{person}{Madalin
  Dumitru-Guzu}, \bibinfo{person}{Julian~FP Kooij}, {and}
  \bibinfo{person}{Dariu~M Gavrila}.} \bibinfo{year}{2015}\natexlab{}.
\newblock \showarticletitle{A probabilistic framework for joint pedestrian head
  and body orientation estimation}.
\newblock \bibinfo{journal}{{\em IEEE Transactions on Intelligent
  Transportation Systems\/}} \bibinfo{volume}{16}, \bibinfo{number}{4}
  (\bibinfo{year}{2015}), \bibinfo{pages}{1872--1882}.
\newblock


\bibitem[\protect\citeauthoryear{Foresti, Micheloni, Snidaro, Remagnino, and
  Ellis}{Foresti et~al\mbox{.}}{2005}]%
        {foresti2005active}
\bibfield{author}{\bibinfo{person}{Gian~Luca Foresti},
  \bibinfo{person}{Christian Micheloni}, \bibinfo{person}{Lauro Snidaro},
  \bibinfo{person}{Paolo Remagnino}, {and} \bibinfo{person}{Tim Ellis}.}
  \bibinfo{year}{2005}\natexlab{}.
\newblock \showarticletitle{Active video-based surveillance system: the
  low-level image and video processing techniques needed for implementation}.
\newblock \bibinfo{journal}{{\em Signal Processing Magazine, IEEE\/}}
  \bibinfo{volume}{22}, \bibinfo{number}{2} (\bibinfo{year}{2005}),
  \bibinfo{pages}{25--37}.
\newblock


\bibitem[\protect\citeauthoryear{Fraga-Lamas, Fernández-Caramés,
  Suárez-Albela, Castedo, and González-López}{Fraga-Lamas
  et~al\mbox{.}}{2016}]%
        {DPS}
\bibfield{author}{\bibinfo{person}{Paula Fraga-Lamas},
  \bibinfo{person}{Tiago~M. Fernández-Caramés}, \bibinfo{person}{Manuel
  Suárez-Albela}, \bibinfo{person}{Luis Castedo}, {and}
  \bibinfo{person}{Miguel González-López}.} \bibinfo{year}{2016}\natexlab{}.
\newblock \showarticletitle{A Review on Internet of Things for Defense and
  Public Safety}.
\newblock \bibinfo{journal}{{\em Sensors\/}} \bibinfo{volume}{16},
  \bibinfo{number}{10} (\bibinfo{year}{2016}).
\newblock


\bibitem[\protect\citeauthoryear{G.~La~Vigne, S.~Lowry, Markman, and
  Dwyer}{G.~La~Vigne et~al\mbox{.}}{2011}]%
        {vigne2011evaluating}
\bibfield{author}{\bibinfo{person}{Nancy G.~La~Vigne},
  \bibinfo{person}{Samantha S.~Lowry}, \bibinfo{person}{Joshwa Markman}, {and}
  \bibinfo{person}{Allison Dwyer}.} \bibinfo{year}{2011}\natexlab{}.
\newblock \bibinfo{booktitle}{{\em Evaluating the Use of Public Surveillance
  Cameras for Crime Control and Prevention}}.
\newblock \bibinfo{publisher}{Urban Institute, Justice Policy Center}.
\newblock


\bibitem[\protect\citeauthoryear{{Galvane}, {Lino}, {Christie}, {Fleureau},
  {Servant}, {Tariolle}, and {Guillotel}}{{Galvane} et~al\mbox{.}}{2017}]%
        {cin_director}
\bibfield{author}{\bibinfo{person}{Q. {Galvane}}, \bibinfo{person}{C. {Lino}},
  \bibinfo{person}{M. {Christie}}, \bibinfo{person}{J. {Fleureau}},
  \bibinfo{person}{F. {Servant}}, \bibinfo{person}{F.-L. {Tariolle}}, {and}
  \bibinfo{person}{P. {Guillotel}}.} \bibinfo{year}{2017}\natexlab{}.
\newblock \showarticletitle{{Directing Cinematographic Drones}}.
\newblock \bibinfo{journal}{{\em ArXiv e-prints\/}} (\bibinfo{date}{Dec.}
  \bibinfo{year}{2017}).
\newblock
\showeprint[arxiv]{1712.04216}


\bibitem[\protect\citeauthoryear{Gonz{\'a}lez-Ba{\~n}os}{Gonz{\'a}lez-Ba{\~n}os}{2001}]%
        {gonzalez2001randomized}
\bibfield{author}{\bibinfo{person}{H{\'e}ctor Gonz{\'a}lez-Ba{\~n}os}.}
  \bibinfo{year}{2001}\natexlab{}.
\newblock \showarticletitle{A randomized art-gallery algorithm for sensor
  placement}. In \bibinfo{booktitle}{{\em SoCG'01}}.
\newblock


\bibitem[\protect\citeauthoryear{Google}{Google}{2016}]%
        {GHome}
\bibfield{author}{\bibinfo{person}{LLC Google}.}
  \bibinfo{year}{2016}\natexlab{}.
\newblock \bibinfo{title}{Google Home}.
\newblock   (\bibinfo{year}{2016}).
\newblock
\showURL{%
\url{http://home.google.com/}}


\bibitem[\protect\citeauthoryear{Han, Cao, Lloyd, and Shen}{Han
  et~al\mbox{.}}{2008}]%
        {han2008deploying}
\bibfield{author}{\bibinfo{person}{Xiaofeng Han}, \bibinfo{person}{Xiang Cao},
  \bibinfo{person}{Errol~L Lloyd}, {and} \bibinfo{person}{Chien-Chung Shen}.}
  \bibinfo{year}{2008}\natexlab{}.
\newblock \showarticletitle{Deploying directional sensor networks with
  guaranteed connectivity and coverage}. In \bibinfo{booktitle}{{\em
  SECON'08}}. IEEE, \bibinfo{pages}{153--160}.
\newblock


\bibitem[\protect\citeauthoryear{Hassanalian and Abdelkefi}{Hassanalian and
  Abdelkefi}{2017}]%
        {HASSANALIAN201799}
\bibfield{author}{\bibinfo{person}{M. Hassanalian} {and} \bibinfo{person}{A.
  Abdelkefi}.} \bibinfo{year}{2017}\natexlab{}.
\newblock \showarticletitle{Classifications, applications, and design
  challenges of drones: A review}.
\newblock \bibinfo{journal}{{\em Progress in Aerospace Sciences\/}}
  \bibinfo{volume}{91}, \bibinfo{number}{Supplement C} (\bibinfo{year}{2017}),
  \bibinfo{pages}{99 -- 131}.
\newblock
\showISSN{0376-0421}
\showDOI{%
\url{http://dx.doi.org/https://doi.org/10.1016/j.paerosci.2017.04.003}}


\bibitem[\protect\citeauthoryear{Hausman, M{\"u}ller, Hariharan, Ayanian, and
  Sukhatme}{Hausman et~al\mbox{.}}{2015}]%
        {hausman2015cooperative}
\bibfield{author}{\bibinfo{person}{Karol Hausman}, \bibinfo{person}{J{\"o}rg
  M{\"u}ller}, \bibinfo{person}{Abishek Hariharan}, \bibinfo{person}{Nora
  Ayanian}, {and} \bibinfo{person}{Gaurav~S Sukhatme}.}
  \bibinfo{year}{2015}\natexlab{}.
\newblock \showarticletitle{Cooperative multi-robot control for target tracking
  with onboard sensing}.
\newblock \bibinfo{journal}{{\em The International Journal of Robotics
  Research\/}} \bibinfo{volume}{34}, \bibinfo{number}{13}
  (\bibinfo{year}{2015}), \bibinfo{pages}{1660--1677}.
\newblock


\bibitem[\protect\citeauthoryear{Helbling and Wood}{Helbling and Wood}{2017}]%
        {Helbling2017}
\bibfield{author}{\bibinfo{person}{Elizabeth~Farrell Helbling} {and}
  \bibinfo{person}{Robert Wood}.} \bibinfo{year}{2017}\natexlab{}.
\newblock \showarticletitle{A Review of Propulsion, Power, and Control
  Architectures for Insect-Scale Flapping-Wing Vehicles}.
\newblock \bibinfo{journal}{{\em Applied Mechanics Reviews\/}}
  (\bibinfo{date}{21 Dec} \bibinfo{year}{2017}).
\newblock
\showISSN{0003-6900}
\showDOI{%
\url{http://dx.doi.org/10.1115/1.4038795}}


\bibitem[\protect\citeauthoryear{Henry~et al.}{Henry~et al.}{2012}]%
        {henry2012rgb}
\bibfield{author}{\bibinfo{person}{Peter Henry~et al.}}
  \bibinfo{year}{2012}\natexlab{}.
\newblock \showarticletitle{RGB-D mapping: Using Kinect-style depth cameras for
  dense 3D modeling of indoor environments}.
\newblock \bibinfo{journal}{{\em The International Journal of Robotics
  Research\/}} (\bibinfo{year}{2012}).
\newblock


\bibitem[\protect\citeauthoryear{Hexsel, Chakraborty, and Sycara}{Hexsel
  et~al\mbox{.}}{2011}]%
        {hexsel2011coverage}
\bibfield{author}{\bibinfo{person}{Bruno Hexsel}, \bibinfo{person}{Nilanjan
  Chakraborty}, {and} \bibinfo{person}{K Sycara}.}
  \bibinfo{year}{2011}\natexlab{}.
\newblock \showarticletitle{Coverage control for mobile anisotropic sensor
  networks}. In \bibinfo{booktitle}{{\em ICRA'11}}.
\newblock


\bibitem[\protect\citeauthoryear{Hsu and Chen}{Hsu and Chen}{2015}]%
        {dornetchallenges}
\bibfield{author}{\bibinfo{person}{Hwai-Jung Hsu} {and}
  \bibinfo{person}{Kuan-Ta Chen}.} \bibinfo{year}{2015}\natexlab{}.
\newblock \showarticletitle{Face Recognition on Drones: Issues and
  Limitations}. In \bibinfo{booktitle}{{\em DroNet '15}}.
\newblock


\bibitem[\protect\citeauthoryear{Hu, Tan, Wang, and Maybank}{Hu
  et~al\mbox{.}}{2004}]%
        {hu2004survey}
\bibfield{author}{\bibinfo{person}{Weiming Hu}, \bibinfo{person}{Tieniu Tan},
  \bibinfo{person}{Liang Wang}, {and} \bibinfo{person}{Steve Maybank}.}
  \bibinfo{year}{2004}\natexlab{}.
\newblock \showarticletitle{A survey on visual surveillance of object motion
  and behaviors}.
\newblock \bibinfo{journal}{{\em IEEE Transactions on Systems, Man, and
  Cybernetics, Part C: Applications and Reviews\/}} \bibinfo{volume}{34},
  \bibinfo{number}{3} (\bibinfo{year}{2004}), \bibinfo{pages}{334--352}.
\newblock


\bibitem[\protect\citeauthoryear{Hu, Wang, and Gan}{Hu et~al\mbox{.}}{2014}]%
        {hu2014critical}
\bibfield{author}{\bibinfo{person}{Yitao Hu}, \bibinfo{person}{Xinbing Wang},
  {and} \bibinfo{person}{Xiaoying Gan}.} \bibinfo{year}{2014}\natexlab{}.
\newblock \showarticletitle{Critical Sensing Range for Mobile Heterogeneous
  Camera Sensor Networks}. In \bibinfo{booktitle}{{\em INFOCOM'14}}.
\newblock


\bibitem[\protect\citeauthoryear{Hönig and Ayanian}{Hönig and
  Ayanian}{2016}]%
        {dynamic_cov_IROS16}
\bibfield{author}{\bibinfo{person}{W. Hönig} {and} \bibinfo{person}{N.
  Ayanian}.} \bibinfo{year}{2016}\natexlab{}.
\newblock \showarticletitle{Dynamic multi-target coverage with robotic
  cameras}. In \bibinfo{booktitle}{{\em 2016 IEEE/RSJ International Conference
  on Intelligent Robots and Systems (IROS)}}. \bibinfo{pages}{1871--1878}.
\newblock
\showDOI{%
\url{http://dx.doi.org/10.1109/IROS.2016.7759297}}


\bibitem[\protect\citeauthoryear{Joubert, Goldman, Berthouzoz, Roberts, Landay,
  Hanrahan, et~al\mbox{.}}{Joubert et~al\mbox{.}}{2016}]%
        {joubert2016towards}
\bibfield{author}{\bibinfo{person}{Niels Joubert}, \bibinfo{person}{Dan~B
  Goldman}, \bibinfo{person}{Floraine Berthouzoz}, \bibinfo{person}{Mike
  Roberts}, \bibinfo{person}{James~A Landay}, \bibinfo{person}{Pat Hanrahan},
  {and} \bibinfo{person}{others}.} \bibinfo{year}{2016}\natexlab{}.
\newblock \showarticletitle{Towards a Drone Cinematographer: Guiding Quadrotor
  Cameras using Visual Composition Principles}.
\newblock \bibinfo{journal}{{\em arXiv preprint arXiv:1610.01691\/}}
  (\bibinfo{year}{2016}).
\newblock


\bibitem[\protect\citeauthoryear{Joubert, Roberts, Truong, Berthouzoz, and
  Hanrahan}{Joubert et~al\mbox{.}}{2015}]%
        {cin_TOG15}
\bibfield{author}{\bibinfo{person}{Niels Joubert}, \bibinfo{person}{Mike
  Roberts}, \bibinfo{person}{Anh Truong}, \bibinfo{person}{Floraine
  Berthouzoz}, {and} \bibinfo{person}{Pat Hanrahan}.}
  \bibinfo{year}{2015}\natexlab{}.
\newblock \showarticletitle{An Interactive Tool for Designing Quadrotor Camera
  Shots}.
\newblock \bibinfo{journal}{{\em ACM Trans. Graph.\/}} \bibinfo{volume}{34},
  \bibinfo{number}{6}, Article \bibinfo{articleno}{238} (\bibinfo{date}{Oct.}
  \bibinfo{year}{2015}), \bibinfo{numpages}{11}~pages.
\newblock
\showISSN{0730-0301}
\showDOI{%
\url{http://dx.doi.org/10.1145/2816795.2818106}}


\bibitem[\protect\citeauthoryear{Khan, Rinner, and Cavallaro}{Khan
  et~al\mbox{.}}{2018}]%
        {coop_cov_review}
\bibfield{author}{\bibinfo{person}{A. Khan}, \bibinfo{person}{B. Rinner}, {and}
  \bibinfo{person}{A. Cavallaro}.} \bibinfo{year}{2018}\natexlab{}.
\newblock \showarticletitle{Cooperative Robots to Observe Moving Targets:
  Review}.
\newblock \bibinfo{journal}{{\em IEEE Transactions on Cybernetics\/}}
  \bibinfo{volume}{48}, \bibinfo{number}{1} (\bibinfo{date}{Jan}
  \bibinfo{year}{2018}), \bibinfo{pages}{187--198}.
\newblock
\showISSN{2168-2267}
\showDOI{%
\url{http://dx.doi.org/10.1109/TCYB.2016.2628161}}


\bibitem[\protect\citeauthoryear{Khan, Alam, Mohamed, and Harras}{Khan
  et~al\mbox{.}}{2016}]%
        {khan2016simulating}
\bibfield{author}{\bibinfo{person}{Mouhyemen Khan}, \bibinfo{person}{Sidra
  Alam}, \bibinfo{person}{Amr Mohamed}, {and} \bibinfo{person}{Khaled~A
  Harras}.} \bibinfo{year}{2016}\natexlab{}.
\newblock \showarticletitle{Simulating Drone-be-Gone: Agile Low-Cost
  Cyber-Physical UAV Testbed (Demonstration)}. In \bibinfo{booktitle}{{\em
  AAMAS'16}}.
\newblock


\bibitem[\protect\citeauthoryear{Kim and Shim}{Kim and Shim}{2013}]%
        {kim2013unmanned}
\bibfield{author}{\bibinfo{person}{Jeong~Woon Kim} {and}
  \bibinfo{person}{David~Hyunchul Shim}.} \bibinfo{year}{2013}\natexlab{}.
\newblock \showarticletitle{A vision-based target tracking control system of a
  quadrotor by using a tablet computer}. In \bibinfo{booktitle}{{\em
  International Conference on Unmanned Aircraft Systems (ICUAS'13)}}.
  \bibinfo{pages}{1165--1172}.
\newblock


\bibitem[\protect\citeauthoryear{Koppal}{Koppal}{2016}]%
        {photography_small_SPM16}
\bibfield{author}{\bibinfo{person}{S.~J. Koppal}.}
  \bibinfo{year}{2016}\natexlab{}.
\newblock \showarticletitle{A Survey of Computational Photography in the Small:
  Creating intelligent cameras for the next wave of miniature devices}.
\newblock \bibinfo{journal}{{\em IEEE Signal Processing Magazine\/}}
  \bibinfo{volume}{33}, \bibinfo{number}{5} (\bibinfo{date}{Sept}
  \bibinfo{year}{2016}), \bibinfo{pages}{16--22}.
\newblock
\showISSN{1053-5888}
\showDOI{%
\url{http://dx.doi.org/10.1109/MSP.2016.2581418}}


\bibitem[\protect\citeauthoryear{Krahnstoever, Yu, Lim, Patwardhan, and
  Tu}{Krahnstoever et~al\mbox{.}}{2008}]%
        {krahnstoever2008collaborative}
\bibfield{author}{\bibinfo{person}{Nils Krahnstoever}, \bibinfo{person}{Ting
  Yu}, \bibinfo{person}{Ser-Nam Lim}, \bibinfo{person}{Kedar Patwardhan}, {and}
  \bibinfo{person}{Peter Tu}.} \bibinfo{year}{2008}\natexlab{}.
\newblock \showarticletitle{Collaborative real-time control of active cameras
  in large scale surveillance systems}. In \bibinfo{booktitle}{{\em Workshop on
  Multi-camera and Multi-modal Sensor Fusion Algorithms and Applications-M2SFA2
  2008}}.
\newblock


\bibitem[\protect\citeauthoryear{Krajn{\'\i}k, Von{\'a}sek, Fi{\v{s}}er, and
  Faigl}{Krajn{\'\i}k et~al\mbox{.}}{2011}]%
        {krajnik2011ar}
\bibfield{author}{\bibinfo{person}{Tom{\'a}{\v{s}} Krajn{\'\i}k},
  \bibinfo{person}{Vojt{\v{e}}ch Von{\'a}sek}, \bibinfo{person}{Daniel
  Fi{\v{s}}er}, {and} \bibinfo{person}{Jan Faigl}.}
  \bibinfo{year}{2011}\natexlab{}.
\newblock \showarticletitle{{AR-drone as a platform for robotic research and
  education}}.
\newblock In \bibinfo{booktitle}{{\em International ConferenceResearch and
  Education in Robotics}}.
\newblock


\bibitem[\protect\citeauthoryear{Kulkarni, Ganesan, Shenoy, and Lu}{Kulkarni
  et~al\mbox{.}}{2005}]%
        {kulkarni2005senseye}
\bibfield{author}{\bibinfo{person}{Purushottam Kulkarni},
  \bibinfo{person}{Deepak Ganesan}, \bibinfo{person}{Prashant Shenoy}, {and}
  \bibinfo{person}{Qifeng Lu}.} \bibinfo{year}{2005}\natexlab{}.
\newblock \showarticletitle{SensEye: a multi-tier camera sensor network}. In
  \bibinfo{booktitle}{{\em MM'05}}.
\newblock


\bibitem[\protect\citeauthoryear{Kumar, Lai, and Arora}{Kumar
  et~al\mbox{.}}{2005}]%
        {kumar2005barrier}
\bibfield{author}{\bibinfo{person}{Santosh Kumar}, \bibinfo{person}{Ten~H Lai},
  {and} \bibinfo{person}{Anish Arora}.} \bibinfo{year}{2005}\natexlab{}.
\newblock \showarticletitle{Barrier coverage with wireless sensors}. In
  \bibinfo{booktitle}{{\em MobiCom'05}}.
\newblock


\bibitem[\protect\citeauthoryear{Lee and Lin}{Lee and Lin}{1986}]%
        {lee1986computational}
\bibfield{author}{\bibinfo{person}{Der-Tsai Lee} {and} \bibinfo{person}{Arthurk
  Lin}.} \bibinfo{year}{1986}\natexlab{}.
\newblock \showarticletitle{Computational complexity of art gallery problems}.
\newblock \bibinfo{journal}{{\em IEEE Transactions on Information Theory\/}}
  \bibinfo{volume}{32}, \bibinfo{number}{2} (\bibinfo{year}{1986}),
  \bibinfo{pages}{276--282}.
\newblock


\bibitem[\protect\citeauthoryear{Lin, Pankanti, Natesan~Ramamurthy, and
  Aravkin}{Lin et~al\mbox{.}}{2015}]%
        {Lin_2015_CVPR}
\bibfield{author}{\bibinfo{person}{Chung-Ching Lin},
  \bibinfo{person}{Sharathchandra~U. Pankanti}, \bibinfo{person}{Karthikeyan
  Natesan~Ramamurthy}, {and} \bibinfo{person}{Aleksandr~Y. Aravkin}.}
  \bibinfo{year}{2015}\natexlab{}.
\newblock \showarticletitle{Adaptive As-Natural-As-Possible Image Stitching}.
  In \bibinfo{booktitle}{{\em CVPR'15}}.
\newblock


\bibitem[\protect\citeauthoryear{Lino and Christie}{Lino and Christie}{2015}]%
        {Lino_toric}
\bibfield{author}{\bibinfo{person}{Christophe Lino} {and} \bibinfo{person}{Marc
  Christie}.} \bibinfo{year}{2015}\natexlab{}.
\newblock \showarticletitle{Intuitive and Efficient Camera Control with the
  Toric Space}.
\newblock \bibinfo{journal}{{\em ACM Trans. Graph.\/}} \bibinfo{volume}{34},
  \bibinfo{number}{4}, Article \bibinfo{articleno}{82} (\bibinfo{date}{July}
  \bibinfo{year}{2015}), \bibinfo{numpages}{12}~pages.
\newblock
\showISSN{0730-0301}
\showDOI{%
\url{http://dx.doi.org/10.1145/2766965}}


\bibitem[\protect\citeauthoryear{Lino, Christie, Ranon, and Bares}{Lino
  et~al\mbox{.}}{2011}]%
        {director_ICMR11}
\bibfield{author}{\bibinfo{person}{Christophe Lino}, \bibinfo{person}{Marc
  Christie}, \bibinfo{person}{Roberto Ranon}, {and} \bibinfo{person}{William
  Bares}.} \bibinfo{year}{2011}\natexlab{}.
\newblock \showarticletitle{The Director's Lens: An Intelligent Assistant for
  Virtual Cinematography}. In \bibinfo{booktitle}{{\em Proceedings of the 19th
  ACM International Conference on Multimedia}} {\em (\bibinfo{series}{MM
  '11})}. \bibinfo{publisher}{ACM}, \bibinfo{address}{New York, NY, USA},
  \bibinfo{pages}{323--332}.
\newblock
\showISBNx{978-1-4503-0616-4}
\showDOI{%
\url{http://dx.doi.org/10.1145/2072298.2072341}}


\bibitem[\protect\citeauthoryear{Mueller, Sharma, Smith, and Ghanem}{Mueller
  et~al\mbox{.}}{2016}]%
        {mueller2016persistent}
\bibfield{author}{\bibinfo{person}{Matthias Mueller}, \bibinfo{person}{Gopal
  Sharma}, \bibinfo{person}{Neil Smith}, {and} \bibinfo{person}{Bernard
  Ghanem}.} \bibinfo{year}{2016}\natexlab{}.
\newblock \showarticletitle{Persistent Aerial Tracking system for UAVs}. In
  \bibinfo{booktitle}{{\em IROS'16}}.
\newblock


\bibitem[\protect\citeauthoryear{Mulgaonkar and Kumar}{Mulgaonkar and
  Kumar}{2014}]%
        {mulgaonkar2014towards}
\bibfield{author}{\bibinfo{person}{Yash Mulgaonkar} {and}
  \bibinfo{person}{Vijay Kumar}.} \bibinfo{year}{2014}\natexlab{}.
\newblock \showarticletitle{Towards Open-Source, Printable Pico-Quadrotors}. In
  \bibinfo{booktitle}{{\em Proc. Robot Makers Workshop, Robotics: Science
  Systems Conf}}.
\newblock


\bibitem[\protect\citeauthoryear{N\"{a}geli, Meier, Domahidi, Alonso-Mora, and
  Hilliges}{N\"{a}geli et~al\mbox{.}}{2017}]%
        {ETH}
\bibfield{author}{\bibinfo{person}{Tobias N\"{a}geli}, \bibinfo{person}{Lukas
  Meier}, \bibinfo{person}{Alexander Domahidi}, \bibinfo{person}{Javier
  Alonso-Mora}, {and} \bibinfo{person}{Otmar Hilliges}.}
  \bibinfo{year}{2017}\natexlab{}.
\newblock \showarticletitle{Real-time Planning for Automated Multi-view Drone
  Cinematography}.
\newblock \bibinfo{journal}{{\em ACM Trans. Graph.\/}} \bibinfo{volume}{36},
  \bibinfo{number}{4}, Article \bibinfo{articleno}{132} (\bibinfo{date}{July}
  \bibinfo{year}{2017}), \bibinfo{numpages}{10}~pages.
\newblock
\showISSN{0730-0301}
\showDOI{%
\url{http://dx.doi.org/10.1145/3072959.3073712}}


\bibitem[\protect\citeauthoryear{Naseer, Sturm, and Cremers}{Naseer
  et~al\mbox{.}}{2013}]%
        {naseer2013followme}
\bibfield{author}{\bibinfo{person}{Tayyab Naseer}, \bibinfo{person}{J{\"u}rgen
  Sturm}, {and} \bibinfo{person}{Daniel Cremers}.}
  \bibinfo{year}{2013}\natexlab{}.
\newblock \showarticletitle{Followme: Person following and gesture recognition
  with a quadrocopter}. In \bibinfo{booktitle}{{\em IROS'13}}.
\newblock


\bibitem[\protect\citeauthoryear{Natarajan, Atrey, and Kankanhalli}{Natarajan
  et~al\mbox{.}}{2015}]%
        {natarajan2015multi}
\bibfield{author}{\bibinfo{person}{Prabhu Natarajan},
  \bibinfo{person}{Pradeep~K. Atrey}, {and} \bibinfo{person}{Mohan
  Kankanhalli}.} \bibinfo{year}{2015}\natexlab{}.
\newblock \showarticletitle{Multi-Camera Coordination and Control in
  Surveillance Systems: A Survey}.
\newblock \bibinfo{journal}{{\em ACM Transactions on Multimedia Computing,
  Communications, and Applications\/}} \bibinfo{volume}{11},
  \bibinfo{number}{4} (\bibinfo{year}{2015}), \bibinfo{pages}{57:1--57:30}.
\newblock


\bibitem[\protect\citeauthoryear{Neishaboori, Saeed, Harras, and
  Mohamed}{Neishaboori et~al\mbox{.}}{2014a}]%
        {ours_mass}
\bibfield{author}{\bibinfo{person}{Azin Neishaboori}, \bibinfo{person}{Ahmed
  Saeed}, \bibinfo{person}{Khaled~A. Harras}, {and} \bibinfo{person}{Amar
  Mohamed}.} \bibinfo{year}{2014}\natexlab{a}.
\newblock \showarticletitle{Low Complexity Target Coverage Heuristics Using
  Mobile Cameras}. In \bibinfo{booktitle}{{\em MASS'14}}.
\newblock


\bibitem[\protect\citeauthoryear{Neishaboori, Saeed, Harras, and
  Mohamed}{Neishaboori et~al\mbox{.}}{2014b}]%
        {ours_mobiwac}
\bibfield{author}{\bibinfo{person}{Azin Neishaboori}, \bibinfo{person}{Ahmed
  Saeed}, \bibinfo{person}{Khaled~A. Harras}, {and} \bibinfo{person}{Amr
  Mohamed}.} \bibinfo{year}{2014}\natexlab{b}.
\newblock \showarticletitle{On Target Coverage in Mobile Visual Sensor
  Networks}. In \bibinfo{booktitle}{{\em MobiWac'14}}.
\newblock


\bibitem[\protect\citeauthoryear{Nigam}{Nigam}{2014}]%
        {machines2010013}
\bibfield{author}{\bibinfo{person}{Nikhil Nigam}.}
  \bibinfo{year}{2014}\natexlab{}.
\newblock \showarticletitle{The Multiple Unmanned Air Vehicle Persistent
  Surveillance Problem: A Review}.
\newblock \bibinfo{journal}{{\em Machines\/}} \bibinfo{volume}{2},
  \bibinfo{number}{1} (\bibinfo{year}{2014}), \bibinfo{pages}{13--72}.
\newblock
\showISSN{2075-1702}
\showDOI{%
\url{http://dx.doi.org/10.3390/machines2010013}}


\bibitem[\protect\citeauthoryear{Orfanus, de~Freitas, and Eliassen}{Orfanus
  et~al\mbox{.}}{2016}]%
        {relay_networks}
\bibfield{author}{\bibinfo{person}{D. Orfanus}, \bibinfo{person}{E.~P. de
  Freitas}, {and} \bibinfo{person}{F. Eliassen}.}
  \bibinfo{year}{2016}\natexlab{}.
\newblock \showarticletitle{Self-Organization as a Supporting Paradigm for
  Military UAV Relay Networks}.
\newblock \bibinfo{journal}{{\em IEEE Communications Letters\/}}
  \bibinfo{volume}{20}, \bibinfo{number}{4} (\bibinfo{date}{April}
  \bibinfo{year}{2016}), \bibinfo{pages}{804--807}.
\newblock
\showISSN{1089-7798}
\showDOI{%
\url{http://dx.doi.org/10.1109/LCOMM.2016.2524405}}


\bibitem[\protect\citeauthoryear{O'Rourke and Supowit}{O'Rourke and
  Supowit}{1983}]%
        {o1983some}
\bibfield{author}{\bibinfo{person}{Joseph O'Rourke} {and}
  \bibinfo{person}{Kenneth Supowit}.} \bibinfo{year}{1983}\natexlab{}.
\newblock \showarticletitle{Some {NP}-hard polygon decomposition problems}.
\newblock \bibinfo{journal}{{\em Transactions on Information Theory\/}}
  \bibinfo{volume}{29}, \bibinfo{number}{2} (\bibinfo{year}{1983}),
  \bibinfo{pages}{181--190}.
\newblock


\bibitem[\protect\citeauthoryear{P.~Johnson and Bar-Noy}{P.~Johnson and
  Bar-Noy}{2011}]%
        {johnson2011pan}
\bibfield{author}{\bibinfo{person}{Matthew P.~Johnson} {and}
  \bibinfo{person}{Amotz Bar-Noy}.} \bibinfo{year}{2011}\natexlab{}.
\newblock \showarticletitle{Pan and scan: Configuring cameras for coverage}. In
  \bibinfo{booktitle}{{\em INFOCOM'11}}.
\newblock


\bibitem[\protect\citeauthoryear{Palacios-Gasós, Montijano, Sagüés, and
  Llorente}{Palacios-Gasós et~al\mbox{.}}{2016}]%
        {pers_dist_TR16}
\bibfield{author}{\bibinfo{person}{J.~M. Palacios-Gasós}, \bibinfo{person}{E.
  Montijano}, \bibinfo{person}{C. Sagüés}, {and} \bibinfo{person}{S.
  Llorente}.} \bibinfo{year}{2016}\natexlab{}.
\newblock \showarticletitle{Distributed Coverage Estimation and Control for
  Multirobot Persistent Tasks}.
\newblock \bibinfo{journal}{{\em IEEE Transactions on Robotics\/}}
  \bibinfo{volume}{32}, \bibinfo{number}{6} (\bibinfo{date}{Dec}
  \bibinfo{year}{2016}), \bibinfo{pages}{1444--1460}.
\newblock
\showISSN{1552-3098}
\showDOI{%
\url{http://dx.doi.org/10.1109/TRO.2016.2602383}}


\bibitem[\protect\citeauthoryear{Palacios-Gasós, Talebpour, Montijano,
  Sagüés, and Martinoli}{Palacios-Gasós et~al\mbox{.}}{2017}]%
        {pers_ICRA17}
\bibfield{author}{\bibinfo{person}{J.~M. Palacios-Gasós}, \bibinfo{person}{Z.
  Talebpour}, \bibinfo{person}{E. Montijano}, \bibinfo{person}{C. Sagüés},
  {and} \bibinfo{person}{A. Martinoli}.} \bibinfo{year}{2017}\natexlab{}.
\newblock \showarticletitle{Optimal path planning and coverage control for
  multi-robot persistent coverage in environments with obstacles}. In
  \bibinfo{booktitle}{{\em 2017 IEEE International Conference on Robotics and
  Automation (ICRA)}}. \bibinfo{pages}{1321--1327}.
\newblock
\showDOI{%
\url{http://dx.doi.org/10.1109/ICRA.2017.7989156}}


\bibitem[\protect\citeauthoryear{Perera, Zaslavsky, Christen, and
  Georgakopoulos}{Perera et~al\mbox{.}}{2014}]%
        {IoT_context}
\bibfield{author}{\bibinfo{person}{C. Perera}, \bibinfo{person}{A. Zaslavsky},
  \bibinfo{person}{P. Christen}, {and} \bibinfo{person}{D. Georgakopoulos}.}
  \bibinfo{year}{2014}\natexlab{}.
\newblock \showarticletitle{Context Aware Computing for The Internet of Things:
  A Survey}.
\newblock \bibinfo{journal}{{\em IEEE Communications Surveys Tutorials\/}}
  \bibinfo{volume}{16}, \bibinfo{number}{1} (\bibinfo{date}{First}
  \bibinfo{year}{2014}), \bibinfo{pages}{414--454}.
\newblock
\showISSN{1553-877X}
\showDOI{%
\url{http://dx.doi.org/10.1109/SURV.2013.042313.00197}}


\bibitem[\protect\citeauthoryear{Ramakrishna, Kanade, and Sheikh}{Ramakrishna
  et~al\mbox{.}}{2012}]%
        {ramakrishna2012reconstructing}
\bibfield{author}{\bibinfo{person}{Varun Ramakrishna}, \bibinfo{person}{Takeo
  Kanade}, {and} \bibinfo{person}{Yaser Sheikh}.}
  \bibinfo{year}{2012}\natexlab{}.
\newblock \showarticletitle{Reconstructing 3D human pose from 2D image
  landmarks}.
\newblock In \bibinfo{booktitle}{{\em Computer Vision--ECCV 2012}}.
  \bibinfo{publisher}{Springer}, \bibinfo{pages}{573--586}.
\newblock


\bibitem[\protect\citeauthoryear{Saeed, Ammar, Harras, and Zegura}{Saeed
  et~al\mbox{.}}{2015}]%
        {SymbIoT}
\bibfield{author}{\bibinfo{person}{Ahmed Saeed}, \bibinfo{person}{Mostafa
  Ammar}, \bibinfo{person}{Khaled~A. Harras}, {and} \bibinfo{person}{Ellen
  Zegura}.} \bibinfo{year}{2015}\natexlab{}.
\newblock \showarticletitle{Vision: The Case for Symbiosis in the Internet of
  Things}. In \bibinfo{booktitle}{{\em Proceedings of the 6th International
  Workshop on Mobile Cloud Computing and Services}} {\em (\bibinfo{series}{MCS
  '15})}. \bibinfo{publisher}{ACM}, \bibinfo{address}{New York, NY, USA},
  \bibinfo{pages}{23--27}.
\newblock
\showISBNx{978-1-4503-3545-4}
\showDOI{%
\url{http://dx.doi.org/10.1145/2802130.2802133}}


\bibitem[\protect\citeauthoryear{Saeed, Neishaboori, Mohamed, and Harras}{Saeed
  et~al\mbox{.}}{2014}]%
        {una}
\bibfield{author}{\bibinfo{person}{Ahmed Saeed}, \bibinfo{person}{Azin
  Neishaboori}, \bibinfo{person}{Amr Mohamed}, {and} \bibinfo{person}{Khaled~A.
  Harras}.} \bibinfo{year}{2014}\natexlab{}.
\newblock \showarticletitle{Up and away: A visually-controlled easy-to-deploy
  wireless UAV Cyber-Physical testbed}. In \bibinfo{booktitle}{{\em WiMob'14}}.
\newblock


\bibitem[\protect\citeauthoryear{Sahai, Galloway, and Wood}{Sahai
  et~al\mbox{.}}{2013}]%
        {flapping_wing_image}
\bibfield{author}{\bibinfo{person}{R. Sahai}, \bibinfo{person}{K.~C. Galloway},
  {and} \bibinfo{person}{R.~J. Wood}.} \bibinfo{year}{2013}\natexlab{}.
\newblock \showarticletitle{Elastic Element Integration for Improved
  Flapping-Wing Micro Air Vehicle Performance}.
\newblock \bibinfo{journal}{{\em IEEE Transactions on Robotics\/}}
  \bibinfo{volume}{29}, \bibinfo{number}{1} (\bibinfo{date}{Feb}
  \bibinfo{year}{2013}), \bibinfo{pages}{32--41}.
\newblock
\showISSN{1552-3098}
\showDOI{%
\url{http://dx.doi.org/10.1109/TRO.2012.2218936}}


\bibitem[\protect\citeauthoryear{Schwager, Julian, Angermann, and Rus}{Schwager
  et~al\mbox{.}}{2011}]%
        {Eyes_in_the_sky}
\bibfield{author}{\bibinfo{person}{M. Schwager}, \bibinfo{person}{B.~J.
  Julian}, \bibinfo{person}{M. Angermann}, {and} \bibinfo{person}{D. Rus}.}
  \bibinfo{year}{2011}\natexlab{}.
\newblock \showarticletitle{Eyes in the Sky: Decentralized Control for the
  Deployment of Robotic Camera Networks}.
\newblock \bibinfo{journal}{{\it Proc. IEEE}} \bibinfo{volume}{99},
  \bibinfo{number}{9} (\bibinfo{date}{Sept} \bibinfo{year}{2011}),
  \bibinfo{pages}{1541--1561}.
\newblock
\showISSN{0018-9219}
\showDOI{%
\url{http://dx.doi.org/10.1109/JPROC.2011.2158377}}


\bibitem[\protect\citeauthoryear{Schwager, Vitus, Powers, Rus, and
  Tomlin}{Schwager et~al\mbox{.}}{2017}]%
        {robust_cov}
\bibfield{author}{\bibinfo{person}{M. Schwager}, \bibinfo{person}{M.~P. Vitus},
  \bibinfo{person}{S. Powers}, \bibinfo{person}{D. Rus}, {and}
  \bibinfo{person}{C.~J. Tomlin}.} \bibinfo{year}{2017}\natexlab{}.
\newblock \showarticletitle{Robust Adaptive Coverage Control for Robotic Sensor
  Networks}.
\newblock \bibinfo{journal}{{\em IEEE Transactions on Control of Network
  Systems\/}} \bibinfo{volume}{4}, \bibinfo{number}{3} (\bibinfo{date}{Sept}
  \bibinfo{year}{2017}), \bibinfo{pages}{462--476}.
\newblock
\showDOI{%
\url{http://dx.doi.org/10.1109/TCNS.2015.2512326}}


\bibitem[\protect\citeauthoryear{Tezcan and Wang}{Tezcan and Wang}{2008}]%
        {tezcan2008self}
\bibfield{author}{\bibinfo{person}{Nurcan Tezcan} {and} \bibinfo{person}{Wenye
  Wang}.} \bibinfo{year}{2008}\natexlab{}.
\newblock \showarticletitle{Self-orienting wireless multimedia sensor networks
  for occlusion-free viewpoints}.
\newblock \bibinfo{journal}{{\em Computer networks\/}} \bibinfo{volume}{52},
  \bibinfo{number}{13} (\bibinfo{year}{2008}), \bibinfo{pages}{2558--2567}.
\newblock


\bibitem[\protect\citeauthoryear{Tijmons, de~Croon, Remes, Wagter, and
  Mulder}{Tijmons et~al\mbox{.}}{2017}]%
        {MAV_stereo_TR17}
\bibfield{author}{\bibinfo{person}{S. Tijmons}, \bibinfo{person}{G.~C. H.~E. de
  Croon}, \bibinfo{person}{B.~D.~W. Remes}, \bibinfo{person}{C.~De Wagter},
  {and} \bibinfo{person}{M. Mulder}.} \bibinfo{year}{2017}\natexlab{}.
\newblock \showarticletitle{Obstacle Avoidance Strategy using Onboard Stereo
  Vision on a Flapping Wing MAV}.
\newblock \bibinfo{journal}{{\em IEEE Transactions on Robotics\/}}
  \bibinfo{volume}{33}, \bibinfo{number}{4} (\bibinfo{date}{Aug}
  \bibinfo{year}{2017}), \bibinfo{pages}{858--874}.
\newblock
\showISSN{1552-3098}
\showDOI{%
\url{http://dx.doi.org/10.1109/TRO.2017.2683530}}


\bibitem[\protect\citeauthoryear{Tokekar and Isler}{Tokekar and Isler}{2014}]%
        {Tokekar2014}
\bibfield{author}{\bibinfo{person}{P. Tokekar} {and} \bibinfo{person}{V.
  Isler}.} \bibinfo{year}{2014}\natexlab{}.
\newblock \showarticletitle{Polygon guarding with orientation}. In
  \bibinfo{booktitle}{{\em 2014 IEEE International Conference on Robotics and
  Automation (ICRA)}}. \bibinfo{pages}{1014--1019}.
\newblock


\bibitem[\protect\citeauthoryear{Tokekar, Isler, and Franchi}{Tokekar
  et~al\mbox{.}}{2014}]%
        {multi_tracking_IROS14}
\bibfield{author}{\bibinfo{person}{P. Tokekar}, \bibinfo{person}{V. Isler},
  {and} \bibinfo{person}{A. Franchi}.} \bibinfo{year}{2014}\natexlab{}.
\newblock \showarticletitle{Multi-target visual tracking with aerial robots}.
  In \bibinfo{booktitle}{{\em 2014 IEEE/RSJ International Conference on
  Intelligent Robots and Systems}}. \bibinfo{pages}{3067--3072}.
\newblock
\showISSN{2153-0858}
\showDOI{%
\url{http://dx.doi.org/10.1109/IROS.2014.6942986}}


\bibitem[\protect\citeauthoryear{Urrutia et~al\mbox{.}}{Urrutia
  et~al\mbox{.}}{2000}]%
        {urrutia2000art}
\bibfield{author}{\bibinfo{person}{Jorge Urrutia} {and}
  \bibinfo{person}{others}.} \bibinfo{year}{2000}\natexlab{}.
\newblock \showarticletitle{Art gallery and illumination problems}.
\newblock \bibinfo{journal}{{\em Handbook of computational geometry\/}}
  \bibinfo{volume}{1}, \bibinfo{number}{1} (\bibinfo{year}{2000}),
  \bibinfo{pages}{973--1027}.
\newblock


\bibitem[\protect\citeauthoryear{Vanhoutte, Ruffier, and Serres}{Vanhoutte
  et~al\mbox{.}}{2017}]%
        {open_hardware_Sensors17}
\bibfield{author}{\bibinfo{person}{E. Vanhoutte}, \bibinfo{person}{F. Ruffier},
  {and} \bibinfo{person}{J. Serres}.} \bibinfo{year}{2017}\natexlab{}.
\newblock \showarticletitle{A quasi-panoramic bio-inspired eye for flying
  parallel to walls}. In \bibinfo{booktitle}{{\em 2017 IEEE SENSORS}}.
  \bibinfo{pages}{1--3}.
\newblock
\showDOI{%
\url{http://dx.doi.org/10.1109/ICSENS.2017.8234110}}


\bibitem[\protect\citeauthoryear{Vegter}{Vegter}{1990}]%
        {Vegter1990}
\bibfield{author}{\bibinfo{person}{Gert Vegter}.}
  \bibinfo{year}{1990}\natexlab{}.
\newblock \showarticletitle{The visibility diagram: A data structure for
  visibility problems and motion planning}. In \bibinfo{booktitle}{{\em 2nd
  Scandinavian Workshop on Algorithm Theory}}.
\newblock


\bibitem[\protect\citeauthoryear{Wang, Krishnamurti, and Gupta}{Wang
  et~al\mbox{.}}{2007}]%
        {wang2007metric}
\bibfield{author}{\bibinfo{person}{Pengpeng Wang}, \bibinfo{person}{Ramesh
  Krishnamurti}, {and} \bibinfo{person}{Kamal Gupta}.}
  \bibinfo{year}{2007}\natexlab{}.
\newblock \showarticletitle{Metric view planning problem with traveling cost
  and visibility range}. In \bibinfo{booktitle}{{\em Robotics and Automation,
  2007 IEEE International Conference on}}. IEEE, \bibinfo{pages}{1292--1297}.
\newblock


\bibitem[\protect\citeauthoryear{Wang, Chowdhery, and Chiang}{Wang
  et~al\mbox{.}}{2016}]%
        {video_streaming}
\bibfield{author}{\bibinfo{person}{Xiaoli Wang}, \bibinfo{person}{Aakanksha
  Chowdhery}, {and} \bibinfo{person}{Mung Chiang}.}
  \bibinfo{year}{2016}\natexlab{}.
\newblock \showarticletitle{SkyEyes: Adaptive Video Streaming from UAVs}. In
  \bibinfo{booktitle}{{\em Proceedings of the 3rd Workshop on Hot Topics in
  Wireless}} {\em (\bibinfo{series}{HotWireless '16})}.
  \bibinfo{publisher}{ACM}, \bibinfo{address}{New York, NY, USA},
  \bibinfo{pages}{2--6}.
\newblock
\showISBNx{978-1-4503-4251-3}
\showDOI{%
\url{http://dx.doi.org/10.1145/2980115.2980119}}


\bibitem[\protect\citeauthoryear{Wang, Chowdhery, and Chiang}{Wang
  et~al\mbox{.}}{2017}]%
        {sports}
\bibfield{author}{\bibinfo{person}{X. Wang}, \bibinfo{person}{A. Chowdhery},
  {and} \bibinfo{person}{M. Chiang}.} \bibinfo{year}{2017}\natexlab{}.
\newblock \showarticletitle{Networked Drone Cameras for Sports Streaming}. In
  \bibinfo{booktitle}{{\em 2017 IEEE 37th International Conference on
  Distributed Computing Systems (ICDCS)}}. \bibinfo{pages}{308--318}.
\newblock
\showISSN{1063-6927}
\showDOI{%
\url{http://dx.doi.org/10.1109/ICDCS.2017.200}}


\bibitem[\protect\citeauthoryear{Wang and Cao}{Wang and Cao}{2011a}]%
        {wang2011barrier}
\bibfield{author}{\bibinfo{person}{Yi Wang} {and} \bibinfo{person}{Guohong
  Cao}.} \bibinfo{year}{2011}\natexlab{a}.
\newblock \showarticletitle{Barrier coverage in camera sensor networks}. In
  \bibinfo{booktitle}{{\em MobiHoc '11}}.
\newblock


\bibitem[\protect\citeauthoryear{Wang and Cao}{Wang and Cao}{2011b}]%
        {wang2011full}
\bibfield{author}{\bibinfo{person}{Yi Wang} {and} \bibinfo{person}{Guohong
  Cao}.} \bibinfo{year}{2011}\natexlab{b}.
\newblock \showarticletitle{On full-view coverage in camera sensor networks}.
  In \bibinfo{booktitle}{{\em INFOCOM'11}}.
\newblock


\bibitem[\protect\citeauthoryear{Weinland, {\"O}zuysal, and Fua}{Weinland
  et~al\mbox{.}}{2010}]%
        {weinland2010making}
\bibfield{author}{\bibinfo{person}{Daniel Weinland}, \bibinfo{person}{Mustafa
  {\"O}zuysal}, {and} \bibinfo{person}{Pascal Fua}.}
  \bibinfo{year}{2010}\natexlab{}.
\newblock \showarticletitle{Making action recognition robust to occlusions and
  viewpoint changes}.
\newblock In \bibinfo{booktitle}{{\em European Conference on Computer Vision}}.
  \bibinfo{pages}{635--648}.
\newblock


\bibitem[\protect\citeauthoryear{Wood}{Wood}{2008}]%
        {Wood_TR08}
\bibfield{author}{\bibinfo{person}{R.~J. Wood}.}
  \bibinfo{year}{2008}\natexlab{}.
\newblock \showarticletitle{The First Takeoff of a Biologically Inspired
  At-Scale Robotic Insect}.
\newblock \bibinfo{journal}{{\em IEEE Transactions on Robotics\/}}
  \bibinfo{volume}{24}, \bibinfo{number}{2} (\bibinfo{date}{April}
  \bibinfo{year}{2008}), \bibinfo{pages}{341--347}.
\newblock
\showISSN{1552-3098}
\showDOI{%
\url{http://dx.doi.org/10.1109/TRO.2008.916997}}


\bibitem[\protect\citeauthoryear{Wu and Wang}{Wu and Wang}{2012}]%
        {wu2012achieving}
\bibfield{author}{\bibinfo{person}{Yibo Wu} {and} \bibinfo{person}{Xinbing
  Wang}.} \bibinfo{year}{2012}\natexlab{}.
\newblock \showarticletitle{Achieving full view coverage with randomly-deployed
  heterogeneous camera sensors}. In \bibinfo{booktitle}{{\em ICDCS'12}}.
\newblock


\bibitem[\protect\citeauthoryear{Xu, He, and Li}{Xu et~al\mbox{.}}{2014}]%
        {IoT_industry}
\bibfield{author}{\bibinfo{person}{L.~D. Xu}, \bibinfo{person}{W. He}, {and}
  \bibinfo{person}{S. Li}.} \bibinfo{year}{2014}\natexlab{}.
\newblock \showarticletitle{Internet of Things in Industries: A Survey}.
\newblock \bibinfo{journal}{{\em IEEE Transactions on Industrial
  Informatics\/}} \bibinfo{volume}{10}, \bibinfo{number}{4}
  (\bibinfo{date}{Nov} \bibinfo{year}{2014}), \bibinfo{pages}{2233--2243}.
\newblock
\showISSN{1551-3203}
\showDOI{%
\url{http://dx.doi.org/10.1109/TII.2014.2300753}}


\bibitem[\protect\citeauthoryear{Yildiz, Akkaya, Sisikoglu, and Sir}{Yildiz
  et~al\mbox{.}}{2014}]%
        {6471967}
\bibfield{author}{\bibinfo{person}{Enes Yildiz}, \bibinfo{person}{Kemal
  Akkaya}, \bibinfo{person}{Esra Sisikoglu}, {and} \bibinfo{person}{Mustafa~Y
  Sir}.} \bibinfo{year}{2014}\natexlab{}.
\newblock \showarticletitle{Optimal Camera Placement for Providing Angular
  Coverage in Wireless Video Sensor Networks}.
\newblock \bibinfo{journal}{{\it IEEE Trans. Comput.}} \bibinfo{volume}{63},
  \bibinfo{number}{7} (\bibinfo{year}{2014}), \bibinfo{pages}{1812--1825}.
\newblock


\bibitem[\protect\citeauthoryear{Yu, Yang, Teng, Champion, and Xuan}{Yu
  et~al\mbox{.}}{2015}]%
        {7218437}
\bibfield{author}{\bibinfo{person}{Zuoming Yu}, \bibinfo{person}{Fan Yang},
  \bibinfo{person}{Jin Teng}, \bibinfo{person}{A.C. Champion}, {and}
  \bibinfo{person}{Dong Xuan}.} \bibinfo{year}{2015}\natexlab{}.
\newblock \showarticletitle{Local face-view barrier coverage in camera sensor
  networks}. In \bibinfo{booktitle}{{\em INFOCOM'15}}.
\newblock


\bibitem[\protect\citeauthoryear{Zanella, Bui, Castellani, Vangelista, and
  Zorzi}{Zanella et~al\mbox{.}}{2014}]%
        {IoT_cities}
\bibfield{author}{\bibinfo{person}{A. Zanella}, \bibinfo{person}{N. Bui},
  \bibinfo{person}{A. Castellani}, \bibinfo{person}{L. Vangelista}, {and}
  \bibinfo{person}{M. Zorzi}.} \bibinfo{year}{2014}\natexlab{}.
\newblock \showarticletitle{Internet of Things for Smart Cities}.
\newblock \bibinfo{journal}{{\em IEEE Internet of Things Journal\/}}
  \bibinfo{volume}{1}, \bibinfo{number}{1} (\bibinfo{date}{Feb}
  \bibinfo{year}{2014}), \bibinfo{pages}{22--32}.
\newblock
\showISSN{2327-4662}
\showDOI{%
\url{http://dx.doi.org/10.1109/JIOT.2014.2306328}}


\end{thebibliography}

\end{document}